\documentclass[dvipsnames]{article} %
\usepackage{colm2024_conference}

\usepackage{booktabs}
\usepackage{graphicx}
\usepackage{enumitem}
\usepackage{wrapfig}
\usepackage{algorithm}
\usepackage{algpseudocode}

\usepackage{microtype}
\usepackage{amsmath}
\usepackage{colortbl}
\usepackage[utf8]{inputenc}
\definecolor{lightgray}{rgb}{0.9,0.9,0.9}
\usepackage{caption}
\usepackage{subcaption}
\usepackage{setspace}
\usepackage{url}
\usepackage{multirow}
\usepackage{colortbl}
\usepackage{tabularx}
\usepackage{blindtext}
\usepackage{pgfplots}
\pgfplotsset{compat=1.18} 
\usepackage{tikz}
\usetikzlibrary{er,positioning,bayesnet}
\usepackage{makecell}
\usepackage{tipa}
\usepackage{siunitx}
\usepackage{nicefrac}
\usepackage{tocloft}
\usepackage{listings}
\usepackage[raster,skins]{tcolorbox} %
\usepackage{xltabular}
\usepackage{adjustbox}
\usepackage{xurl}
\usepackage{rotating}
\usepackage[normalem]{ulem}
\useunder{\uline}{\ul}{}
\usepackage{crossreftools}

\usepackage{amsthm}
\usepackage{amssymb}
\usepackage{mathtools}

\usepackage[capitalize,noabbrev,nameinlink]{cleveref}
\crefname{assumption}{Assumption}{Assumptions}

\theoremstyle{plain}
\newtheorem{theorem}{Theorem}[section]
\newtheorem{proposition}[theorem]{Proposition}
\crefname{proposition}{proposition}{propositions}
\Crefname{proposition}{Proposition}{Propositions}
\newtheorem{lemma}[theorem]{Lemma}

\theoremstyle{definition}
\newtheorem{definition}{Definition}[section]
\newtheorem{example}{Example}[section]
\newtheorem{assumption}{Assumption}[section]
\theoremstyle{remark}
\newtheorem{remark}{Remark}[section]

\makeatletter
\def\maketag@@@#1{\hbox{\m@th\normalfont\normalsize#1}}
\makeatother

\makeatletter
\pretocmd{\contentsline}
  {\begingroup\hypersetup{linkcolor=black}}{}{}
\apptocmd{\contentsline}
  {\endgroup}{}{}
\makeatother

\usepackage{amsmath,amsfonts,bm}

\def\1{\bm{1}}

\def\rva{{\mathbf{a}}}
\def\rvb{{\mathbf{b}}}

\def\rve{{\mathbf{e}}}

\def\rvu{{\mathbf{i}}}

\def\rvu{{\mathbf{u}}}

\def\rvx{{\mathbf{x}}}
\def\rvy{{\mathbf{y}}}

\def\rmE{{\mathbf{E}}}

\def\rmW{{\mathbf{W}}}

\DeclareMathAlphabet{\mathsfit}{\encodingdefault}{\sfdefault}{m}{sl}
\SetMathAlphabet{\mathsfit}{bold}{\encodingdefault}{\sfdefault}{bx}{n}

\def\gB{{\mathcal{B}}}

\def\gN{{\mathcal{N}}}
\def\gO{{\mathcal{O}}}
\def\gP{{\mathcal{P}}}

\def\gX{{\mathcal{X}}}

\def\sR{{\mathbb{R}}}

\newcommand{\E}{\mathbb{E}}

\DeclareMathOperator*{\argmin}{arg\,min}

\newcommand{\bfX}{\mathbf{X}}
\newcommand{\bfXtest}{\mathbf{X}^{\text{te}}}
\newcommand{\bfXcontext}{\mathbf{X}^{\text{ct}}}
\newcommand{\Xtest}{X^{\text{te}}}

\newcommand{\bfxtest}{\mathbf{x}^{\text{te}}}

\newcommand{\bfxtesti}{\mathbf{x}^{\text{te},(i)}}
\newcommand{\bfxcontext}{\mathbf{x}^{\text{ct}}}
\newcommand{\bfxcontexti}{\mathbf{x}^{\text{ct}, (i)}}
\newcommand{\bfmu}{\boldsymbol{\mu}}

\newcommand{\SCM}{S}
\newcommand{\uniform}{\mathrm{Unif}}
\newcommand{\htheta}{\hat{\theta}}

\newcommand{\ours}{LimiX}
\newcommand{\ourL}{LimiX-16M}
\newcommand{\ourS}{LimiX-2M}
\newcommand{\ourA}{LimiX models}

\renewcommand{\ghlink}{https://github.com/limix-ldm/LimiX/}

\newcommand*\justify{%
  \fontdimen2\font=0.4em
  \fontdimen3\font=0.2em
  \fontdimen4\font=0.1em
  \fontdimen7\font=0.1em
  \hyphenchar\font=`\-
}

\renewcommand{\texttt}[1]{%
  \begingroup
  \ttfamily
  \begingroup\lccode`~=`/\lowercase{\endgroup\def~}{/\discretionary{}{}{}}%
  \begingroup\lccode`~=`[\lowercase{\endgroup\def~}{[\discretionary{}{}{}}%
  \begingroup\lccode`~=`.\lowercase{\endgroup\def~}{.\discretionary{}{}{}}%
  \catcode`/=\active\catcode`[=\active\catcode`.=\active
  \justify\scantokens{#1\noexpand}%
  \endgroup
}

\title{LimiX: Unleashing Structured-Data Modeling Capability for Generalist Intelligence}

\author{
{\normalsize \bf{LimiX Team}} \\
\vspace{10pt}
 \rm Stable AI \& Tsinghua University
}

\begin{document}
\maketitle

\vspace{1em}

\begin{abstract}

We argue that progress toward general intelligence requires complementary foundation models grounded in language, the physical world, and structured data. This report presents LimiX-16M and LimiX-2M, two instantiations of our large structured-data models (LDMs). Both models treat structured data as a joint distribution over variables and missingness, thus capable of addressing a wide range of tabular tasks through query-based conditional prediction via a single model. 
They are pretrained using masked joint-distribution modeling with an episodic, context-conditional objective, supporting rapid, training-free adaptation at inference.
We evaluate LimiX models across 11 large structured-data benchmarks with broad regimes of sample size, feature dimensionality, class number, categorical–to-numerical feature ratio, missingness and sample-to-feature ratios.  
LimiX-16M consistently surpasses strong baselines, as shown in \cref{fig:auc_rank} and \cref{fig:r2_rank}. 
The superiority holds across a wide range of tasks, such as classification, regression, missing value imputation, and data generation, often by substantial margins, while avoiding task-specific architectures or bespoke training per task. Notably, LimiX-2M delivers strong results under tight compute and memory budgets. We also present the first scaling law study for LDMs, revealing how data and model scaling jointly influence downstream performance and offering quantitative guidance for tabular foundation modeling.
All LimiX models are publicly accessible under Apache 2.0.

\end{abstract}

\begin{figure}[th]
    \centering
    \vspace{-10pt}
    \includegraphics[width=1.0 \linewidth]{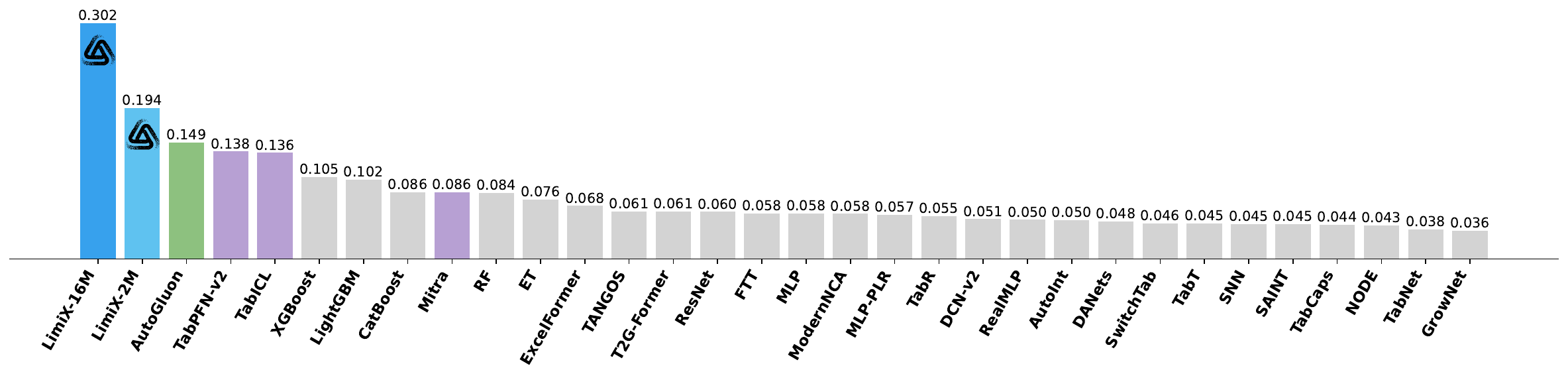}
    \vspace{-14pt}
    \caption{Performance comparison on the averaged reciprocal of the ranks, where the rank is that of the corresponding model on ROC AUC. Higher values indicate stronger average ranking performance across all the classification benchmarks.}
    \label{fig:auc_rank}
\end{figure}

\begin{figure}[th]
    \centering
    \includegraphics[width=1.0 \linewidth]{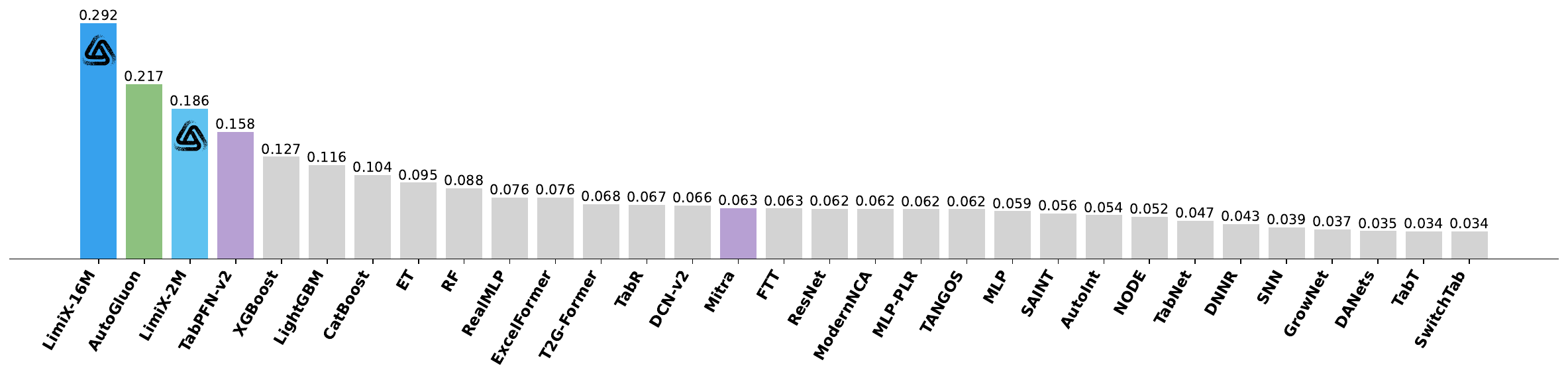}
    \vspace{-14pt}
    \caption{Performance comparison on the averaged reciprocal of the ranks, where the rank is that of the corresponding model on R\textsuperscript{2} across all the regression benchmarks.}
    \label{fig:r2_rank}
\end{figure}

\clearpage
\setcounter{tocdepth}{2} 
\tableofcontents
\clearpage
\section{Introduction}
We posit that progress toward general intelligence is best organized around three complementary spaces: language, physical-world, and structured data, each anchored to a distinct data modality and set of inductive biases. In the language space, large language models (LLMs) provide a universal interface for natural and programming languages and have rapidly advanced instruction following, tool use, and explicit reasoning over token sequences~\citep{openai2023gpt4,touvron2023llama,team2023gemini,bai2023qwen}. In the physical-world space, recent foundation models ground knowledge in space perception and embodied reasoning. They emphasize spatial intelligence through structured scene understanding, controllable scene generation, and neural radiance field reconstruction~\citep{mildenhall2020nerf,ke20243d,xiang2025repurposing,li2024behavior,yi2024gaussiandreamer}. These models also include self-supervised video world models such as V-JEPA, which learn predictive abstractions from large-scale videos and support downstream planning and control~\citep{bardes2024vjepa,assran2025vjepa2}.

On the other hand, structured data serves as the foundational bedrock for evidence-based decision-making across a multitude of critical domains, including finance, healthcare, logistics, and public policy~\citep{fuster2022predictably,johnson2016mimic,yu2021scf,krafft2020defining}. 
The structural consistency and inherent order of structured data provide a robust framework for quantitative analysis and reliable operations~\citep{ramakrishnan2003dbms,silberschatz2011db,stonebraker2018one}, enabling precise prediction, automated reasoning, and rigorous causal inference~\citep{pearl2009causality,hernan2010causal,little2019statistical}. While the emergence of unstructured data has captured considerable attention, the analytical power and operational reliability of structured data remain unparalleled for a vast array of real-world applications~\citep{fang2024llmtabular,vanbreugel2024priority}. Consequently, advancements in structured-data prediction, analysis, and reasoning are not merely an academic pursuit but a critical enabler for efficiency, innovation, and accuracy in modern data-driven systems.
Structured data is not subsumed by language models or embodied intelligence. Converting tables into free text discards metric geometry, physical units, and patterns of missingness that are central to reliable prediction, while models focusing on perception and control in three-dimensional physical environments do not capture discrete interventions, business rules, or causal heterogeneity across environments. Empirical surveys also document current limitations of language models on tabular prediction without bespoke adaptation~\citep{fang2024llmtabular,sui2024table}.

Traditionally, practitioners deploy pipelines of specialized models with gradient-boosting trees and automated ensembles such as XGBoost~\citep{chen2016xgboost}, LightGBM~\citep{ke2017lightgbm}, CatBoost~\citep {dorogush2018catboost}, and AutoGluon~\citep{erickson2020autogluon}—that are trained separately for each dataset and task. These systems excel at supervised prediction but require full retraining on every new dataset, prolonging deployment and preventing reuse of knowledge across domains. Recent deep approaches for tables have improved accuracy on mixed-type data~\citep{huang2020tabtransformer,yoon2020vime,gorishniy2021revisiting,somepalli2021saint,bahri2021scarf}. However, they are still typically trained per dataset and do not provide a single model that transfers across objectives and constraints. 

These limitations motivate the pursuit of a foundation model for structured data, i.e. models trained on large families of datasets to perform in-context learning (ICL) without per-task fine-tuning. Notably, TabPFN~\citep{hollmann2022tabpfn} and its successor TabPFN-v2~\citep{hollmann2025tabpfn_nature} demonstrate state-of-the-art performance and speed on small-to-medium-scale tables via prior-data fitting over diverse generative processes, and community efforts increasingly argue for tabular foundation models as a distinct paradigm~\citep{vanbreugel2024priority}. 
Concurrent works such as TabDPT~\citep{ma2024tabdpt} and TabICL~\citep{qu2025tabicl} explore scaling these ideas to real and larger datasets, narrowing the gap between tabular foundation models and classical methods under the scenarios of larger sample sizes.

Despite the progress, these are still in the early stages of foundation model development and remain limited in generality and performance. Most models are developed and evaluated primarily for supervised prediction (classification or regression), and typically require per-task models, adapters, or external pipelines to address other objectives. 
In practice, one still needs to assemble separate components for classification, regression, missing value imputation, data generation, and sample selection for interpretability, with different training losses and hyperparameters, so the resulting system is not a single reusable learner that delivers all of these capabilities end-to-end while maintaining reliable performance. This gap motivates a large structured-data model (LDM) that treats structured data as a joint distribution over variables and missingness, enabling multiple tasks to be posed as queries to one model.

In this work, we present \ours, a unified family of our LDM series, and release its first two variants, \ourL~and \ourS. \ourA~aim to push generality further: a single model capable of classification, regression, missing value imputation, data generation, and sample selection for interpretability under one training and inference recipe, shifting the paradigm from bespoke pipelines to unified, foundation-style tabular learning. 
\ourA~adopt a lightweight, scalable architecture that represents structured data as a set of sample–feature embeddings and learns dependencies across two dimensions: across features (columns) and across samples (rows). 
To make the attention module explicitly column-aware without inflating parameters, we introduce a low-rank discriminative feature encoding that encodes feature identities. 
Pretraining of \ourA~follows a masked joint-distribution objective and an episodic, context-conditional formulation: For each dataset, an in-context subset establishes dataset-specific priors, and the model is trained to predict masked entries in a disjoint query subset, enabling per-dataset adaptation without fine-tuning at inference. 
The pretraining corpus consists of data synthesized from hierarchical structural causal models (SCMs). Within the synthesis pipeline, we employ graph-aware sampling to obey the causal structure and solvability-aware sampling to accommodate the data quantity of various downstream tasks, improving coverage of local patterns and generalization. 
At inference, attention-guided retrieval provides an efficient, optional ensemble and fine-tuning mechanism. \ourA~retrieve informative samples and features using their own attention scores, aggregate predictions across a handful of lightweight pipelines, and deliver calibrated outputs for various downstream tasks, all through a unified conditional-inference interface and without task-specific architectures or bespoke per-dataset training.

We evaluate LimiX-16M and LimiX-2M across 11 large structured-data benchmarks with broad regimes of sample size, feature dimensionality, class number, categorical–to-numerical feature ratio, missingness and sample-to-feature ratios. With a single model and a unified inference interface, \ourL~surpasses competitive baselines including gradient-boosting trees, deep tabular networks, recent tabular foundation models, and automated ensemble methods. 
Across classification, regression, missing value imputation, data generation, and out-of-distribution prediction, \ourL~delivers consistent gains, often by large margins, while avoiding task-specific model architectures, customized ensembles, or per-dataset training. Notably, LimiX-2M deliver strong results under tight compute and memory budgets.

On most benchmarks, such as OpenML-CC18~\citep{bischl2017openml}, TabArena~\citep{erickson2025tabarena}, TALENT~\citep{liu2024talenttabularanalyticslearning}, \ourL~is the only model that consistently outperforms AutoGluon, which is considered an outstanding baseline across various tabular-data tasks.

This work introduces the first scaling law for LDMs, providing quantitative insights into how data and model scaling shape downstream performance and inform principled design of tabular foundation models.

\section{Architecture} \label{sect:architecture}

We consider a dataset $\mathcal{D}=\{(\rvx^R, \rvy^R)\}$ of $d$ features and an outcome variable, where $\rvx^R=\{\rvx^R_i\}_{i=1}^m$ and $\rvy^R=\{y^R_i\}_{i=1}^m$; 
the superscript $R$ indicates the raw input.
Here, $\rvx^R_i\in \mathbb{R}^d$ and $y^R_i\in\mathbb{R}$ correspond to the $i^{th}$ sample, while $\rvx^R \in \mathbb{R}^{m \times d}$ and $\mathbf{y}^R\in\mathbb{R}^{m}$ correspond to the 2D tabular data. 
For in-context samples and test samples, we use subscripts to distinguish between them. For example, $\rvx^R_{ct}\in\mathbb{R}^{m_{ct}\times d}$ denotes the features of in-context samples, while $\rvx^R_{te}\in\mathbb{R}^{m_{te}\times d}$ denotes those of the test samples.

\subsection{Embedding of Tabular Data}
To ensure compatibility with modern architectures such as Transformers~\citep{vaswani2017attention}, each cell of the 2D tabular input $x^R_{i,j}\in\mathbb{R}$ is first projected into a latent embedding space $\mathbb{R}^{p}$. Specifically, $\rvx^R \in \mathbb{R}^{m \times d}$ is transformed into $\rvx\in \mathbb{R}^{m\times d\times p}$ and $\mathbf{y}^R\in\mathbb{R}^{m}$ is transformed into $\rvy\in \mathbb{R}^{m\times p}$. 
All subsequent attention operations are then conducted within this latent embedding space.
Such a high-dimensional embedding space could strengthen the expressivity of the model.
Concretely, we employ a two-layer MLP with LayerNorm~\citep{ba2016layernorm} as the embedding module, i.e. $\rvx_{i,j}=\sigma(LN(\sigma(LN(\rvx^R_{i,j}\rmW^{(1)} + \rvb^{(1)}))\rmW^{(2)} + \rvb^{(2)}))$, where LN is LayerNorm and $\sigma$ is the GELU activation function \citep{hendrycks2016gaussian}. Separate embedding modules are used for $\rvx^R$ and $\rvy^R$.

\begin{figure}[th]
    \centering
    \includegraphics[width=1 \linewidth]{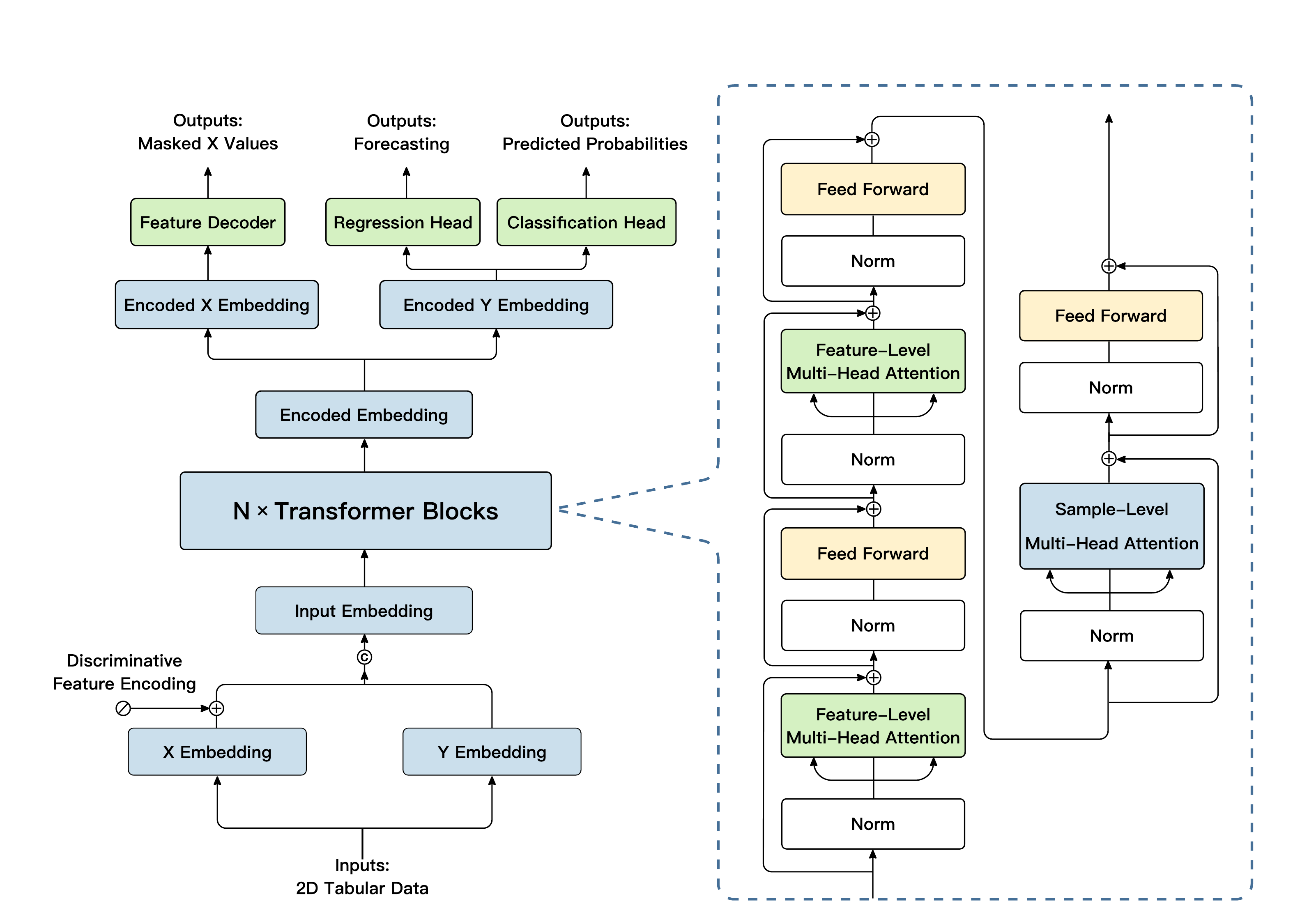}
    \caption{The overall model structure of \ourL.}
    \label{fig:model_structure}
\end{figure}

\subsection{Discriminative Feature Encoding}

The attention score between feature $j$ and $j'$ for sample $i$ is $\frac{1}{\sqrt{p}}
\left\langle \rvx_{i,j}\rmW_{\mathrm{Q}}, \rvx_{i,j'}\rmW_{\mathrm{K}} \right\rangle$.
This score depends solely on the interactions between the embeddings of cell values and imposes no
explicit prior on feature (column) identity. 
As a result, the model cannot reliably infer the column from which a cell value originates.
Inspired by~\citet{hollmann2025tabpfn_nature}, we introduce a learnable low-rank column identifier termed discriminative feature encoding (DFE) whose design philosophy includes two objectives:
(i) Encodings of different features should be well separated to ensure discriminability; 
(ii) Encodings in the embedding space admit a low effective rank so that the model can express column identities compactly and share statistical strength across features.

For implementation, 
let $s$ denote the rank of DFE, and $s\ll p$ compared with the dimension of the embedding space $p$.  
We initialize a matrix $\rvu \in \mathbb{R}^{d\times s}$
whose $j^{th}$ row vector $\rvu_j\in \mathbb{R}^s$ is a low-dimensional code of the column identifier corresponding to feature (column) $j$. Rows are initialized to be approximately
orthogonal and normalized. Then a linear transformation $\rmE \in \mathbb{R}^{s\times p}$ lifts
the code from the low-dimensional code space to the high-dimensional embedding space, i.e.
$\rve_j = \,\rvu_j \rmE\in \mathbb{R}^p$, serving as the DFE for feature $j$. 
Finally, the embedding of cell $(i,j)$ is
augmented additively as $\tilde{\rvx}_{i,j} = \rvx_{i,j} + \rve_j$, which is 
analogous to absolute positional encodings but is applied along the feature axis.
Empirically, we set $s=p/4$ by default. 

\subsection{Model Structure}
As shown in \Cref{fig:model_structure}, \ourA~comprise 12 transformer blocks. Each block performs axis-wise self-attention along the feature and sample axes, incorporating position-wise feed-forward networks (FFNs). In \ourL, we employ an asymmetric configuration with two feature-level passes and one sample-level pass, each followed by an FFN, as ablation studies demonstrate that equal numbers of feature and sample passes underrepresent feature interactions, while increased feature-axis attention enhances modeling capacity for heterogeneous schemas with negligible overhead. All sublayers use pre-normalized LayerNorm~\citep{ba2016layernorm} to stabilize optimization and support deeper scaling.

\section{Pretraining} \label{sect:pretraining}

\subsection{Context-Conditional Masked Modeling for Joint Distribution Learning}

We pretrain \ourA~by randomly masking cells in each row and training the model to recover the hidden entries from the visible context. Exposing the model to various mask patterns forces it to master a wide spectrum of conditional dependencies among variables. When these conditionals are learned properly, they effectively define a single joint model of the data. This joint model can then be queried for diverse tasks with one mechanism: Treat a chosen column as the target to perform regression or classification, fill in missing values by predicting the masked cells from the observed ones, and generate new samples by iteratively masking and refilling subsets of features. 

To better align pretraining with inference, we adopt a Context-Conditional Masked Modeling (CCMM) objective. For each dataset, an episode splits rows into a context subset and a query subset. The model encodes the context and conditions predictions for the query rows on this context through attention or feature modulation, learning to answer queries of the form “predict masked entries in the query given the observed query features and the context”. This episodic formulation enables rapid and label-free adaptation to new datasets at inference time: A handful of context rows establish dataset-specific priors such as category frequencies, marginal scales, and cross-feature couplings, while a single parametric model serves all schemas and tasks. 

Unlike BERT-style masked modeling~\citep{devlin2019bert}, where adaptation is implicit in parameters accumulated during pretraining, CCMM considers context as a non-parametric memory that the model can consult during inference.
This yields per-dataset calibration, better handling of rare categories, and improved robustness to distribution shift without gradient updates, aligning with principles of in-context learning and meta-learning for mixed-type tabular data.

We also find that CCMM improves modeling of the underlying dependency structure compared with recent tabular foundation models~\citep{hollmann2025tabpfn_nature,qu2025tabicl}. By demanding consistency across numerous conditional predictions of all features rather than optimizing the loss of a single conditional prediction of a prefixed feature, the model is compelled to capture stable variable–variable relations instead of brittle decision boundaries. As the coverage of masks becomes richer, these cross-conditional constraints get tightened, yielding more reliable recovery of the joint distribution of all the features and more stable estimates in downstream usages. Please refer to \Cref{sect:theory} for details.

\subsection{Mask Pattern Design}

We employ a heterogeneous schedule that interleaves cell-wise, column-wise, and block masks, enabling the model to practise recovering both isolated entries and coordinated subsets of variables. Cell-wise masks refine local conditional predictions; column-wise masks force the model to treat an entire feature as missing and infer it from the remaining attributes; block masks target higher-order dependencies across semantically related groups (e.g., demographics with outcomes, laboratory panels with diagnoses). Masking rates are stratified by variable type, prevalence, and dispersion to avoid overfitting to common categories and to limit the influence of high-variance continuous features, and we exclude degenerate patterns that remove nearly all informative context. This diversified schedule provides broad coverage of conditional relationships and prevents the model from specializing to a narrow reconstruction regime, yielding a more faithful approximation of the joint distribution. In practice, we sample the mask ratio in [0.1, 0.4] for \ourA.

\subsection{Mask Embedding}
To model masked cells, we introduce learnable mask embeddings that explicitly mark positions to be predicted. For each masked entry, its embedding is replaced by a trainable mask vector that is combined with the column embedding, so the encoder can condition on what is missing as well as where. We align the training masks introduced by the objective with naturally missing/structurally unavailable values observed in real data by using the shared mask embedding and calibrating the masking schedule to match empirical missingness patterns, thereby minimizing distribution shift between synthetic and real scenarios.

The mask embeddings flow through the same attention blocks as observed cells, enabling the model to request information from relevant columns and to produce calibrated distributional predictions at the output head. To reduce the mismatch between pretraining and fine-tuning, we condition the network on the corruption level of each dataset using a mask density feature, defined as the proportion of cells that are masked in the dataset. This scalar is encoded via a lightweight statistics token whose embedding is conditioned on the mask density with a small MLP module. The mask ratio embedding is included alongside the data tokens and participates in self-attention. Providing this cue regularizes the model across a range of masking rates and reduces pretraining to inference mismatch, which improves calibration when the inference pattern, such as masking a single target column, differs from the heavier masks used during pretraining.

\section{Pretraining Data Generation} \label{sect:data-generation}

Performance of foundation models largely depends on the diversity and quality of pretraining data. 
To obtain a well-generalized foundation model for tabular data, we generate synthetic datasets using Directed Acyclic Graphs (DAGs), enabling the creation of datasets with diverse characteristics. The data generation process consists of three stages: DAG generation, data sampling, and task adaptation. 
First, a DAG is constructed to represent complex causal dependencies among variables. Then, a subset of these variables is sampled to define a specific problem, allowing \ourA~to develop causal reasoning capabilities. Finally, the sampled data is processed to align with various downstream tasks. Compared with the data generation methods used in current foundation models ~\citep{hollmann2025tabpfn_nature, qu2025tabicl}, we adopt a hierarchical generation paradigm that establishes causal dependencies in a more controllable and interpretable manner, while the sampling strategies further enhance the ability to generate datasets with diverse solvability and characteristics.

\subsection{DAG Generation based on Hierarchical SCMs}

The theoretical foundation of data generation lies in structural causal models (SCMs).
As illustrated in \Cref{fig: datagen_DAG}, within a DAG, initial data for each root node is independently sampled from an assigned distribution, with the distributions themselves chosen from a collection parameterized by randomized hyperparameters. Beginning at the root nodes, the initial data 
propagates through the DAG along its directed edges, wherein each encountered node represents a distinct local causal structure (LCS). As previously mentioned, the diversity of synthetic data is crucial for effective pretraining. To ensure that DAG generation is both diverse and well-structured, we adopt a hierarchical generation scheme rather than generating DAGs directly, thereby allowing for more fine-grained controls over the generation process. Within the LCS, the data input first propagates to the connected parent nodes, and the data of the child node can be obtained through $X_i = f(\{g_{i,k}(X_k)_{k\in {\rm PA}(X_i)}, \varepsilon_i)\}$ where $g_{i,k}(\cdot)$ represents the edge function from $X_k$ to $X_i$, ${\rm PA}(X_i)$ denotes the set of parent nodes of $X_i$, $f(\cdot)$ is a parameterized aggregation function, and $\varepsilon_i$ is an observational noise.
This framework allows us to capture diverse local causal dependencies. For example, if an LCS contains only one parent node, that node is the direct cause of the child, making the dependency clear and straightforward. However, when multiple parent nodes are included in the LCS, the causal relationship becomes more complex due to the presence of potential confounding variables within the network.

\begin{figure}[t]
    \centering
    \includegraphics[width=0.9\linewidth]{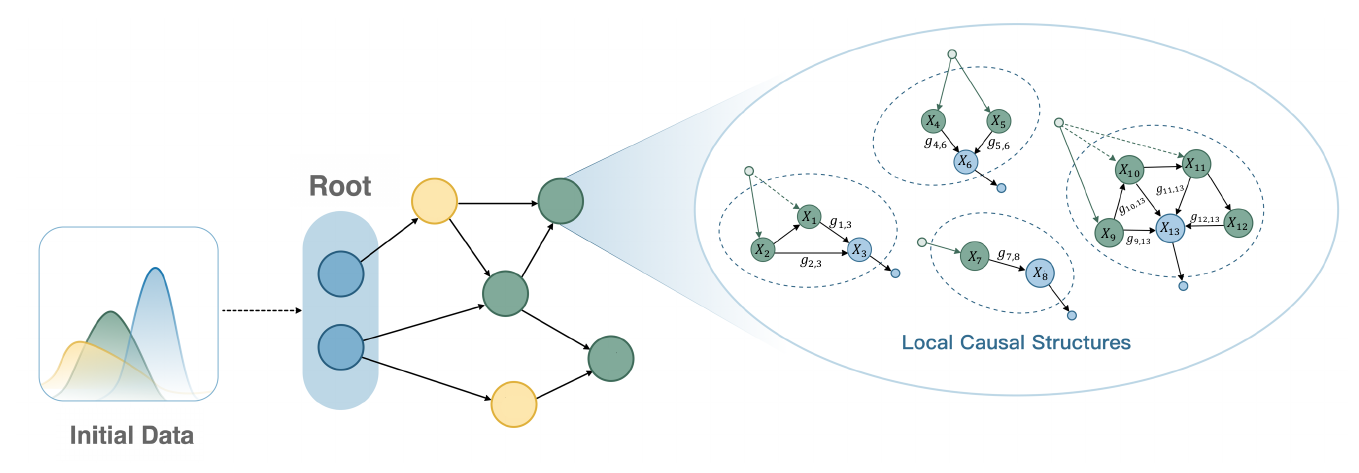}
    \caption{An example of the generated DAG, where $g_{i,j}$ is an edge function defining the relationship between a parent node $X_i$ and its child node $X_j$ in the local causal structures.
    }
    \label{fig: datagen_DAG}
\end{figure}

In addition to the structural properties of LCSs, edge functions and the aggregation function also play crucial roles in shaping the dependencies. In this work, we use three types of edge functions:

\begin{itemize}
    \item \textbf{Multilayer perceptrons (MLPs)}: To determine the architecture of these neural networks, we uniformly sample properties such as the number of linear layers and their associated activation functions. When multiple layers are included, the MLP introduces complex nonlinear dependencies along the edge, whereas with only a single layer, it degenerates into a simple transformation. For weight initialization, we randomly choose from Xavier initialization \citep{glorot2010understanding} and He initialization \citep{he2015delving}. The activation functions are uniformly sampled from identity mapping, hyperbolic tangent, sigmoid, logarithm, absolute value, sine, squaring, modulo operation, heaviside step function, and GELU.

    \item \textbf{Convolutional layers}: For node-level tabular data, the convolution operation is applied along the sample dimension, serving as a local information mixer. The choices of weight initialization and activation functions follow the same procedure as those employed for MLPs.

    \item \textbf{Decision trees}: These are employed to introduce rule-based mappings and can take the form of either classification or regression trees. Unlike some approaches that fit decision trees on random data, we argue that model fitting is neither necessary nor computationally efficient for the purpose of introducing rule-based dependencies. In this work, decision trees are fixed once constructed, with their hyperparameters sampled independently for each edge.
\end{itemize}

Once the child node receives the mapped values, they are aggregated using an aggregation function selected from one of three options: simple average, weighted average, or MLP-based aggregation. With respect to observational noise introduced at each edge, rather than adding noise of a fixed magnitude, we generate noise whose magnitude is scaled according to the distribution of each feature. To construct the DAG from all LCSs, we begin with a single-LCS DAG. Each time a new LCS is incorporated, it may connect to one or multiple existing LCSs in the DAG, whereby the child node of the new LCS is linked to the parent nodes of the target LCSs. This process is repeated until no LCSs remain to be added. In the resulting structure, all nodes with an in-degree of zero serve as the root nodes.

\subsection{Data Sampling}
Determining how to sample a subset of the generated data as the training set is also a critical challenge. Although random sampling is possible, the resulting training data often deviates significantly from being truly representative or good for training models. Thus it is necessary to devise an efficient strategy to sample high-quality training data for pretraining. 
We employ two sampling strategies: graph-aware sampling and solvability-aware sampling. Graph-aware sampling can be regarded as a more advanced variant of random sampling, where the key difference lies in its consideration of the graphical distribution of the sampled training data, which constrains the sampling space and renders it significantly smaller than in the case of random sampling. In contrast to graph-aware sampling, solvability-aware sampling aims to provide training data with varying degrees of solvability, thereby encouraging the model to achieve better generalization.
From this point, we divide the subsampled problems into three classes: high-solvability, moderate-solvability and low-solvability problems. We ensure that the sampling ratios of these classes follow a categorical distribution, each of whose parameters is sampled from a distinct Gaussian distribution. In practice, we alternate between the two sampling strategies according to a predefined probability.

\subsection{Task Adaptation}
During the data generation process, the sampled data may be intended for either classification or regression tasks, and the major difference is the processing of the target variable $y$, which is initially sampled as a continuous variable. 
For classification tasks, it is subsequently discretized into categorical variables. 
For regression tasks, the procedure differs slightly. If the sampled $y$ is already a discrete variable, a regression task will not be constructed; In cases where $y$ is continuous but clusters closely around a limited set of distinct values, an in-order transformation can be applied to achieve a more uniform distribution across the magnitude scale.

\section{Retrieval-based Ensemble}
\label{sec:retrieval}

We adopt an inference-time and retrieval-based ensemble strategy that leverages \ours's learned attention scores to upweight and select representative in-context samples and features without any additional training, so that we can further improve the performance of \ourA.

In terms of ensemble, we run multiple inference pipelines per dataset and aggregate the results. 
In each pipeline, we (i) randomly permute the feature columns and reorder the labels for categorical features or outcomes, and (ii) augment a subset of features with simple, schema-preserving transformations like quantile normalization, log-normal transformation, and high-energy SVD components. 
For classification, the number of inference pipelines is set to 4. For regression, it is set to 8.

In terms of retrieval, we perform two forward passes of \ourA~for each pipeline. The first pass is performed based on all in-context samples and is employed for retrieval. The second pass is performed based on the customized in-context samples retrieved in the first pass. 

The procedure of the first pass is as follows.
First, last-layer feature-level attention provides importance scores over features, which can be used as feature weights for subsequent sample selection. Concretely, for each test sample, we calculate $\rva_f\in\mathbb{R}^{d+1}$, which is the feature-level attention score between the outcome $y_{te}$ and the concatenation of $F$ features and outcome $(\rvx_{te},y_{te})$. 
Then, the module of last-layer sample-level cross-attention between test samples and in-context samples induces importance scores over in-context samples. Concretely, we calculate $\rva_s\in\mathbb{R}^{m_{ct}\times (d+1)}$, which is the sample-level cross-attention scores between the test sample $(\rvx_{te}, y_{te})$ and in-context samples $(\rvx_{ct}, \rvy_{ct})$.
Finally, for a given test sample, we retrieve in-context samples with highest reweighted attention scores as customized in-context samples for this test sample, and perform a forward propagation again for prediction. Concretely, we calculate $\rva_{sf}=\rva_{s}\rva_{f}\in\mathbb{R}^{m_{ct}}$, a weighted average of $\rva_s$ along the feature dimension with $\rva_f$ as the feature weights.

\begin{figure}[th]
    \centering
    \begin{subfigure}[b]{0.45\textwidth}
        \centering
        \includegraphics[width=\textwidth]{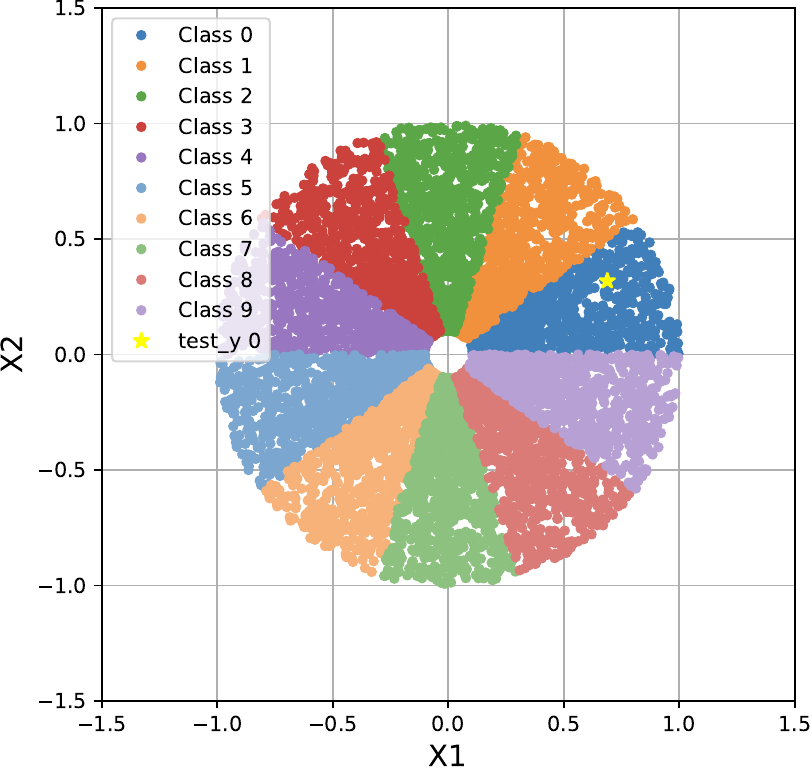}
        \caption{Visualization of synthetic samples. Each sector represents a category of in-context samples. The yellow pentagram represent the test sample. 
        }
        \label{fig:sample_level_total}
    \end{subfigure}
    \hfill
    \begin{subfigure}[b]{0.45\textwidth}
        \centering
        \includegraphics[width=\textwidth]{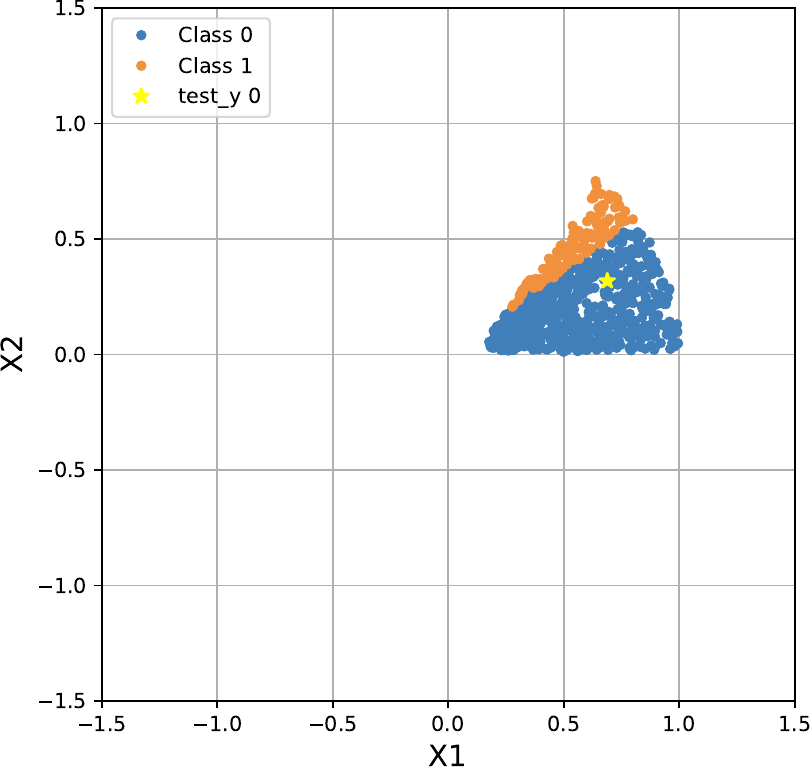}
        \caption{Top 10\% of in-context samples sorted by sample-level attention scores. \\ } 
        \label{fig:sample_level_top_N} 
    \end{subfigure}
    \caption{Toy example of sample-level attention. In-context samples that share the same category as the query sample are assigned higher scores through the attention module of \ourL.
    } 
    \label{fig:sample_level}
\end{figure}

\begin{figure}[th]
    \centering
    \begin{subfigure}[b]{0.48\textwidth}
        \centering
        \includegraphics[width=\textwidth]{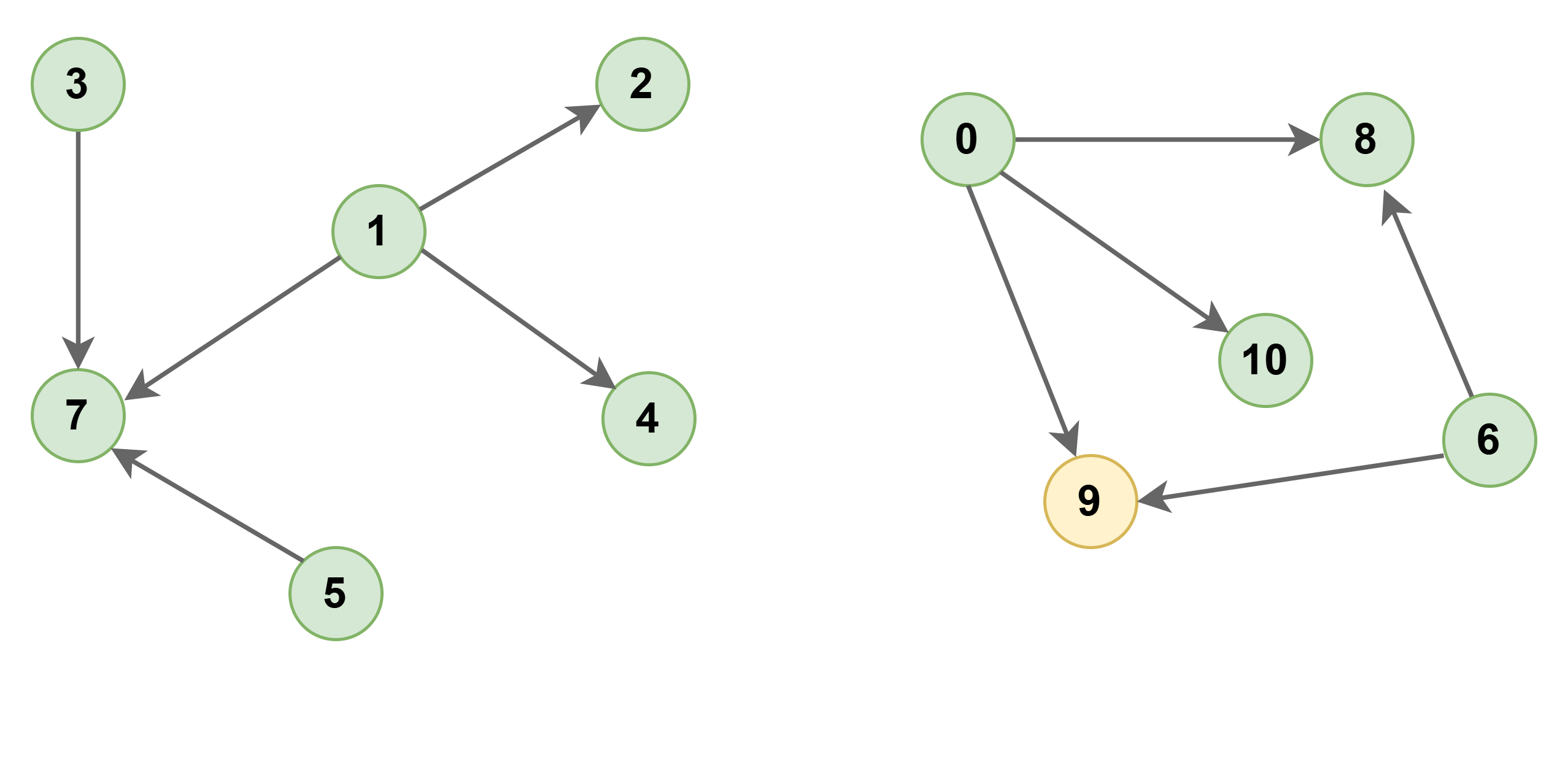}
        \caption{Causal DAG of synthetic samples.}
        \label{fig:feature_scm}
    \end{subfigure}
    \hfill
    \begin{subfigure}[b]{0.48\textwidth}
        \centering
        \includegraphics[width=\textwidth]{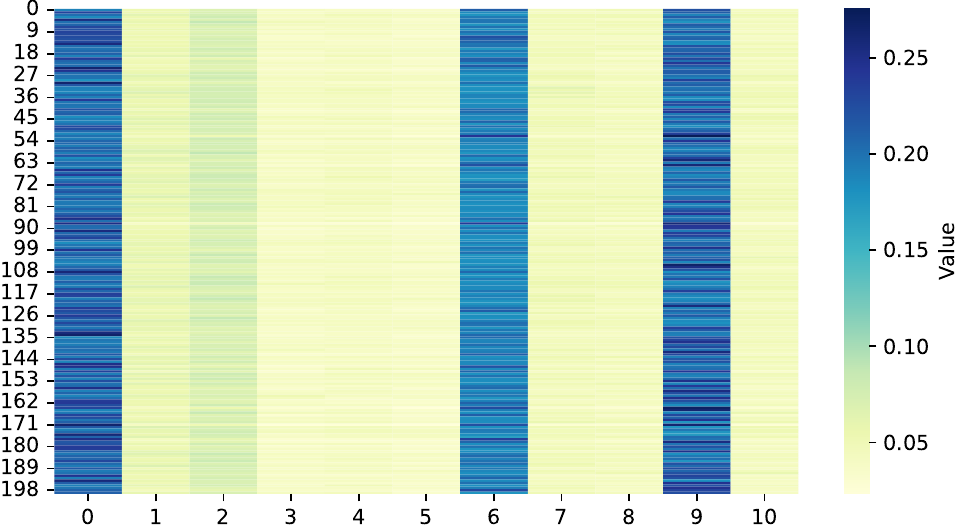}
        \caption{Heatmap of feature-level attention scores of $Y$. }
        \label{fig:feature_attn}
    \end{subfigure}
    \caption{Toy example of feature-level attention for the outcome variable. Direct causes of the outcome and the outcome itself are assigned almost all the attention weights in \ourL. 
    }
    \label{fig:feature_level}
\end{figure}

\paragraph{Toy examples.} We present the effect of bi-level attention-based retrieval via toy examples.

\begin{itemize}
    \item \textbf{Sample-level retrieval.}\quad We generate 2D synthetic data of 10 classes, where each class of data points occupy a sector of a circle in~\Cref{fig:sample_level_total}. 
    
    Since attention depicts sample similarity in the latent embedding space and it also takes the dependency between input features and the category label into account, it is capable of capturing more complex dependencies than naively using Euclidean distance in the original 2D space of features. 
    From \Cref{fig:sample_level_top_N}, we can see that the attention module predominantly assigns large weights to in-context instances sharing the same category label as the query instance. This indicates that the model leverages class-consistent contextual information from in-context samples to assist its prediction.
 
    \item \textbf{Feature-level retrieval.}\quad We generate synthetic data via the SCM in~\Cref{fig:feature_scm}, where each green node denotes an observed feature and the yellow node denotes the outcome. Each edge is a two-layer MLP using ReLU as the activation function added with Gaussian noise. 
    From \Cref{fig:feature_attn}, we can see that the attention module assigns most weights to a subset of features, which is exactly the set of direct causes of the outcome in the SCM. This suggests that feature-level attention could help to focus on causal features and reduce reliance on spurious correlations. 

\end{itemize}

\section{Theoretical Analysis on Context-Conditional Masked Modeling} \label{sect:theory}
We begin by introducing the mathematical formulation underlying our theoretical analysis.  

\paragraph{Notations.}
Let $\Omega$ denote the space for each feature in the table. We represent the $m$ in-context samples, each with $d$ features, by the random matrix $\bfXcontext \in \Omega^{m \times d}$. We use the random vector $\bfXtest = (\Xtest_1, \dots, \Xtest_d) \in \Omega^{d}$ to denote the test sample\footnote{
For simplicity, we assume (i) all features share the same space $\Omega$, though the results generalize to distinct spaces; (ii) we use $m$ instead of $m_{ct}$ in \Cref{sect:architecture}; and (iii) there is only one test sample, i.e., $m_{te} = 1$ in \Cref{sect:architecture}.}. Meanwhile, we do not explicitly introduce a target label $y$ in this section for simplicity; instead, we treat it as a particular dimension of both the in-context and test samples. Consequently, one dimension of $\bfXtest$ is unknown to the model during the test phase.

For any $\pi \subseteq [d]$, let $\bfXtest_\pi = (\Xtest_j)_{j \in \pi}$ and $\bfXtest_{-\pi} = (\Xtest_j)_{j \notin \pi}$ denote the subvectors on $\pi$ and its complement, respectively. In particular, for $\pi = \{j\}$ we use $\bfXtest_{-j}$ as shorthand for $\bfXtest_{-\{j\}}$.

\paragraph{Training and test procedures.}
We are given $n$ i.i.d. samples $\{(\bfxcontexti, \bfxtesti)\}_{i=1}^n$ drawn from $p(\bfXcontext, \bfXtest)$, given in \cref{sect:data-generation}. The model is trained using CCMM, a masked pre-training method (\cref{sect:pretraining}). Let $\Pi \subseteq 2^{[d]}$ be a set of masks and $\uniform(\Pi)$ the uniform distribution over $\Pi$. In practice, we take $\Pi = \{\pi \subseteq [d] : |\pi| \in [0.1d,0.4d]\}$, while for theory we simplify to masks of fixed size $k$:
\begin{equation} \label{eq:Pi-k}
    \Pi_k := \{ \pi \subseteq [d] : |\pi| = k \}.
\end{equation}
It is easy to verify that our theoretical insights extend naturally to settings with variable mask sizes.

For each sample $(\bfXcontext, \bfXtest)$, we draw $\pi \sim \uniform(\Pi_k)$ and train 
$q_{\theta}(\bfXtest_{\pi} | \bfXtest_{-\pi}, \bfXcontext)$, where $\theta \in \Theta$ 
and $\Theta$ denotes the parameter space of the model, to reconstruct the masked features. Let $\pi_i$ be the mask for $(\bfxcontexti, \bfxtesti)$. The empirical loss $\hat{L}_{k}(\theta)$ and estimator $\htheta_{k,n}$ are
\begin{equation} \label{eq:loss-empirical}
    \hat{L}_k(\theta) = \frac{1}{n} \sum_{i=1}^n -\log q_{\theta}\left(\bfxtesti_{\pi_i}|  \bfxtesti_{-\pi_i}, \bfxcontexti\right), 
    \qquad 
    \htheta_{k,n} = \argmin_{\theta \in \Theta} \hat{L}_k(\theta).
\end{equation}  
As $n \to \infty$, this converges to the population-level loss $L_k(\theta)$ and the corresponding solution $\theta^*_k$:
\begin{equation} \label{eq:loss-population}
    L_k(\theta) = \E_{(\bfXcontext, \bfXtest) \sim p,\, \pi \sim \uniform(\Pi_k)} 
    \left[ -\log q_{\theta}(\bfXtest_\pi | \bfXtest_{-\pi}, \bfXcontext) \right], 
    \qquad \theta^*_k = \argmin_{\theta \in \Theta} L_k(\theta).
\end{equation}

At test time, all features in $\bfXcontext$ are observed, while one feature $\Xtest_j$ of $\bfXtest$ is missing and must be inferred. The goal is to output $p(\Xtest_j \mid \bfXtest_{-j}, \bfXcontext)$.

\subsection{Modeling the Joint Conditional with Random Masks} \label{sect:joint-distribution}

We now explain the necessity of randomly sampling masks from $\uniform(\Pi_k)$. Our result is based on the following proposition.

\begin{proposition}[Informal; See \Cref{prop:identifiability-full}] \label{prop:identifiability}
    Under mild assumptions, for any $k \in [d]$, there is a one-to-one correspondence between the distribution $p(\bfXtest|\bfXcontext)$ and the family of conditionals $\{p(\bfXtest_{\pi}|\bfXtest_{-\pi}, \bfXcontext): \forall \pi \in \Pi_k\}$.
\end{proposition}

At test time, the target variable $y$ may correspond to any feature $\Xtest_j$ of the test sample $\bfXtest$, so the model must be able to estimate $p(\Xtest_j | \bfXtest_{-j}, \bfXcontext)$ for every $j \in [d]$. As shown in \Cref{prop:identifiability}, this requirement is equivalent to learning the joint conditional distribution $p(\bfXtest | \bfXcontext)$. Hence, a model can achieve strong predictive performance at test time if and only if it learns $p(\bfXtest | \bfXcontext)$. Moreover, if each $p(\bfXtest_{\pi} | \bfXtest_{-\pi}, \bfXcontext)$ can be well approximated by $q_{\theta}(\bfXtest_{\pi} | \bfXtest_{-\pi}, \bfXcontext)$ for all $\pi \in \Pi_k$, then the model can recover the joint conditional distribution and thereby generalize effectively. By contrast, \Cref{example:identifiability} shows that knowledge of $p(\Xtest_{\pi} | \bfXtest_{-\pi}, \bfXcontext)$ for only a subset of $\pi \in \Pi_k$ may be insufficient to recover $p(\bfXtest | \bfXcontext)$, preventing the model from generalizing effectively at test time.

Moreover, the target distribution $p(\bfXtest | \bfXcontext)$ is closely connected to the underlying structural causal models (SCMs). Let $\SCM$ denote a random variable corresponding to a structural causal model (SCM). By the data-generating process, we have $\bfXcontext \perp \bfXtest \mid S$, which yields
\begin{equation} \label{eq:conditional-target}
    p(\bfXtest | \bfXcontext) 
    = \int_{\SCM} p(\bfXtest, S | \bfXcontext) \, \mathrm{d}\SCM 
    = \int_{\SCM} p(S | \bfXcontext) \, p(\bfXtest | S, \bfXcontext) \, \mathrm{d}\SCM 
    = \int_{\SCM} p(S | \bfXcontext) \, p(\bfXtest | S) \, \mathrm{d}\SCM.
\end{equation}
Intuitively, \cref{eq:conditional-target} suggests that one efficient way to learn $p(\bfXtest | \bfXcontext)$ is through two components: (i) $p(S | \bfXcontext)$, the posterior distribution over SCMs given the context, and (ii) $p(\bfXtest | S)$, the likelihood of the test sample under a given SCM.

\subsection{The Choice of Mask Number} \label{sect:mask-number}

In this subsection, we demonstrate that the choice $k$ in \cref{eq:Pi-k} has great impact on models' performances and we choose $k > 1$ due to both sample efficiency and generalization considerations.
Our result is an extension from the analysis by \citet{li2024promises} on the properties of masked sequence prediction.

Based on \Cref{prop:identifiability}, we will henceforth use $q_{\theta}(\bfXtest | \bfXcontext)$ to denote the distributions induced by the learned family of conditional probabilities $\{q_{\theta}(\bfXtest_{\pi} | \bfXtest_{-\pi}, \bfXcontext) : \pi \in \Pi_k\}$.

\paragraph{Sample efficiency.} Theorem below shows that larger $k$ yields lower estimation uncertainty.

\begin{theorem}[Informal; see \Cref{thrm:sample-efficiency-full}] \label{thrm:sample-efficiency}
    Suppose there exists $\theta^* \in \Theta$ such that $q_{\theta^*}(\bfXtest_{\pi} | \bfXtest_{-\pi}, \bfXcontext) = p(\bfXtest_{\pi} | \bfXtest_{-\pi}, \bfXcontext)$ for all $\bfXcontext \in \Omega^{m \times d}$, $\bfXtest \in \Omega^d$, and $\pi \subseteq [d]$, and that the minimizer of $L_k(\theta)$ is unique for every $k$. Then, under mild regularity conditions, as $n \to \infty$,
    \[
        \sqrt{n}\,(\htheta_{k,n} - \theta^*) \xrightarrow{d} \gN(0, \Gamma_k),
    \]
    where $\Gamma_k$ does not depend on $n$ and satisfies $\Gamma_{k+1} \preceq \Gamma_k$.
\end{theorem}

The assumption on $\theta^*$ states that the optimal solution of the model can approximate any conditional distribution, which is reasonable given the expressive power of transformer-based architectures. This theorem further shows that for sufficiently large and fixed $n$, if $k_1 > k_2$ then $\Gamma_{k_1} \preceq \Gamma_{k_2}$, implying that $\htheta_{k_1,n}$ has lower estimation uncertainty than $\htheta_{k_2,n}$. Equivalently, when $k$ is larger, fewer samples are required to achieve the same uncertainty, leading to greater sample efficiency.

\paragraph{Generalization for joint distribution learning.} Masked pretraining empowers the model to infer the joint distribution of data from the context of training samples. We further show that an increase in the density of random masks enables a more accurate reconstruction of the joint distribution.

\begin{theorem}[Informal; See \Cref{theo:joint_dist-full}] 
\label{theo:joint_dist}
Under regularity conditions on $q_\theta$, with high probability, for any $\theta \in \Theta$, the expected total variation between $q_\theta(\bfXtest|\bfXcontext)$ and the true conditional distribution $p(\bfXtest|\bfXcontext)$ is at most:
\[
\sqrt{ \frac{1}{2}C_k(q_\theta)\big(\hat L_k(\theta) + \text{complexity terms}\big) + \gO(n^{-1/2})},
\]
where constants $C_k(q_\theta)$ depend only on $q_\theta$ and $k$. Furthermore, $C_{k+1}(q_\theta) \leq C_k(q_\theta)$ for any $\theta\in \Theta$.
\end{theorem}

The theorem establishes an upper bound for the generalization error of the estimated conditional joint distribution $q_{\theta}(\bfXtest|\bfXcontext)$ compared to the true joint distribution $p(\bfXtest|\bfXcontext)$, and shows that the upper bound decreases monotonically with respect to the number of masked cells.

\section{Evaluation}

\subsection{Classification}

\paragraph{Benchmarks.} For the quantitative evaluation of classification performance, we utilize multiple benchmarks, including TALENT-CLS~\citep{liu2024talenttabularanalyticslearning}, OpenML-CC18~\citep{bischl2017openml}, PFN-CLS~\citep{hollmann2025tabpfn_nature}, TabZilla~\citep{mcelfresh2023neuralzilla}, and TabArena-CLS~\citep{erickson2025tabarena}. Among these benchmarks, the datasets containing more than 50,000 training samples (The number of testing samples is not constrained), 10,000 features, or 10 target categories were excluded. This selection process resulted in a final collection of 179 datasets from TALENT-CLS, 62 from OpenML-CC18, 29 from PFN-CLS, 27 from TabZilla, and 33 from TabArena-CLS.

Furthermore, we introduce \textbf{B}alanced \textbf{C}omprehensive \textbf{C}hallenging \textbf{O}mni-domain (BCCO) Benchmark, comprising BCCO-CLS and BCCO-REG, for the evaluation of \ourA~and baseline models. The BCCO benchmark is constructed from extensive open-source structured-data corpora, meticulously deduplicated and cleaned. It presents a significant challenge due to several intrinsic characteristics: the distribution of dataset attributes (such as the ratio of categorical features), and the diversity of real-world prediction targets. Unlike previous benchmarks, nearly one-third of the datasets in our BCCO benchmark contain missing values. The BCCO benchmarks are publicly available at \url{https://huggingface.co/datasets/stableai-org/bcco_cls} and \url{https://huggingface.co/datasets/stableai-org/bcco_reg}.

The BCCO-CLS benchmark comprising 106 datasets. The collection spans diverse sources, domains, and scales, and is designed to cover a wide range of problem characteristics, including the number of samples, number of features, number of classes, categorical–to-numerical feature ratio, sample-to-feature ratio, and proportion of missing values. This diversity allows for uniform binning along these dimensions, enabling results to be reported within bins and macro-averaged across bins, thereby providing a nearly unbiased assessment of model performance across heterogeneous task regimes. Additionally, datasets containing more than 50,000 training samples (The number of testing samples is not constrained), 10,000 features, or 10 target categories were excluded.

For each dataset in these benchmarks, we use the provided train-test split when available. If no predefined test set exists, the data are partitioned into a 70\% training set and a 30\% test set using stratified sampling to preserve the label distribution.

\vspace{-5pt}
\paragraph{Baselines.} We compare \ourA~with a range of state-of-the-art baseline models, categorized into tree-based models, neural networks (NN), and recent ICL-based approaches. 
\vspace{-5pt}

\begin{itemize}
    \item \textbf{Tree-based approaches.}\quad We include XGBoost~\citep{chen2016xgboost}, LightGBM~\citep{ke2017lightgbm}, 
    CatBoost~\citep{dorogush2018catboost},   
    Random Forest (RF)~\citep{breiman2001random}, and Extra Trees (ET)~\citep{geurts2006extremely}. All models are optimized using the Optuna~\citep{akiba2019optuna} framework via 5-fold stratified cross-validation, with hyperparameters sampled from the ranges specified in \cref{appendix:searchspace}. Additionally, for AutoGluon-Tabular~\citep{erickson2020autogluon}, which automates workflows of model searching and ensemble, we use the default search space and set a default 600s time constraint for hyperparameter searching for each dataset. 
    \item \textbf{NN-based approaches.}\quad We evaluate against SNN~\citep{klambauer2017self},
    AutoInt~\citep{song2019autoint}, NODE~\citep{popov2019neuralnode}, 
    TabTransformer~\citep{huang2020tabtransformer}, 
    GrowNet~\citep{badirli2020gradient},
    TabNet~\citep{arik2021tabnet}, 
    DCN-v2~\citep{wang2021dcn},  
    FT-Transformer~\citep{gorishniy2021revisiting}, 
    MLP~\citep{goodfellow2016deep,gorishniy2021revisiting},
    ResNet~\citep{he2016deepresnet,gorishniy2021revisiting}, 
    SAINT~\citep{somepalli2021saint},
    MLP-PLR~\citep{gorishniy2022embeddingsmlpplr}, 
    DANets~\citep{chen2022danets},
    TANGOS~\citep{jeffares2023tangos}, 
    T2G-Former~\citep{yan2023t2g}, 
    ExcelFormer~\citep{chen2023excelformer},
    Trompt~\citep{chen2023trompt},
    TabR~\citep{gorishniy2023tabr},
    TabCaps~\citep{chen2023tabcaps},
    RealMLP~\citep{holzmuller2024betterrealmlp}, 
    SwitchTab~\citep{wu2024switchtab},
    and ModernNCA~\citep{ye2024revisitingmodernnca}.
    These NN-based models are trained using the TALENT~\citep{liu2024talenttabularanalyticslearning} Toolbox. 
    \item \textbf{ICL-based models.}\quad Recent baselines includes TabPFN-v2~\citep{hollmann2025tabpfn_nature}, 
    TabICL~\citep{qu2025tabicl},
    and Mitra~\citep{exampleWebsite}.
\end{itemize}
Since TabDPT~\citep{ma2024tabdpt}, another recent ICL-based model, is pretrained on real data that has a large overlap with datasets in the benchmarks, we do not adopt it as a baseline in the evaluation of standard classification and regression for a fair comparison. 
However, we include it in our experiments of fine-tuning in \Cref{sec:ft} to show the superiority of \ourA.

\paragraph{Metrics.} To evaluate model performance, we employ ROC AUC (area under the receiver operating characteristic curve), accuracy, and F1 score for classification tasks. For multi-class classification, the One-vs-One strategy is applied to both ROC AUC and F1 score calculations.

The critical difference diagram is also employed to compare the performance differences between \ourA~and the baseline models. We conducted a Friedman test followed by a post-hoc Wilcoxon-Holm test, using a significance level of 0.05. In the diagram, the horizontal line indicates the range of ranks among which differences are not statistically significant.
                
\paragraph{Results.} \Cref{fig:classification_wf} shows that \ourL~achieves the best performance among all methods on most datasets in BCCO-CLS, followed closely by \ourS. The most competitive baselines include other ICL-based models TabICL and TabPFN-v2, and the AutoML ensemble framework AutoGluon. 
In all critical difference diagrams, i.e. \Cref{fig:cdd106,fig:cdd_talent,fig:cc18,fig:tabarena,fig:cdd_pfn,fig:cddtabzilla}, we can see that \ourL~achieves the highest ranking in terms of all three evaluation metrics on BCCO-CLS, OpenML-CC18, and TALENT-CLS, where critical differences can be observed in most cases, indicating a significant margin. Despite operating under limited computational resources, \ourS~still delivers highly competitive results.

For detailed quantitative results listed in \Cref{tab:results_cls_BCCO,tab:talent-cls,tab:results_cls_CC18,tab:results_cls_TabArena,tab:results_cls_TabPFN-v2 classifier,tab:results_tabzilla}, \ourL~outperforms all baselines on every benchmark in terms of both mean and rank of all three metrics. All these results clearly demonstrate that \ourL~achieves state-of-the-art performance in terms of tabular classification tasks, surpassing not only traditional ensemble methods but also advanced ICL-based models. Across all classification benchmarks, \ourS~also outperforms other strong baselines, despite operating under highly limited computational resources.

\begin{figure}[H]
    \centering
    \includegraphics[width=0.95\linewidth]{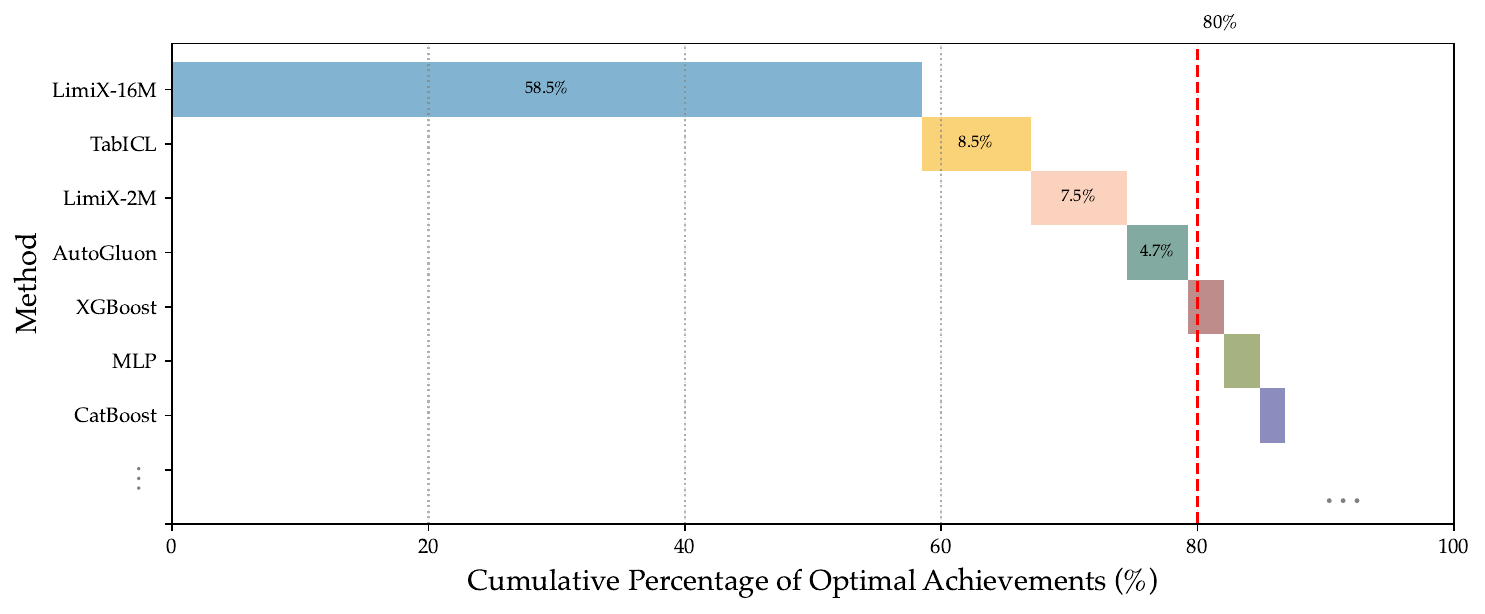}
    \caption{The proportion of models achieving the best AUC. The length of each bar represents the proportion of the 106 datasets in BCCO-CLS where a given method achieves the highest AUC.
    }
    \label{fig:classification_wf}
\end{figure}

\paragraph{Subgroup analysis.} We use the sample subgroups when building BCCO-CLS to perform stratified analyses. Subgroups are defined based on the following criteria: the type of classification (binary or multi-class), the number of training samples, the ratio between the number of samples and features (length-to-width ratio), the proportion of categorical features, and the presence or absence of missing values. The number of training samples, categorical feature ratio, and length-to-width ratio ared discretized into terciles (equal-frequency bins), ensuring that approximately the same number of datasets is allocated to each stratum and that the distribution of the dataset attribute which is not considered during analysis remains nearly uniform across strata.

From \Cref{fig:Subgroup analysis}, we observe that \ourL~exhibits leading performance across all subgroups compared with other methods. Notably, after stratification, under some subgroups like the third subgroup in \Cref{fig:barplot_group-sub3} that indicates a larger training sample size, AutoGluon proves to be a strong competitor to ICL-based models. In these cases, \ourL~is the only ICL-based model that outperforms AutoGluon.
Notably, \Cref{fig:barplot_group-sub6} shows that, as the proportion of categorical features increases, performance for most baselines drops rapidly, reflecting the sparsity and high-cardinality challenges. In contrast, \ourL~exhibits only modest degradation and remains comparatively stable across these regimes.

\begin{table}[H]
\centering
\renewcommand{\arraystretch}{1.05}
\setlength{\tabcolsep}{12pt} 
\caption{Classification results on the BCCO-CLS benchmark. The best scores are shown in bold.}
\label{tab:results_cls_BCCO}
\begin{adjustbox}{max width=0.95\textwidth}
\begin{tabular}{l|cccccc}
\toprule
& \multicolumn{6}{c}{\rule{0pt}{20pt}\raisebox{5pt}{BCCO-CLS}} \\ \midrule
& \multicolumn{3}{c|}{Mean} & \multicolumn{3}{c}{Rank} \\
\cmidrule(lr){2-4}\cmidrule(lr){5-7}
\multicolumn{1}{l|}{Model} & AUC ($\uparrow$) & Acc. ($\uparrow$) & \multicolumn{1}{c|}{F1 ($\uparrow$)} & AUC ($\downarrow$) & Acc. ($\downarrow$) & F1 ($\downarrow$) \\
\midrule
\multicolumn{1}{l|}{LimiX-16M} & \textbf{0.871} & \textbf{0.804} & \multicolumn{1}{c|}{\textbf{0.731}} & \textbf{2.642} & \textbf{2.245} & \textbf{2.887} \\
\multicolumn{1}{l|}{LimiX-2M} & 0.858 & 0.787 & \multicolumn{1}{c|}{0.701} & 6.170 & 6.481 & 8.538 \\
\multicolumn{1}{l|}{TabICL} & 0.847 & 0.768 & \multicolumn{1}{c|}{0.672} & 7.255 & 8.726 & 10.783 \\
\multicolumn{1}{l|}{AutoGluon} & 0.846 & 0.771 & \multicolumn{1}{c|}{0.677} & 8.500 & 8.500 & 9.934 \\
\multicolumn{1}{l|}{TabPFN-v2} & 0.843 & 0.772 & \multicolumn{1}{c|}{0.679} & 9.208 & 9.821 & 11.396 \\
\multicolumn{1}{l|}{XGBoost} & 0.834 & 0.762 & \multicolumn{1}{c|}{0.674} & 10.660 & 10.849 & 11.547 \\
\multicolumn{1}{l|}{Mitra} & 0.836 & 0.764 & \multicolumn{1}{c|}{0.664} & 10.953 & 11.698 & 13.585 \\
\multicolumn{1}{l|}{LightGBM} & 0.832 & 0.763 & \multicolumn{1}{c|}{0.678} & 11.594 & 10.811 & 11.406 \\
\multicolumn{1}{l|}{RF} & 0.829 & 0.756 & \multicolumn{1}{c|}{0.652} & 12.170 & 12.406 & 13.726 \\
\multicolumn{1}{l|}{CatBoost} & 0.829 & 0.757 & \multicolumn{1}{c|}{0.664} & 12.792 & 11.755 & 12.557 \\
\multicolumn{1}{l|}{ET} & 0.825 & 0.745 & \multicolumn{1}{c|}{0.618} & 13.292 & 14.745 & 17.585 \\
\multicolumn{1}{l|}{FT-Transformer} & 0.813 & 0.744 & \multicolumn{1}{c|}{0.642} & 14.340 & 14.708 & 15.104 \\
\multicolumn{1}{l|}{T2G-Former} & 0.808 & 0.742 & \multicolumn{1}{c|}{0.646} & 15.321 & 15.170 & 14.377 \\
\multicolumn{1}{l|}{ModernNCA} & 0.815 & 0.752 & \multicolumn{1}{c|}{0.658} & 16.217 & 16.434 & 16.047 \\
\multicolumn{1}{l|}{MLP-PLR} & 0.804 & 0.733 & \multicolumn{1}{c|}{0.635} & 16.566 & 16.915 & 16.113 \\
\multicolumn{1}{l|}{ExcelFormer} & 0.810 & 0.742 & \multicolumn{1}{c|}{0.655} & 16.660 & 16.009 & 14.566 \\
\multicolumn{1}{l|}{TabR} & 0.809 & 0.750 & \multicolumn{1}{c|}{0.657} & 17.274 & 15.057 & 14.104 \\
\multicolumn{1}{l|}{DCN-v2} & 0.794 & 0.725 & \multicolumn{1}{c|}{0.618} & 17.783 & 17.736 & 17.453 \\
\multicolumn{1}{l|}{ResNet} & 0.800 & 0.728 & \multicolumn{1}{c|}{0.641} & 17.783 & 17.943 & 16.670 \\
\multicolumn{1}{l|}{TANGOS} & 0.799 & 0.731 & \multicolumn{1}{c|}{0.641} & 18.057 & 17.717 & 16.151 \\
\multicolumn{1}{l|}{RealMLP} & 0.794 & 0.738 & \multicolumn{1}{c|}{0.644} & 19.132 & 17.236 & 15.934 \\
\multicolumn{1}{l|}{MLP} & 0.787 & 0.720 & \multicolumn{1}{c|}{0.614} & 19.387 & 18.764 & 19.594 \\
\multicolumn{1}{l|}{SAINT} & 0.791 & 0.726 & \multicolumn{1}{c|}{0.623} & 19.604 & 18.415 & 17.066 \\
\multicolumn{1}{l|}{AutoInt} & 0.779 & 0.718 & \multicolumn{1}{c|}{0.601} & 21.113 & 20.292 & 20.660 \\
\multicolumn{1}{l|}{DANets} & 0.771 & 0.705 & \multicolumn{1}{c|}{0.601} & 21.302 & 20.377 & 20.708 \\
\multicolumn{1}{l|}{TabTransformer} & 0.762 & 0.699 & \multicolumn{1}{c|}{0.566} & 21.557 & 20.698 & 21.170 \\
\multicolumn{1}{l|}{SNN} & 0.773 & 0.708 & \multicolumn{1}{c|}{0.584} & 21.585 & 21.160 & 21.415 \\
\multicolumn{1}{l|}{SwitchTab} & 0.766 & 0.700 & \multicolumn{1}{c|}{0.590} & 22.340 & 22.849 & 21.764 \\
\multicolumn{1}{l|}{TabCaps} & 0.744 & 0.701 & \multicolumn{1}{c|}{0.580} & 24.670 & 22.632 & 23.019 \\
\multicolumn{1}{l|}{NODE} & 0.754 & 0.695 & \multicolumn{1}{c|}{0.531} & 24.877 & 23.443 & 26.198 \\
\multicolumn{1}{l|}{TabNet} & 0.712 & 0.685 & \multicolumn{1}{c|}{0.561} & 28.123 & 25.547 & 25.679 \\
\multicolumn{1}{l|}{GrowNet} & 0.682 & 0.641 & \multicolumn{1}{c|}{0.522} & 28.679 & 28.330 & 27.321 \\
\bottomrule
\end{tabular}
\end{adjustbox}
\end{table}

\begin{table}[H]
\centering
\caption{Statistical profile of the benchmark BCCO-CLS, where Q10, Q50, and Q90 correspond to the 10\%, 50\%, and 90\% quantiles, respectively; for categorical feature statistics, we only consider features that are either string-typed or have fewer than 10 unique values.}
\resizebox{0.9\textwidth}{!}{
    \begin{tabular}{cl ccccccc}
    \toprule
    & & \multicolumn{7}{c}{Statistics} \\
    \cmidrule(lr){3-9}
     & Metric   & Q10  & Q50 & Q90 & Mean & Std  & Min & Max \\
    \midrule
    & \# Features & 4 & 12 & 48 & 23 & 33 & 1 & 259 \\
    & \# Classes & 2 & 2 & 5 & 3 & 2 & 2 & 10 \\
    & Missing Values (Ratio) & 0 & 0 & 0.069 & 0.024 & 0.068 & 0 & 0.403 \\
    & Categorical Features (Ratio) & 0 & 0.384 & 0.929 & 0.399 & 0.337 & 0 & 1 \\
    & Features w/ Missing Values (Ratio) & 0 & 0 & 0.523 & 0.126 & 0.241 & 0 & 0.909 \\
    \bottomrule
    \end{tabular}
}
\label{tab:profile_cls_BCCO}
\end{table}

\begin{figure}[H]
    \centering
    \begin{subfigure}[b]{\textwidth}
        \centering
        \includegraphics[width=\textwidth]{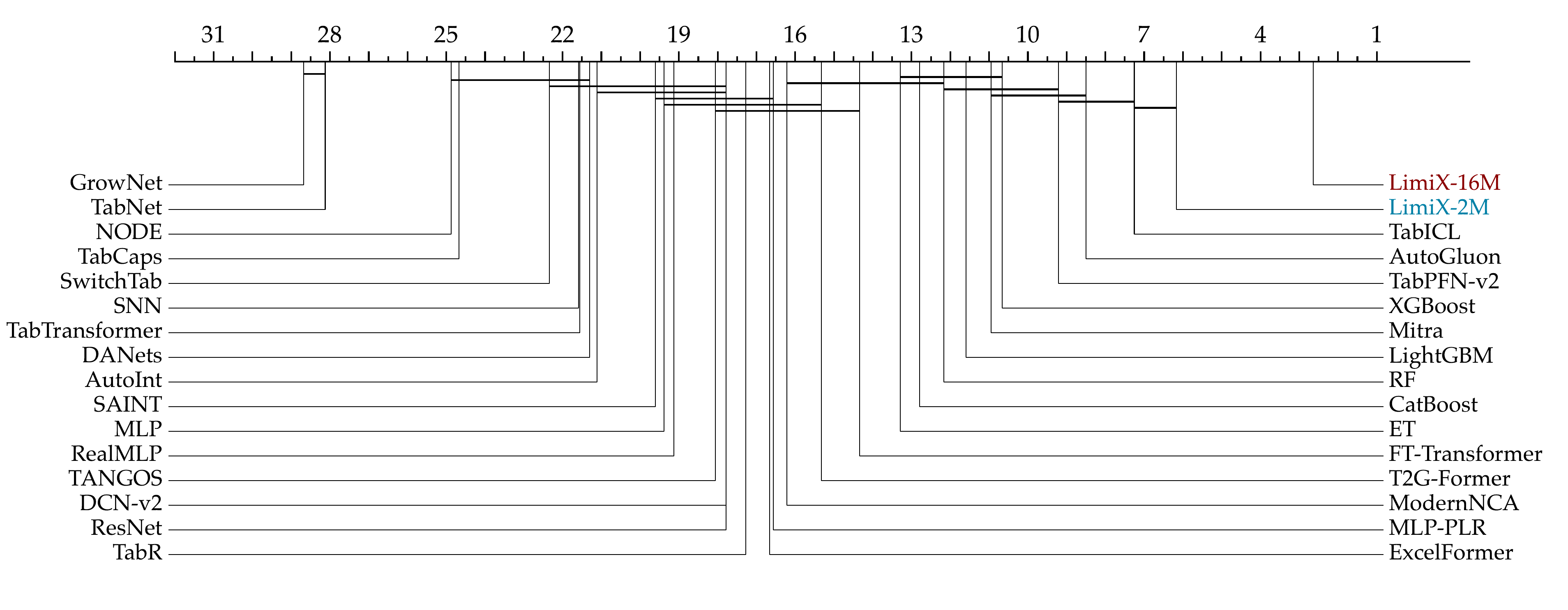}
        \caption{AUC on the BCCO-CLS benchmark.}
        \label{fig:cdd106-sub1}
    \end{subfigure}
    
    \vspace{2em}
    
    \begin{subfigure}[b]{\textwidth}
        \centering
           \includegraphics[width=\textwidth]{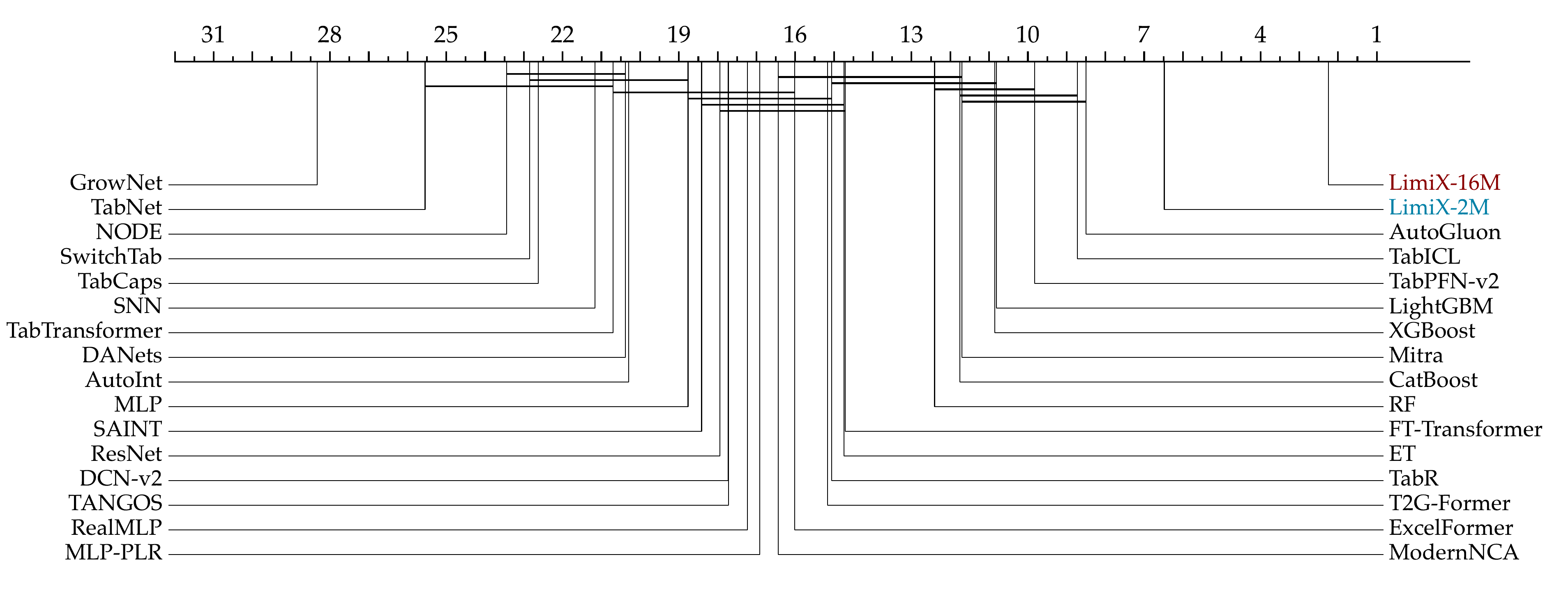}
        \caption{Accuracy on the BCCO-CLS benchmark.}
        \label{fig:cdd106-sub2}
    \end{subfigure}
    
    \vspace{2em}

    \begin{subfigure}[b]{\textwidth}
        \centering
        \includegraphics[width=\textwidth]{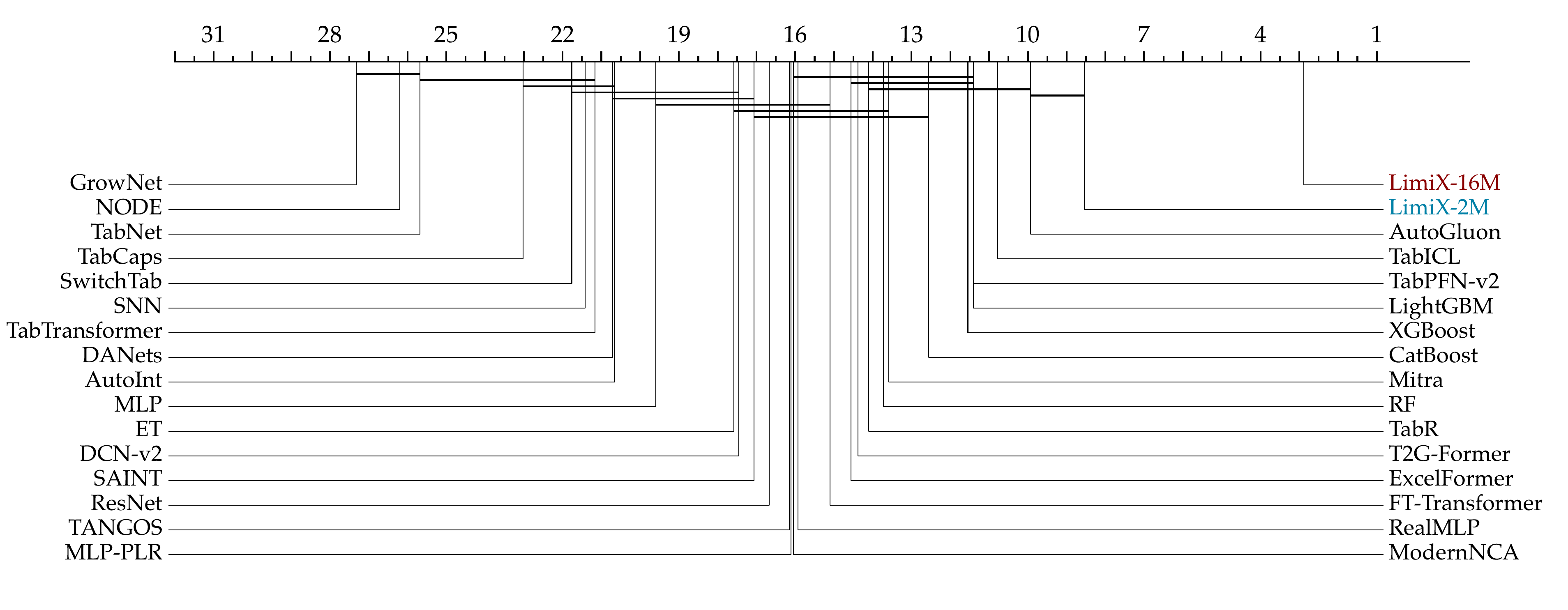}
        \caption{F1-score on the BCCO-CLS benchmark.}
        \label{fig:cdd106-sub3}
    \end{subfigure}
    
    \caption{Critical difference diagrams on BCCO-CLS benchmark.}
    \label{fig:cdd106}
\end{figure}

\clearpage
\begin{table}[H]
\centering
\renewcommand{\arraystretch}{1.05}
\setlength{\tabcolsep}{12pt} 
\caption{Classification results on the TALENT-CLS benchmark. The best scores are shown in bold.}
\label{tab:talent-cls}
\begin{adjustbox}{max width=0.95\textwidth}
\begin{tabular}{l|cccccc}
\toprule
& \multicolumn{6}{c}{\rule{0pt}{20pt}\raisebox{5pt}{TALENT-CLS}} \\ \midrule
& \multicolumn{3}{c|}{Mean} & \multicolumn{3}{c}{Rank} \\
\cmidrule(lr){2-4}\cmidrule(lr){5-7}
\multicolumn{1}{l|}{Model} & AUC ($\uparrow$) & Acc. ($\uparrow$) & \multicolumn{1}{c|}{F1 ($\uparrow$)} & AUC ($\downarrow$) & Acc. ($\downarrow$) & F1 ($\downarrow$) \\
\midrule
\multicolumn{1}{l|}{LimiX-16M} & \textbf{0.903} & \textbf{0.861} & \multicolumn{1}{c|}{\textbf{0.752}} & \textbf{4.073} & \textbf{3.190} & \textbf{4.246} \\
\multicolumn{1}{l|}{LimiX-2M} & 0.897 & 0.853 & \multicolumn{1}{c|}{0.734} & 5.804 & 5.642 & 6.994 \\
\multicolumn{1}{l|}{TabICL} & 0.894 & 0.845 & \multicolumn{1}{c|}{0.715} & 6.073 & 7.028 & 8.514 \\
\multicolumn{1}{l|}{TabPFN-v2} & 0.895 & 0.850 & \multicolumn{1}{c|}{0.727} & 6.670 & 6.458 & 8.296 \\
\multicolumn{1}{l|}{AutoGluon} & 0.891 & 0.845 & \multicolumn{1}{c|}{0.719} & 6.899 & 7.291 & 8.067 \\
\multicolumn{1}{l|}{XGBoost} & 0.881 & 0.837 & \multicolumn{1}{c|}{0.713} & 10.034 & 10.123 & 10.425 \\
\multicolumn{1}{l|}{LightGBM} & 0.880 & 0.836 & \multicolumn{1}{c|}{0.713} & 10.307 & 10.475 & 10.777 \\
\multicolumn{1}{l|}{Mitra} & 0.882 & 0.834 & \multicolumn{1}{c|}{0.689} & 11.162 & 11.849 & 13.916 \\
\multicolumn{1}{l|}{CatBoost} & 0.876 & 0.828 & \multicolumn{1}{c|}{0.704} & 12.302 & 12.302 & 12.508 \\
\multicolumn{1}{l|}{RF} & 0.877 & 0.828 & \multicolumn{1}{c|}{0.691} & 13.179 & 13.637 & 14.754 \\
\multicolumn{1}{l|}{ET} & 0.875 & 0.821 & \multicolumn{1}{c|}{0.662} & 13.877 & 16.229 & 17.788 \\
\multicolumn{1}{l|}{ExcelFormer} & 0.870 & 0.826 & \multicolumn{1}{c|}{0.699} & 14.212 & 14.855 & 13.749 \\
\multicolumn{1}{l|}{ResNet} & 0.866 & 0.825 & \multicolumn{1}{c|}{0.695} & 15.793 & 15.514 & 13.955 \\
\multicolumn{1}{l|}{FT-Transformer} & 0.859 & 0.822 & \multicolumn{1}{c|}{0.678} & 16.536 & 16.715 & 16.413 \\
\multicolumn{1}{l|}{T2G-Former} & 0.858 & 0.823 & \multicolumn{1}{c|}{0.683} & 16.615 & 16.374 & 15.933 \\
\multicolumn{1}{l|}{TANGOS} & 0.861 & 0.818 & \multicolumn{1}{c|}{0.684} & 16.905 & 16.810 & 15.542 \\
\multicolumn{1}{l|}{MLP} & 0.862 & 0.817 & \multicolumn{1}{c|}{0.675} & 17.302 & 16.989 & 17.162 \\
\multicolumn{1}{l|}{ModernNCA} & 0.861 & 0.825 & \multicolumn{1}{c|}{0.683} & 17.955 & 17.693 & 16.966 \\
\multicolumn{1}{l|}{TabR} & 0.858 & 0.824 & \multicolumn{1}{c|}{0.680} & 18.263 & 16.804 & 15.922 \\
\multicolumn{1}{l|}{DCN-v2} & 0.854 & 0.815 & \multicolumn{1}{c|}{0.662} & 18.704 & 18.849 & 18.687 \\
\multicolumn{1}{l|}{MLP-PLR} & 0.849 & 0.816 & \multicolumn{1}{c|}{0.663} & 19.385 & 18.168 & 18.201 \\
\multicolumn{1}{l|}{DANets} & 0.848 & 0.805 & \multicolumn{1}{c|}{0.654} & 19.961 & 19.721 & 19.374 \\
\multicolumn{1}{l|}{RealMLP} & 0.839 & 0.820 & \multicolumn{1}{c|}{0.678} & 21.106 & 18.525 & 16.793 \\
\multicolumn{1}{l|}{SAINT} & 0.813 & 0.781 & \multicolumn{1}{c|}{0.630} & 21.162 & 20.274 & 19.715 \\
\multicolumn{1}{l|}{AutoInt} & 0.842 & 0.803 & \multicolumn{1}{c|}{0.646} & 21.492 & 22.156 & 21.587 \\
\multicolumn{1}{l|}{TabCaps} & 0.834 & 0.813 & \multicolumn{1}{c|}{0.654} & 22.257 & 19.106 & 19.570 \\
\multicolumn{1}{l|}{SwitchTab} & 0.842 & 0.795 & \multicolumn{1}{c|}{0.637} & 22.274 & 22.777 & 21.989 \\
\multicolumn{1}{l|}{TabTransformer} & 0.832 & 0.790 & \multicolumn{1}{c|}{0.627} & 22.637 & 22.598 & 21.760 \\
\multicolumn{1}{l|}{SNN} & 0.836 & 0.796 & \multicolumn{1}{c|}{0.625} & 22.933 & 22.855 & 22.972 \\
\multicolumn{1}{l|}{NODE} & 0.830 & 0.779 & \multicolumn{1}{c|}{0.570} & 23.810 & 24.145 & 25.726 \\
\multicolumn{1}{l|}{TabNet} & 0.818 & 0.794 & \multicolumn{1}{c|}{0.630} & 25.983 & 23.553 & 23.983 \\
\multicolumn{1}{l|}{GrowNet} & 0.743 & 0.704 & \multicolumn{1}{c|}{0.542} & 28.430 & 27.933 & 25.933 \\
\bottomrule
\end{tabular}
\end{adjustbox}
\end{table}

\begin{table}[H]
\centering
\caption{Statistical profile of the benchmark TALENT-CLS, where Q10, Q50, and Q90 correspond to the 10\%, 50\%, and 90\% quantiles, respectively; for categorical feature statistics, we only consider features that are either string-typed or have fewer than 10 unique values.}
\resizebox{0.9\textwidth}{!}{
    \begin{tabular}{cl ccccccc}
    \toprule
    & & \multicolumn{7}{c}{Statistics} \\
    \cmidrule(lr){3-9}
     & Metric   & Q10  & Q50 & Q90 & Mean & Std  & Min & Max \\
    \midrule
    & \# Features & 7 & 19 & 70 & 33 & 47 & 3 & 308 \\
    & \# Classes & 2 & 2 & 10 & 6 & 14 & 2 & 100 \\
    & Missing Values (Ratio) & 0 & 0 & 0 & 0.001 & 0.008 & 0 & 0.1 \\
    & Categorical Features (Ratio) & 0 & 0.121 & 0.972 & 0.306 & 0.365 & 0 & 1 \\
    & Features w/ Missing Values (Ratio) & 0 & 0 & 0 & 0.025 & 0.121 & 0 & 0.979 \\
    \bottomrule
    \end{tabular}
}
\label{tab:profile_cls_TALENT}
\end{table}


    

 
    

\begin{figure}[H]
    \centering
    \begin{subfigure}[b]{\textwidth}
        \centering
        \includegraphics[width=\textwidth]{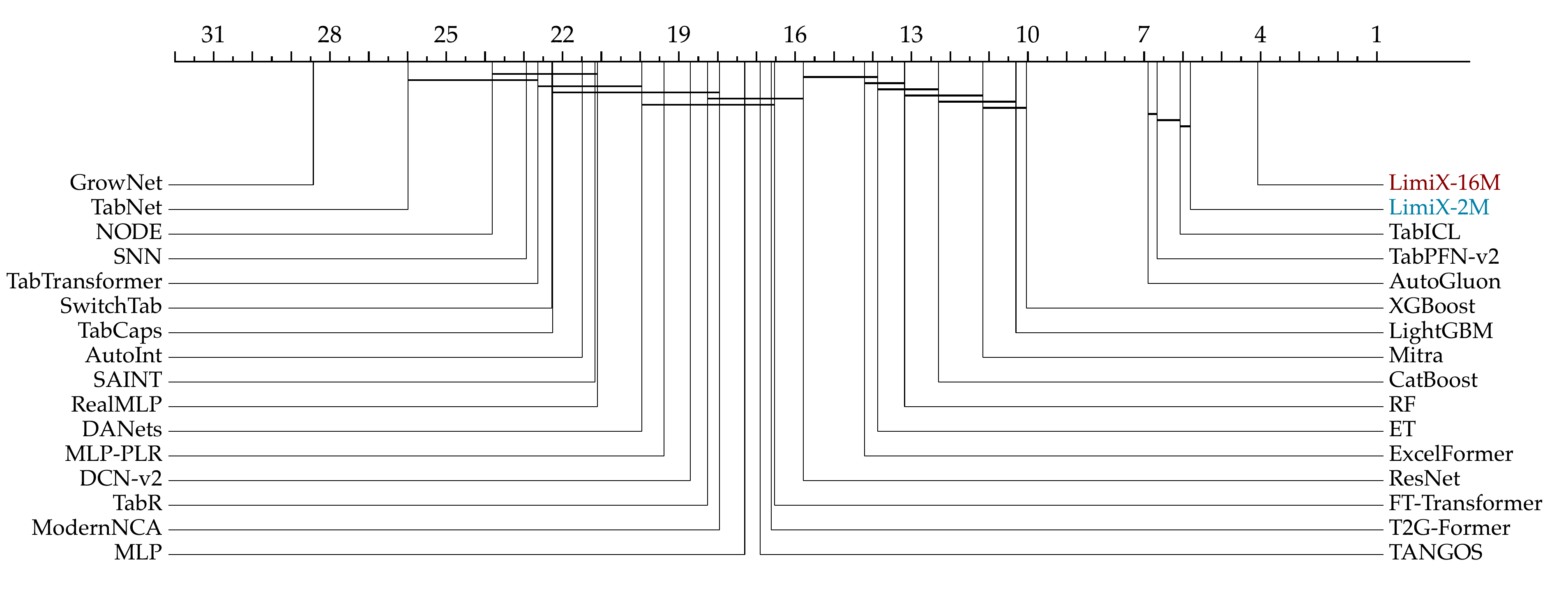}
        \caption{AUC on the TALENT-CLS benchmark.}
        \label{fig:cdd_talent-sub1}
    \end{subfigure}

    \vspace{2em}
    
    \begin{subfigure}[b]{\textwidth}
        \centering
        \includegraphics[width=\textwidth]{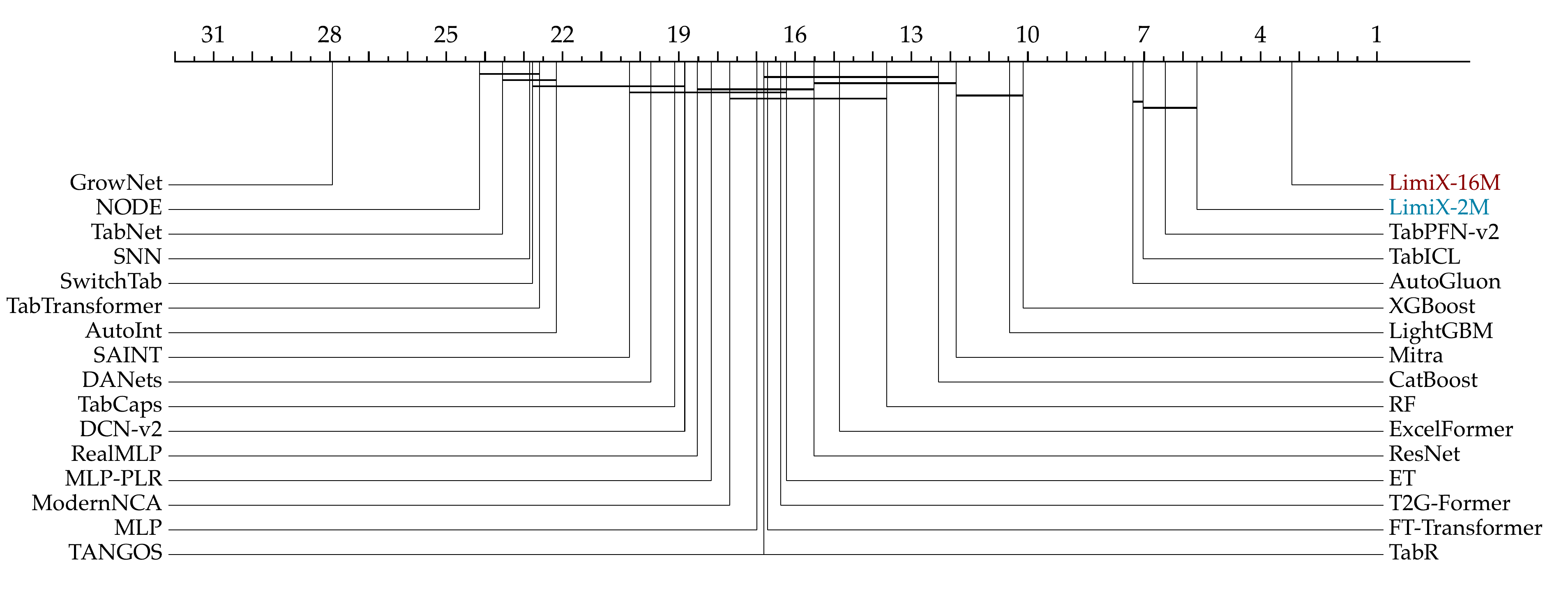}
        \caption{Accuracy on the TALENT-CLS benchmark.}
        \label{fig:cdd_talent-sub2}
    \end{subfigure}

    \vspace{2em}
 
    \begin{subfigure}[b]{\textwidth}
        \centering
        \includegraphics[width=\textwidth]{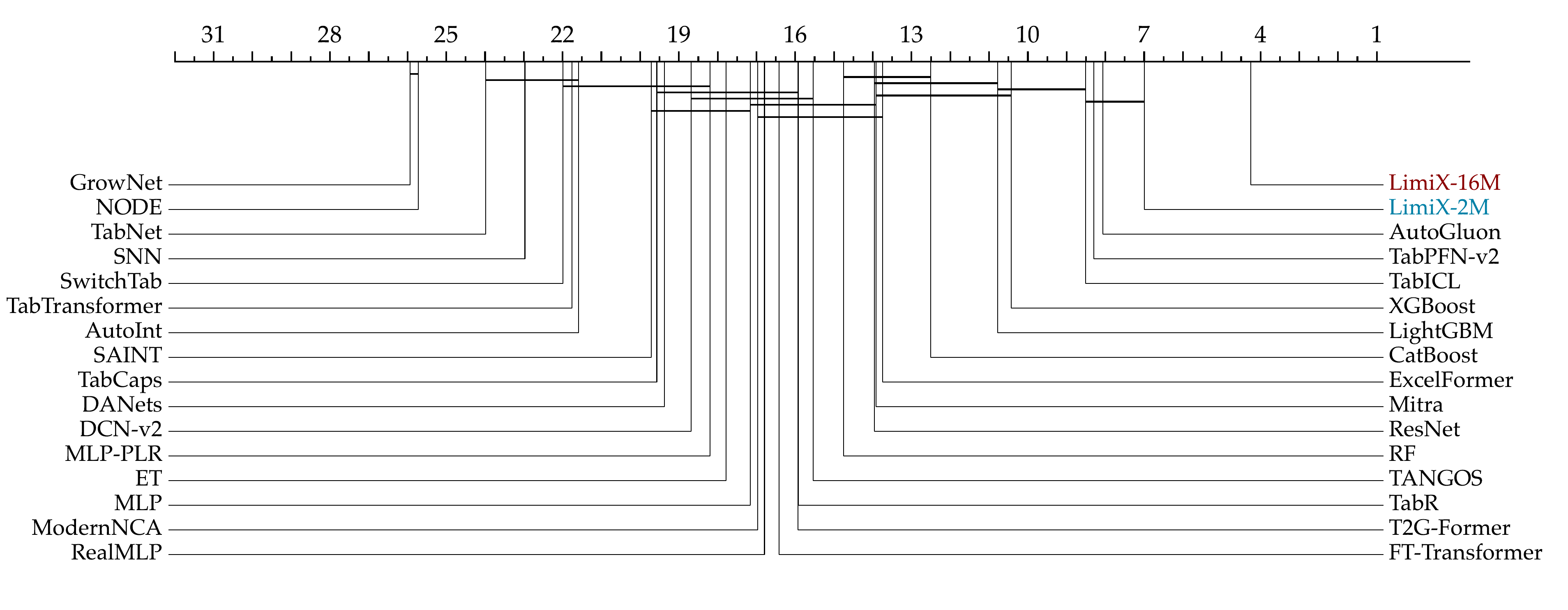}
        \caption{F1-score on the TALENT-CLS benchmark.}
        \label{fig:cdd_talent-sub3}
    \end{subfigure}
    
    \caption{Critical difference diagram on the TALENT-CLS benchmark}
    \label{fig:cdd_talent}
\end{figure}
\clearpage

\begin{table}[H]
\centering
\renewcommand{\arraystretch}{1.05}
\setlength{\tabcolsep}{12pt} 
\caption{Classification results on the OpenML-CC18 benchmark. The best scores are shown in bold.}
\label{tab:results_cls_CC18}
\begin{adjustbox}{max width=0.95\textwidth}
\begin{tabular}{l|cccccc}
\toprule
& \multicolumn{6}{c}{\rule{0pt}{20pt}\raisebox{5pt}{OpenML-cc18}} \\ \midrule
& \multicolumn{3}{c|}{Mean} & \multicolumn{3}{c}{Rank} \\
\cmidrule(lr){2-4}\cmidrule(lr){5-7}
\multicolumn{1}{l|}{Model} & AUC ($\uparrow$) & Acc. ($\uparrow$) & \multicolumn{1}{c|}{F1 ($\uparrow$)} & AUC ($\downarrow$) & Acc. ($\downarrow$) & F1 ($\downarrow$) \\
\midrule
\multicolumn{1}{l|}{LimiX-16M} & \textbf{0.939} & \textbf{0.893} & \multicolumn{1}{c|}{\textbf{0.811}} & \textbf{4.339} & \textbf{5.323} & \textbf{4.548} \\
\multicolumn{1}{l|}{LimiX-2M} & 0.936 & 0.891 & \multicolumn{1}{c|}{0.803} & 5.065 & 5.387 & 5.919 \\
\multicolumn{1}{l|}{AutoGluon} & 0.932 & 0.885 & \multicolumn{1}{c|}{0.790} & 6.500 & 7.113 & 7.613 \\
\multicolumn{1}{l|}{TabPFN-v2} & 0.929 & 0.886 & \multicolumn{1}{c|}{0.790} & 7.887 & 7.726 & 8.161 \\
\multicolumn{1}{l|}{TabICL} & 0.927 & 0.875 & \multicolumn{1}{c|}{0.782} & 8.839 & 10.339 & 10.306 \\
\multicolumn{1}{l|}{XGBoost} & 0.929 & 0.879 & \multicolumn{1}{c|}{0.775} & 9.661 & 10.839 & 10.790 \\
\multicolumn{1}{l|}{LightGBM} & 0.927 & 0.879 & \multicolumn{1}{c|}{0.775} & 10.065 & 10.226 & 10.323 \\
\multicolumn{1}{l|}{CatBoost} & 0.926 & 0.870 & \multicolumn{1}{c|}{0.770} & 12.177 & 13.516 & 13.419 \\
\multicolumn{1}{l|}{Mitra} & 0.920 & 0.866 & \multicolumn{1}{c|}{0.743} & 12.952 & 14.419 & 16.048 \\
\multicolumn{1}{l|}{ET} & 0.922 & 0.861 & \multicolumn{1}{c|}{0.721} & 13.774 & 15.887 & 17.774 \\
\multicolumn{1}{l|}{RF} & 0.925 & 0.871 & \multicolumn{1}{c|}{0.762} & 13.774 & 13.161 & 14.048 \\
\multicolumn{1}{l|}{ExcelFormer} & 0.918 & 0.870 & \multicolumn{1}{c|}{0.773} & 13.935 & 13.032 & 12.839 \\
\multicolumn{1}{l|}{TANGOS} & 0.910 & 0.863 & \multicolumn{1}{c|}{0.759} & 14.161 & 14.129 & 14.290 \\
\multicolumn{1}{l|}{ResNet} & 0.913 & 0.860 & \multicolumn{1}{c|}{0.764} & 14.226 & 14.371 & 13.581 \\
\multicolumn{1}{l|}{MLP} & 0.908 & 0.857 & \multicolumn{1}{c|}{0.743} & 14.710 & 14.871 & 15.871 \\
\multicolumn{1}{l|}{T2G-Former} & 0.908 & 0.859 & \multicolumn{1}{c|}{0.748} & 15.919 & 15.500 & 15.484 \\
\multicolumn{1}{l|}{MLP-PLR} & 0.896 & 0.858 & \multicolumn{1}{c|}{0.734} & 16.694 & 16.435 & 16.790 \\
\multicolumn{1}{l|}{RealMLP} & 0.889 & 0.858 & \multicolumn{1}{c|}{0.742} & 17.032 & 15.435 & 14.823 \\
\multicolumn{1}{l|}{FT-Transformer} & 0.904 & 0.856 & \multicolumn{1}{c|}{0.739} & 17.339 & 16.161 & 17.065 \\
\multicolumn{1}{l|}{ModernNCA} & 0.906 & 0.858 & \multicolumn{1}{c|}{0.747} & 17.452 & 16.839 & 16.371 \\
\multicolumn{1}{l|}{TabR} & 0.900 & 0.863 & \multicolumn{1}{c|}{0.757} & 17.468 & 14.210 & 13.790 \\
\multicolumn{1}{l|}{AutoInt} & 0.889 & 0.835 & \multicolumn{1}{c|}{0.717} & 18.984 & 19.032 & 19.129 \\
\multicolumn{1}{l|}{DANets} & 0.900 & 0.840 & \multicolumn{1}{c|}{0.715} & 19.000 & 18.435 & 19.435 \\
\multicolumn{1}{l|}{DCN-v2} & 0.899 & 0.851 & \multicolumn{1}{c|}{0.729} & 19.581 & 18.355 & 18.968 \\
\multicolumn{1}{l|}{SNN} & 0.896 & 0.835 & \multicolumn{1}{c|}{0.698} & 21.081 & 20.871 & 21.661 \\
\multicolumn{1}{l|}{SwitchTab} & 0.881 & 0.819 & \multicolumn{1}{c|}{0.671} & 22.355 & 21.758 & 22.419 \\
\multicolumn{1}{l|}{TabCaps} & 0.875 & 0.848 & \multicolumn{1}{c|}{0.692} & 22.548 & 19.484 & 21.097 \\
\multicolumn{1}{l|}{TabTransformer} & 0.829 & 0.776 & \multicolumn{1}{c|}{0.616} & 23.500 & 22.532 & 23.274 \\
\multicolumn{1}{l|}{NODE} & 0.889 & 0.809 & \multicolumn{1}{c|}{0.626} & 24.129 & 25.129 & 27.161 \\
\multicolumn{1}{l|}{SAINT} & 0.512 & 0.492 & \multicolumn{1}{c|}{0.441} & 26.129 & 25.532 & 24.210 \\
\multicolumn{1}{l|}{TabNet} & 0.863 & 0.816 & \multicolumn{1}{c|}{0.666} & 26.226 & 24.065 & 25.823 \\
\multicolumn{1}{l|}{GrowNet} & 0.816 & 0.748 & \multicolumn{1}{c|}{0.581} & 26.387 & 26.694 & 25.903 \\
\bottomrule
\end{tabular}
\end{adjustbox}
\end{table}

\begin{table}[H]
\centering
\caption{Statistical profile of the benchmark OpenML-CC18, where Q10, Q50, and Q90 correspond to the 10\%, 50\%, and 90\% quantiles, respectively; for categorical feature statistics, we only consider features that are either string-typed or have fewer than 10 unique values.}
\resizebox{0.9\textwidth}{!}{
    \begin{tabular}{cl ccccccc}
    \toprule
    & & \multicolumn{7}{c}{Statistics} \\
    \cmidrule(lr){3-9}
     & Metric   & Q10  & Q50 & Q90 & Mean & Std  & Min & Max \\
    \midrule
    & \# Features & 6 & 30 & 611 & 336 & 1344 & 4 & 10935 \\
    & \# Classes & 2 & 3 & 10 & 6 & 7 & 2 & 46 \\
    & Missing Values (Ratio) & 0 & 0 & 0.002 & 0.004 & 0.019 & 0 & 0.139 \\
    & Categorical Features (Ratio) & 0 & 0.091 & 1 & 0.327 & 0.406 & 0 & 1 \\
    & Features w/ Missing Values (Ratio) & 0 & 0 & 0.204 & 0.054 & 0.165 & 0 & 0.757 \\
    \bottomrule
    \end{tabular}
}
\label{tab:profile_cls_CC18}
\end{table}

\begin{figure}[H]
    \centering
    \begin{subfigure}[b]{\textwidth}
        \centering
        \includegraphics[width=\textwidth]{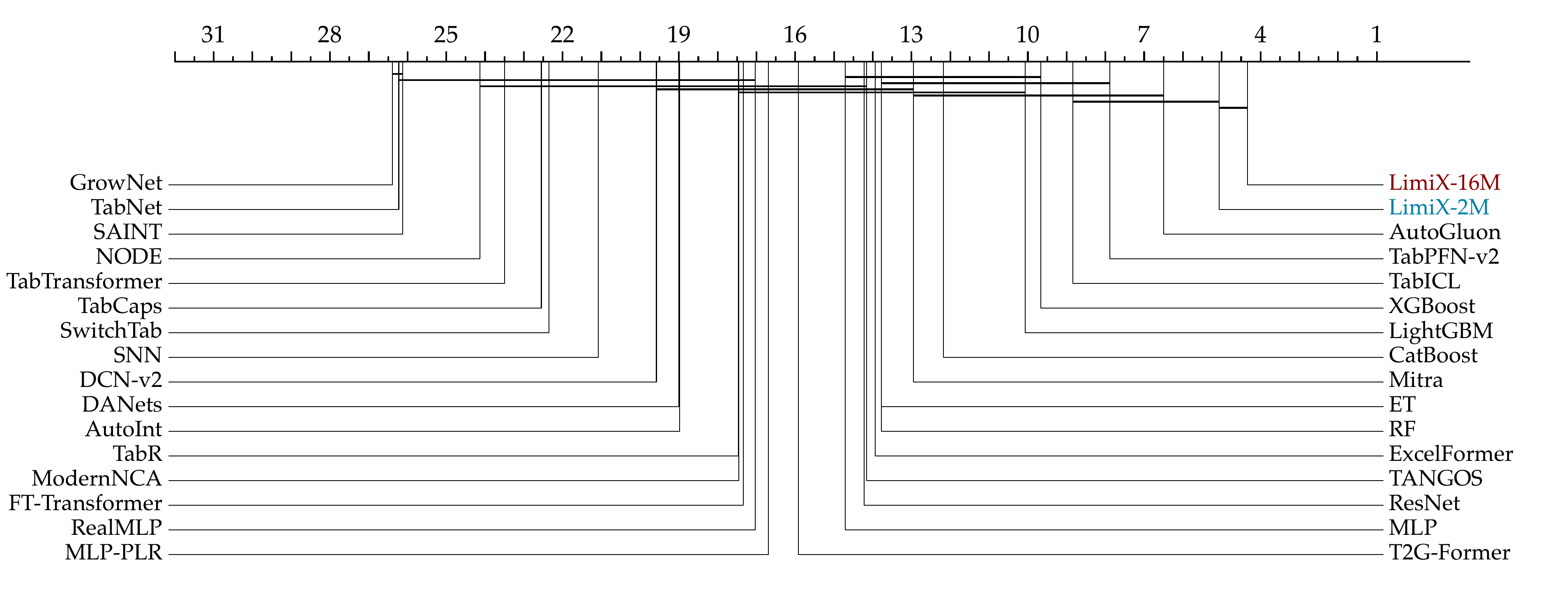}
        \caption{AUC on the OpenML-CC18 benchmark.}
        \label{fig:cc18-sub1}
    \end{subfigure}
    
    \vspace{2em}
    
    \begin{subfigure}[b]{\textwidth}
        \centering
        \includegraphics[width=\textwidth]{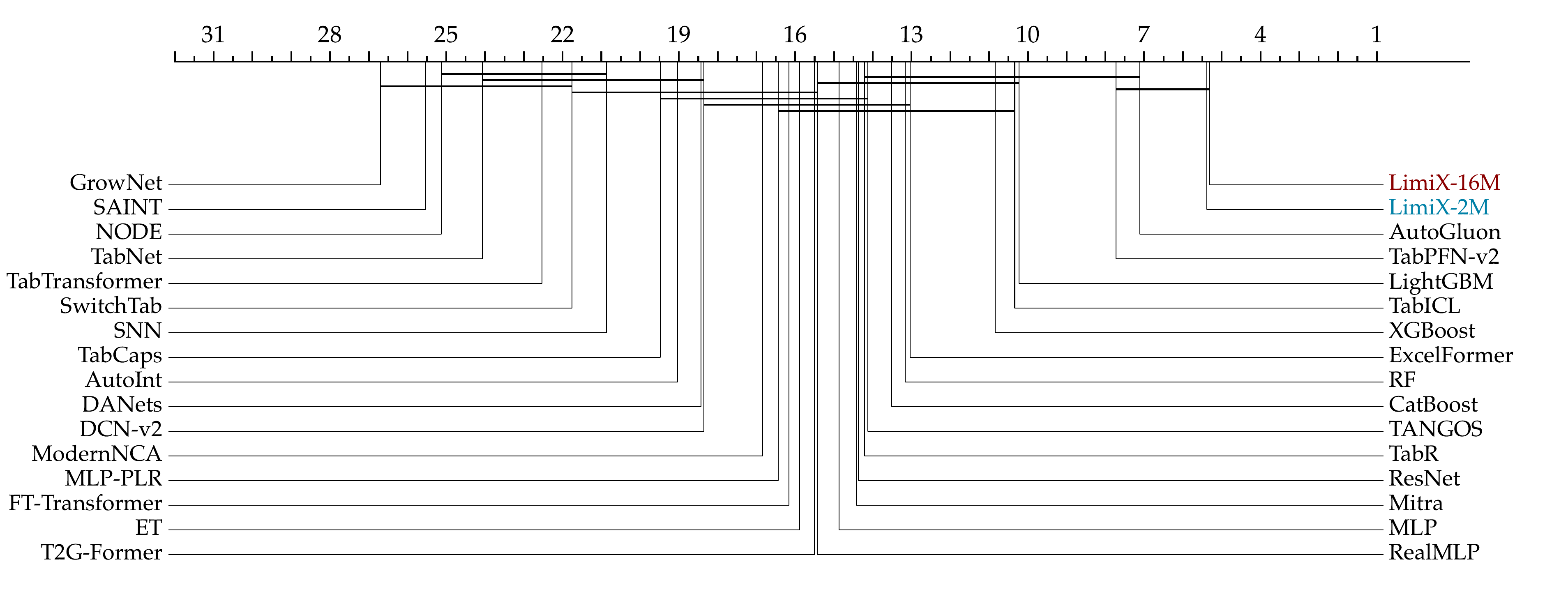}
        \caption{Accuracy on the OpenML-CC18 benchmark.}
        \label{fig:cc18-sub2}
    \end{subfigure}
    
    \vspace{2em}
    
    \begin{subfigure}[b]{\textwidth}
        \centering
        \includegraphics[width=\textwidth]{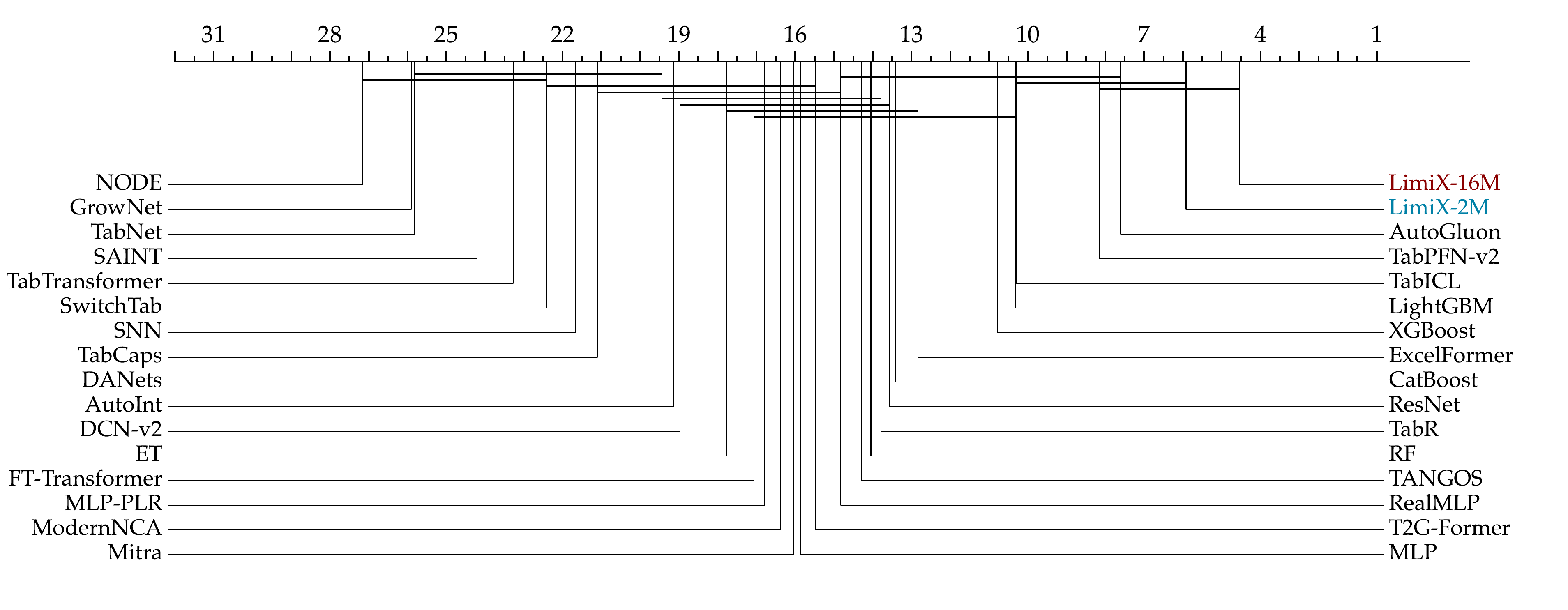}
        \caption{F1-score on the OpenML-CC18 benchmark.}
        \label{fig:cc18-sub3}
    \end{subfigure}
    
    \caption{Critical difference diagrams on the OpenML-CC18 benchmark.}
    \label{fig:cc18}
\end{figure}
\clearpage

\begin{table}[H]
\centering
\renewcommand{\arraystretch}{1.05}
\setlength{\tabcolsep}{12pt} 
\caption{Classification results on the TabArena-CLS benchmark. The best scores are shown in bold.}
\label{tab:results_cls_TabArena}
\begin{adjustbox}{max width=0.95\textwidth}
\begin{tabular}{l|cccccc}
\toprule
& \multicolumn{6}{c}{\rule{0pt}{20pt}\raisebox{5pt}{TabArena-CLS}} \\ \midrule
& \multicolumn{3}{c|}{Mean} & \multicolumn{3}{c}{Rank} \\
\cmidrule(lr){2-4}\cmidrule(lr){5-7}
\multicolumn{1}{l|}{Model} & AUC ($\uparrow$) & Acc. ($\uparrow$) & \multicolumn{1}{c|}{F1 ($\uparrow$)} & AUC ($\downarrow$) & Acc. ($\downarrow$) & F1 ($\downarrow$) \\
\midrule
\multicolumn{1}{l|}{LimiX-16M} & \textbf{0.849} & \textbf{0.877} & \multicolumn{1}{c|}{\textbf{0.597}} & \textbf{3.455} & 3.394 & 7.000 \\
\multicolumn{1}{l|}{LimiX-2M} & 0.846 & 0.876 & \multicolumn{1}{c|}{0.594} & 4.879 & \textbf{3.333} & \textbf{6.727} \\
\multicolumn{1}{l|}{AutoGluon} & 0.844 & 0.870 & \multicolumn{1}{c|}{0.574} & 5.818 & 5.364 & 9.091 \\
\multicolumn{1}{l|}{TabICL} & 0.840 & 0.870 & \multicolumn{1}{c|}{0.553} & 6.879 & 6.394 & 10.576 \\
\multicolumn{1}{l|}{TabPFN-v2} & 0.838 & 0.872 & \multicolumn{1}{c|}{0.589} & 7.273 & 4.636 & 8.848 \\
\multicolumn{1}{l|}{LightGBM} & 0.841 & 0.868 & \multicolumn{1}{c|}{0.574} & 7.455 & 8.333 & 10.545 \\
\multicolumn{1}{l|}{XGBoost} & 0.838 & 0.867 & \multicolumn{1}{c|}{0.567} & 8.394 & 7.818 & 10.909 \\
\multicolumn{1}{l|}{RF} & 0.837 & 0.864 & \multicolumn{1}{c|}{0.558} & 9.818 & 8.303 & 12.061 \\
\multicolumn{1}{l|}{CatBoost} & 0.835 & 0.867 & \multicolumn{1}{c|}{0.574} & 10.121 & 7.606 & 9.970 \\
\multicolumn{1}{l|}{ET} & 0.833 & 0.857 & \multicolumn{1}{c|}{0.505} & 10.939 & 11.212 & 16.333 \\
\multicolumn{1}{l|}{Mitra} & 0.815 & 0.862 & \multicolumn{1}{c|}{0.533} & 12.121 & 10.364 & 14.970 \\
\multicolumn{1}{l|}{ExcelFormer} & 0.810 & 0.849 & \multicolumn{1}{c|}{0.555} & 14.697 & 16.939 & 14.970 \\
\multicolumn{1}{l|}{MLP} & 0.772 & 0.822 & \multicolumn{1}{c|}{0.459} & 18.515 & 17.636 & 19.576 \\
\multicolumn{1}{l|}{TANGOS} & 0.791 & 0.844 & \multicolumn{1}{c|}{0.522} & 18.576 & 17.848 & 16.000 \\
\multicolumn{1}{l|}{T2G-Former} & 0.779 & 0.822 & \multicolumn{1}{c|}{0.482} & 18.667 & 20.061 & 15.848 \\
\multicolumn{1}{l|}{MLP-PLR} & 0.781 & 0.836 & \multicolumn{1}{c|}{0.460} & 19.030 & 19.091 & 19.667 \\
\multicolumn{1}{l|}{ModernNCA} & 0.783 & 0.846 & \multicolumn{1}{c|}{0.511} & 19.606 & 20.394 & 16.970 \\
\multicolumn{1}{l|}{AutoInt} & 0.769 & 0.826 & \multicolumn{1}{c|}{0.474} & 19.697 & 20.636 & 18.576 \\
\multicolumn{1}{l|}{TabR} & 0.785 & 0.842 & \multicolumn{1}{c|}{0.510} & 20.152 & 19.606 & 16.091 \\
\multicolumn{1}{l|}{NODE} & 0.769 & 0.792 & \multicolumn{1}{c|}{0.352} & 20.303 & 21.424 & 24.333 \\
\multicolumn{1}{l|}{ResNet} & 0.781 & 0.824 & \multicolumn{1}{c|}{0.532} & 20.364 & 21.485 & 16.303 \\
\multicolumn{1}{l|}{FT-Transformer} & 0.770 & 0.803 & \multicolumn{1}{c|}{0.468} & 20.515 & 21.636 & 18.606 \\
\multicolumn{1}{l|}{TabTransformer} & 0.739 & 0.781 & \multicolumn{1}{c|}{0.438} & 20.606 & 21.424 & 19.091 \\
\multicolumn{1}{l|}{RealMLP} & 0.777 & 0.822 & \multicolumn{1}{c|}{0.512} & 20.758 & 21.030 & 14.424 \\
\multicolumn{1}{l|}{DCN-v2} & 0.769 & 0.833 & \multicolumn{1}{c|}{0.482} & 21.303 & 19.424 & 18.212 \\
\multicolumn{1}{l|}{SNN} & 0.755 & 0.818 & \multicolumn{1}{c|}{0.442} & 22.242 & 22.061 & 22.000 \\
\multicolumn{1}{l|}{TabCaps} & 0.742 & 0.837 & \multicolumn{1}{c|}{0.471} & 22.333 & 17.970 & 19.667 \\
\multicolumn{1}{l|}{DANets} & 0.749 & 0.776 & \multicolumn{1}{c|}{0.453} & 23.152 & 23.848 & 18.333 \\
\multicolumn{1}{l|}{SwitchTab} & 0.754 & 0.799 & \multicolumn{1}{c|}{0.409} & 23.212 & 24.879 & 22.303 \\
\multicolumn{1}{l|}{SAINT} & 0.694 & 0.739 & \multicolumn{1}{c|}{0.437} & 24.545 & 25.061 & 21.515 \\
\multicolumn{1}{l|}{TabNet} & 0.709 & 0.789 & \multicolumn{1}{c|}{0.438} & 25.848 & 22.364 & 23.182 \\
\multicolumn{1}{l|}{GrowNet} & 0.646 & 0.674 & \multicolumn{1}{c|}{0.361} & 26.636 & 28.091 & 23.273 \\
\bottomrule
\end{tabular}
\end{adjustbox}
\end{table}

\vspace{1em}

\begin{table}[H]
\centering
\caption{Statistical profile of the benchmark TabArena, where Q10, Q50, and Q90 correspond to the 10\%, 50\%, and 90\% quantiles, respectively; for categorical feature statistics, we only consider features that are either string-typed or have fewer than 10 unique values.}
\resizebox{0.9\textwidth}{!}{
    \begin{tabular}{cl ccccccc}
    \toprule
    & & \multicolumn{7}{c}{Statistics} \\
    \cmidrule(lr){3-9}
     & Metric   & Q10  & Q50 & Q90 & Mean & Std  & Min & Max \\
    \midrule
    & \# Features & 9 & 21 & 129 & 125 & 374 & 4 & 1776 \\
    & \# Classes & 2 & 2 & 3 & 262 & 1577 & 2 & 9856 \\
    & Missing Values (Ratio) & 0 & 0 & 0.056 & 0.027 & 0.109 & 0 & 0.672 \\
    & Categorical Features (Ratio) & 0 & 0.47 & 1 & 0.45 & 0.366 & 0 & 1 \\
    & Features w/ Missing Values (Ratio) & 0 & 0 & 0.696 & 0.139 & 0.294 & 0 & 0.994 \\
    \bottomrule
    \end{tabular}
}
\label{tab:profile_cls_TabArena}
\end{table}

\begin{figure}[H]
    \centering
    \begin{subfigure}[b]{\textwidth}
        \centering
        \includegraphics[width=\textwidth]{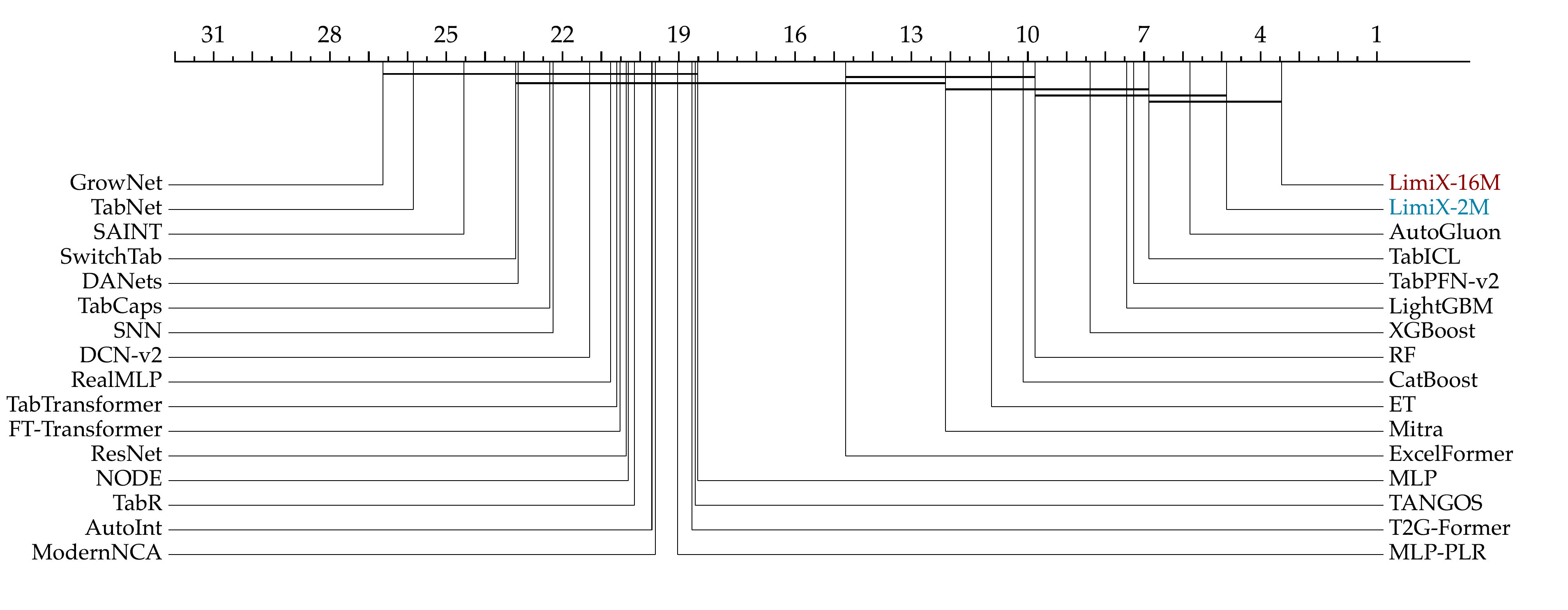}
        \caption{AUC on the TabArena-CLS benchmark.}
        \label{fig:tabarena-sub1}
    \end{subfigure}
    
    \vspace{2em}
    
    \begin{subfigure}[b]{\textwidth}
        \centering
        \includegraphics[width=\textwidth]{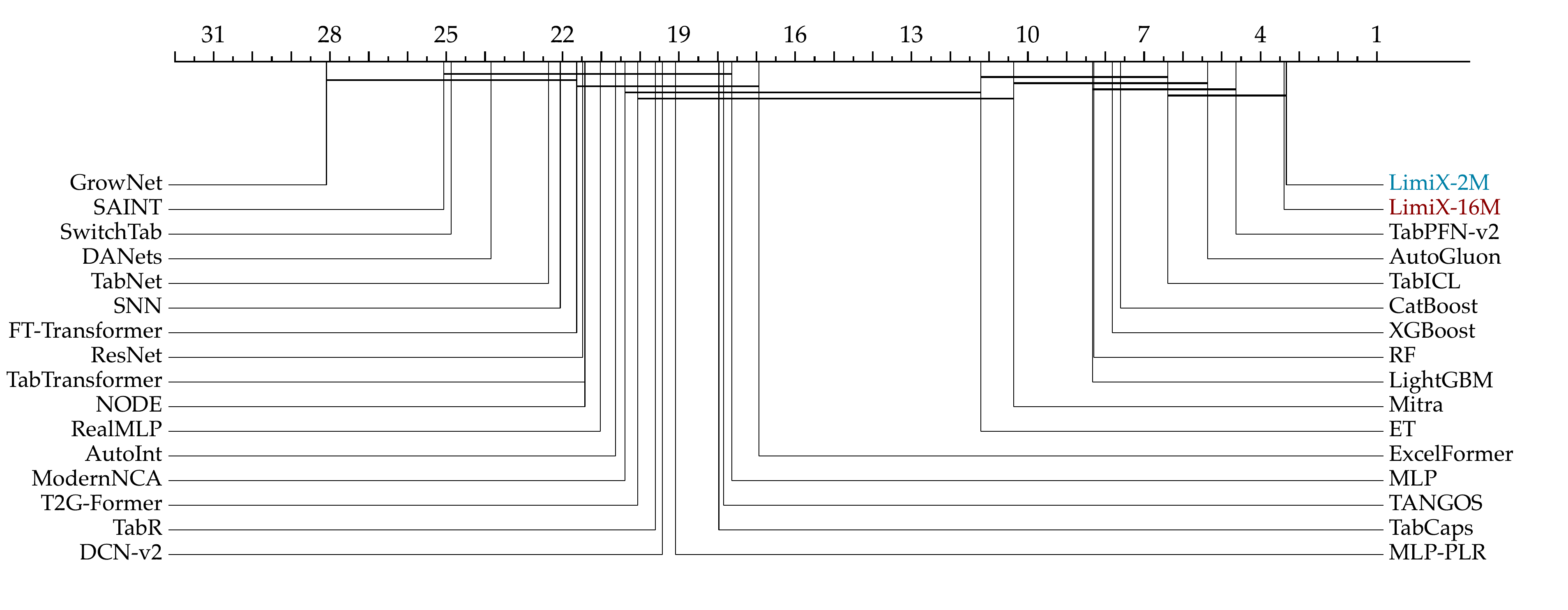}
        \caption{Accuracy on the TabArena-CLS benchmark.}
        \label{fig:tabarena-sub2}
    \end{subfigure}
    
    \vspace{2em}

    \begin{subfigure}[b]{\textwidth}
        \centering
        \includegraphics[width=\textwidth]{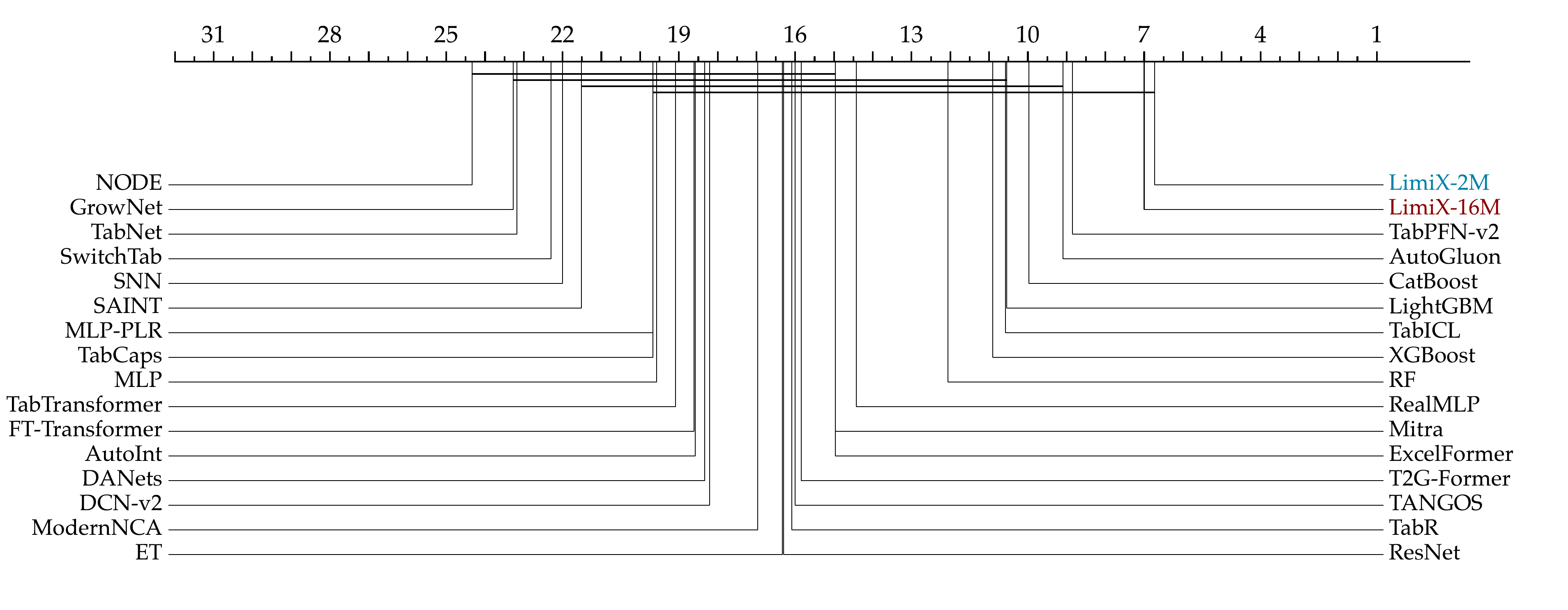}
        \caption{F1-score on the TabArena-CLS benchmark.}
        \label{fig:tabarena-sub3}
    \end{subfigure}
    
    \caption{Critical difference diagrams on the TabArena-CLS benchmark.}
    \label{fig:tabarena}
\end{figure}
\clearpage

\begin{table}[H]
\centering
\renewcommand{\arraystretch}{1.05}
\setlength{\tabcolsep}{12pt} 
\caption{Classification results on the PFN-CLS benchmark. The best scores are shown in bold.}
\label{tab:results_cls_TabPFN-v2 classifier}
\begin{adjustbox}{max width=0.95\textwidth}
\begin{tabular}{l|cccccc}
\toprule
& \multicolumn{6}{c}{\rule{0pt}{20pt}\raisebox{5pt}{PFN-CLS}} \\ \midrule
& \multicolumn{3}{c|}{Mean} & \multicolumn{3}{c}{Rank} \\
\cmidrule(lr){2-4}\cmidrule(lr){5-7}
\multicolumn{1}{l|}{Model} & AUC ($\uparrow$) & Acc. ($\uparrow$) & \multicolumn{1}{c|}{F1 ($\uparrow$)} & AUC ($\downarrow$) & Acc. ($\downarrow$) & F1 ($\downarrow$) \\
\midrule
\multicolumn{1}{l|}{LimiX-16M} & \textbf{0.923} & \textbf{0.862} & \multicolumn{1}{c|}{\textbf{0.786}} & \textbf{2.172} & \textbf{3.207} & \textbf{3.103} \\
\multicolumn{1}{l|}{LimiX-2M} & 0.913 & 0.848 & \multicolumn{1}{c|}{0.766} & 4.034 & 6.621 & 5.897 \\
\multicolumn{1}{l|}{TabPFN-v2} & 0.910 & 0.845 & \multicolumn{1}{c|}{0.756} & 5.586 & 7.172 & 7.138 \\
\multicolumn{1}{l|}{AutoGluon} & 0.906 & 0.835 & \multicolumn{1}{c|}{0.738} & 5.897 & 7.586 & 7.379 \\
\multicolumn{1}{l|}{TabICL} & 0.903 & 0.832 & \multicolumn{1}{c|}{0.742} & 6.862 & 8.552 & 9.034 \\
\multicolumn{1}{l|}{XGBoost} & 0.898 & 0.831 & \multicolumn{1}{c|}{0.733} & 9.103 & 7.483 & 8.448 \\
\multicolumn{1}{l|}{Mitra} & 0.897 & 0.826 & \multicolumn{1}{c|}{0.719} & 9.483 & 10.310 & 12.103 \\
\multicolumn{1}{l|}{LightGBM} & 0.893 & 0.826 & \multicolumn{1}{c|}{0.725} & 9.586 & 8.724 & 9.345 \\
\multicolumn{1}{l|}{CatBoost} & 0.895 & 0.819 & \multicolumn{1}{c|}{0.720} & 10.414 & 10.069 & 11.414 \\
\multicolumn{1}{l|}{RF} & 0.896 & 0.822 & \multicolumn{1}{c|}{0.721} & 11.241 & 10.897 & 12.103 \\
\multicolumn{1}{l|}{ET} & 0.893 & 0.809 & \multicolumn{1}{c|}{0.675} & 12.138 & 13.517 & 15.483 \\
\multicolumn{1}{l|}{ExcelFormer} & 0.883 & 0.812 & \multicolumn{1}{c|}{0.713} & 14.310 & 12.379 & 13.759 \\
\multicolumn{1}{l|}{ResNet} & 0.869 & 0.801 & \multicolumn{1}{c|}{0.700} & 15.241 & 15.379 & 14.828 \\
\multicolumn{1}{l|}{TANGOS} & 0.865 & 0.797 & \multicolumn{1}{c|}{0.698} & 15.897 & 16.759 & 15.207 \\
\multicolumn{1}{l|}{MLP} & 0.866 & 0.795 & \multicolumn{1}{c|}{0.695} & 17.276 & 16.517 & 16.138 \\
\multicolumn{1}{l|}{T2G-Former} & 0.848 & 0.792 & \multicolumn{1}{c|}{0.688} & 18.276 & 16.690 & 16.759 \\
\multicolumn{1}{l|}{ModernNCA} & 0.860 & 0.798 & \multicolumn{1}{c|}{0.692} & 18.276 & 16.586 & 15.586 \\
\multicolumn{1}{l|}{FT-Transformer} & 0.849 & 0.789 & \multicolumn{1}{c|}{0.683} & 19.000 & 18.759 & 17.448 \\
\multicolumn{1}{l|}{SwitchTab} & 0.858 & 0.776 & \multicolumn{1}{c|}{0.626} & 19.172 & 20.379 & 20.000 \\
\multicolumn{1}{l|}{MLP-PLR} & 0.848 & 0.786 & \multicolumn{1}{c|}{0.650} & 19.310 & 19.759 & 20.931 \\
\multicolumn{1}{l|}{DANets} & 0.844 & 0.770 & \multicolumn{1}{c|}{0.631} & 19.862 & 19.690 & 20.862 \\
\multicolumn{1}{l|}{TabCaps} & 0.834 & 0.788 & \multicolumn{1}{c|}{0.636} & 21.000 & 19.034 & 20.138 \\
\multicolumn{1}{l|}{TabR} & 0.842 & 0.789 & \multicolumn{1}{c|}{0.688} & 21.690 & 18.621 & 17.276 \\
\multicolumn{1}{l|}{RealMLP} & 0.829 & 0.780 & \multicolumn{1}{c|}{0.678} & 22.172 & 20.069 & 19.138 \\
\multicolumn{1}{l|}{TabTransformer} & 0.821 & 0.761 & \multicolumn{1}{c|}{0.604} & 22.379 & 21.517 & 23.310 \\
\multicolumn{1}{l|}{DCN-v2} & 0.846 & 0.771 & \multicolumn{1}{c|}{0.633} & 22.759 & 23.759 & 24.069 \\
\multicolumn{1}{l|}{NODE} & 0.844 & 0.754 & \multicolumn{1}{c|}{0.535} & 23.000 & 24.690 & 26.966 \\
\multicolumn{1}{l|}{AutoInt} & 0.838 & 0.772 & \multicolumn{1}{c|}{0.640} & 23.069 & 21.966 & 22.034 \\
\multicolumn{1}{l|}{SAINT} & 0.708 & 0.669 & \multicolumn{1}{c|}{0.563} & 23.207 & 22.138 & 21.207 \\
\multicolumn{1}{l|}{SNN} & 0.831 & 0.762 & \multicolumn{1}{c|}{0.595} & 24.345 & 21.931 & 23.966 \\
\multicolumn{1}{l|}{TabNet} & 0.825 & 0.768 & \multicolumn{1}{c|}{0.631} & 25.759 & 23.552 & 24.759 \\
\multicolumn{1}{l|}{GrowNet} & 0.756 & 0.689 & \multicolumn{1}{c|}{0.497} & 28.621 & 27.586 & 27.138 \\
\bottomrule
\end{tabular}
\end{adjustbox}
\end{table}

\begin{table}[H]
\centering
\caption{Statistical profile of the benchmark PFN-CLS, where Q10, Q50, and Q90 correspond to the 10\%, 50\%, and 90\% quantiles, respectively; for categorical feature statistics, we only consider features that are either string-typed or have fewer than 10 unique values.}
\resizebox{0.9\textwidth}{!}{
    \begin{tabular}{cl ccccccc}
    \toprule
    & & \multicolumn{7}{c}{Statistics} \\
    \cmidrule(lr){3-9}
     & Metric   & Q10  & Q50 & Q90 & Mean & Std  & Min & Max \\
    \midrule
    & \# Features & 6 & 21 & 187 & 58 & 80 & 4 & 308 \\
    & \# Classes & 2 & 2 & 7 & 4 & 2 & 2 & 10 \\
    & Missing Values (Ratio) & 0 & 0 & 0 & 0.001 & 0.006 & 0 & 0.032 \\
    & Categorical Features (Ratio) & 0 & 0.081 & 0.956 & 0.291 & 0.381 & 0 & 1 \\
    & Features w/ Missing Values (Ratio) & 0 & 0 & 0 & 0.016 & 0.086 & 0 & 0.474 \\
    \bottomrule
    \end{tabular}
}
\label{tab:profile_cls_TabPFN-v2 classifier}
\end{table}

\begin{figure}[H]
    \centering
    \begin{subfigure}[b]{\textwidth}
        \centering
        \includegraphics[width=\textwidth]{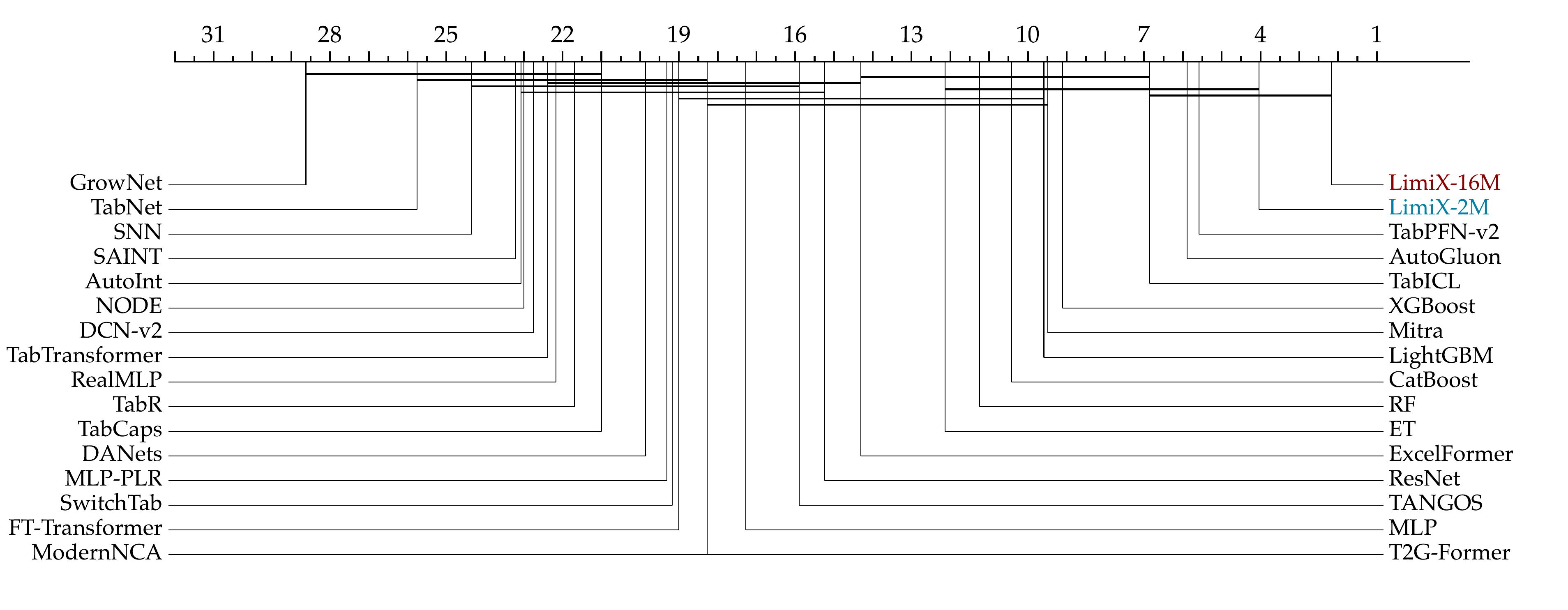}
        \caption{AUC on the PFN-CLS benchmark}
        \label{fig:cdd_pfn-sub1}
    \end{subfigure}

    \vspace{2em}
    
    \begin{subfigure}[b]{\textwidth}
        \centering
        \includegraphics[width=\textwidth]{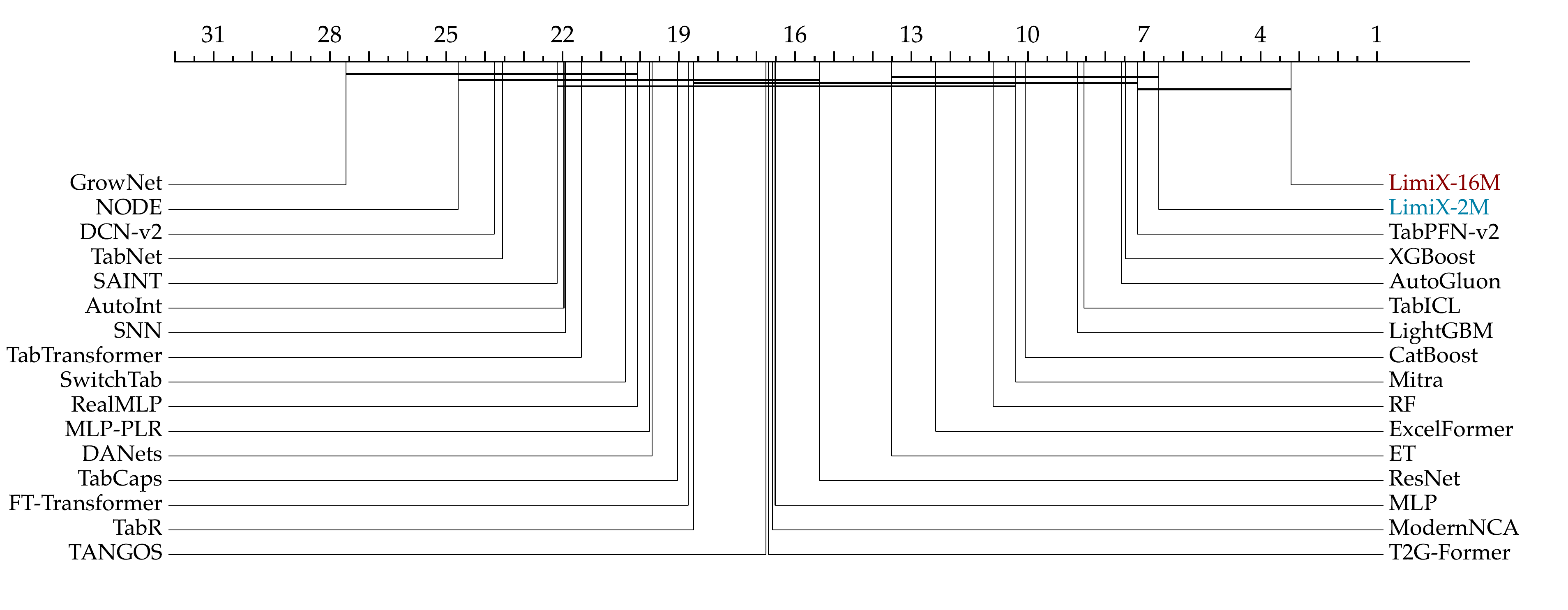}
        \caption{Accuracy on the PFN-CLS benchmark}
        \label{fig:cdd_pfn-sub2}
    \end{subfigure}

    \vspace{2em}
 
    \begin{subfigure}[b]{\textwidth}
        \centering
        \includegraphics[width=\textwidth]{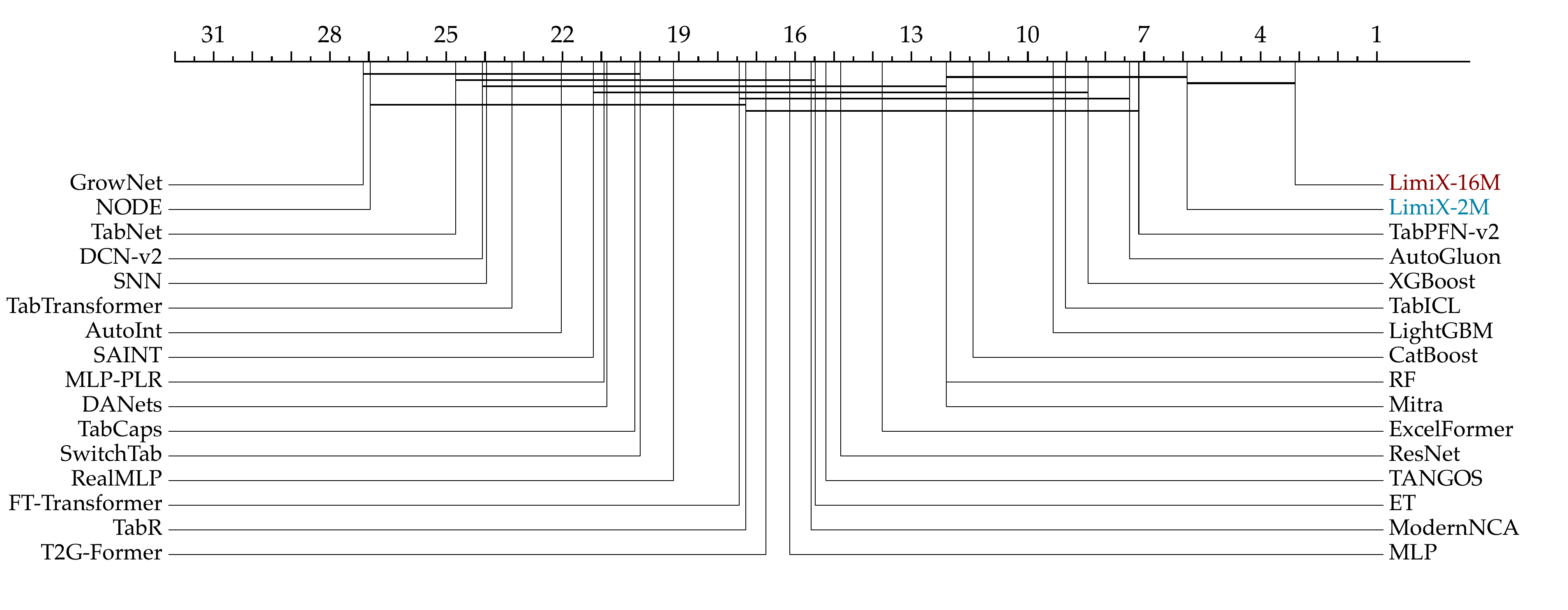}
        \caption{F1-score on the PFN-CLS benchmark}
        \label{fig:cdd_pfn-sub3}
    \end{subfigure}
    
    \caption{Critical difference diagrams on PFN-CLS benchmark}
    \label{fig:cdd_pfn}
\end{figure}
\clearpage
\begin{table}[H]
\centering
\renewcommand{\arraystretch}{1.05}
\setlength{\tabcolsep}{12pt} 
\caption{Classification results on the TabZilla benchmark. The best scores are shown in bold.}
\label{tab:results_tabzilla}
\begin{adjustbox}{max width=0.95\textwidth}
\begin{tabular}{l|cccccc}
\toprule
& \multicolumn{6}{c}{\rule{0pt}{20pt}\raisebox{5pt}{TabZilla}} \\ \midrule
& \multicolumn{3}{c|}{Mean} & \multicolumn{3}{c}{Rank} \\
\cmidrule(lr){2-4}\cmidrule(lr){5-7}
\multicolumn{1}{l|}{Model} & AUC ($\uparrow$) & Acc. ($\uparrow$) & \multicolumn{1}{c|}{F1 ($\uparrow$)} & AUC ($\downarrow$) & Acc. ($\downarrow$) & F1 ($\downarrow$) \\
\midrule
\multicolumn{1}{l|}{LimiX-16M} & \textbf{0.943} & \textbf{0.885} & \multicolumn{1}{c|}{\textbf{0.836}} & \textbf{4.815} & \textbf{4.963} & \textbf{6.185} \\
\multicolumn{1}{l|}{LimiX-2M} & 0.938 & 0.883 & \multicolumn{1}{c|}{0.832} & 5.667 & 6.222 & 6.630 \\
\multicolumn{1}{l|}{AutoGluon} & 0.933 & 0.871 & \multicolumn{1}{c|}{0.803} & 7.370 & 7.556 & 9.148 \\
\multicolumn{1}{l|}{TabPFN-v2} & 0.929 & 0.863 & \multicolumn{1}{c|}{0.797} & 8.000 & 9.037 & 9.481 \\
\multicolumn{1}{l|}{TabICL} & 0.933 & 0.864 & \multicolumn{1}{c|}{0.803} & 9.111 & 9.815 & 10.815 \\
\multicolumn{1}{l|}{XGBoost} & 0.929 & 0.863 & \multicolumn{1}{c|}{0.789} & 9.407 & 10.667 & 11.926 \\
\multicolumn{1}{l|}{LightGBM} & 0.927 & 0.863 & \multicolumn{1}{c|}{0.796} & 11.185 & 10.296 & 10.852 \\
\multicolumn{1}{l|}{RF} & 0.924 & 0.852 & \multicolumn{1}{c|}{0.773} & 11.704 & 12.593 & 13.185 \\
\multicolumn{1}{l|}{CatBoost} & 0.922 & 0.848 & \multicolumn{1}{c|}{0.780} & 12.926 & 13.926 & 13.926 \\
\multicolumn{1}{l|}{ExcelFormer} & 0.915 & 0.861 & \multicolumn{1}{c|}{0.802} & 14.667 & 12.259 & 12.296 \\
\multicolumn{1}{l|}{Mitra} & 0.915 & 0.841 & \multicolumn{1}{c|}{0.758} & 14.926 & 14.815 & 15.778 \\
\multicolumn{1}{l|}{T2G-Former} & 0.909 & 0.852 & \multicolumn{1}{c|}{0.790} & 14.926 & 14.852 & 15.148 \\
\multicolumn{1}{l|}{ModernNCA} & 0.907 & 0.850 & \multicolumn{1}{c|}{0.794} & 14.963 & 15.333 & 14.630 \\
\multicolumn{1}{l|}{ET} & 0.912 & 0.837 & \multicolumn{1}{c|}{0.745} & 15.333 & 15.741 & 17.037 \\
\multicolumn{1}{l|}{TabR} & 0.904 & 0.853 & \multicolumn{1}{c|}{0.793} & 15.741 & 13.963 & 14.370 \\
\multicolumn{1}{l|}{MLP-PLR} & 0.906 & 0.847 & \multicolumn{1}{c|}{0.773} & 15.741 & 15.519 & 16.704 \\
\multicolumn{1}{l|}{TANGOS} & 0.909 & 0.841 & \multicolumn{1}{c|}{0.776} & 15.926 & 14.444 & 14.963 \\
\multicolumn{1}{l|}{FT-Transformer} & 0.903 & 0.842 & \multicolumn{1}{c|}{0.769} & 16.741 & 16.444 & 16.074 \\
\multicolumn{1}{l|}{MLP} & 0.903 & 0.825 & \multicolumn{1}{c|}{0.747} & 17.000 & 17.333 & 17.852 \\
\multicolumn{1}{l|}{ResNet} & 0.908 & 0.834 & \multicolumn{1}{c|}{0.769} & 17.333 & 16.000 & 15.556 \\
\multicolumn{1}{l|}{AutoInt} & 0.896 & 0.833 & \multicolumn{1}{c|}{0.748} & 17.593 & 16.852 & 17.370 \\
\multicolumn{1}{l|}{DCN-v2} & 0.904 & 0.844 & \multicolumn{1}{c|}{0.781} & 18.333 & 17.370 & 17.889 \\
\multicolumn{1}{l|}{SAINT} & 0.824 & 0.764 & \multicolumn{1}{c|}{0.680} & 20.074 & 19.667 & 19.370 \\
\multicolumn{1}{l|}{RealMLP} & 0.892 & 0.846 & \multicolumn{1}{c|}{0.781} & 20.259 & 17.259 & 16.556 \\
\multicolumn{1}{l|}{SNN} & 0.874 & 0.816 & \multicolumn{1}{c|}{0.706} & 21.333 & 19.481 & 20.667 \\
\multicolumn{1}{l|}{DANets} & 0.881 & 0.800 & \multicolumn{1}{c|}{0.712} & 21.815 & 20.852 & 21.444 \\
\multicolumn{1}{l|}{TabTransformer} & 0.814 & 0.759 & \multicolumn{1}{c|}{0.659} & 22.222 & 20.630 & 21.593 \\
\multicolumn{1}{l|}{SwitchTab} & 0.860 & 0.764 & \multicolumn{1}{c|}{0.660} & 22.815 & 24.444 & 24.111 \\
\multicolumn{1}{l|}{TabCaps} & 0.887 & 0.816 & \multicolumn{1}{c|}{0.729} & 23.926 & 21.222 & 21.556 \\
\multicolumn{1}{l|}{NODE} & 0.869 & 0.784 & \multicolumn{1}{c|}{0.633} & 24.444 & 24.481 & 26.519 \\
\multicolumn{1}{l|}{GrowNet} & 0.829 & 0.732 & \multicolumn{1}{c|}{0.651} & 26.407 & 26.000 & 25.037 \\
\multicolumn{1}{l|}{TabNet} & 0.860 & 0.771 & \multicolumn{1}{c|}{0.668} & 27.370 & 26.778 & 26.407 \\
\bottomrule
\end{tabular}
\end{adjustbox}
\end{table}

\begin{table}[H]
\centering
\caption{Statistical profile of the benchmark TabZilla, where Q10, Q50, and Q90 correspond to the 10\%, 50\%, and 90\% quantiles, respectively; for categorical feature statistics, we only consider features that are either string-typed or have fewer than 10 unique values.}
\resizebox{0.9\textwidth}{!}{
    \begin{tabular}{cl ccccccc}
    \toprule
    & & \multicolumn{7}{c}{Statistics} \\
    \cmidrule(lr){3-9}
     & Metric   & Q10  & Q50 & Q90 & Mean & Std  & Min & Max \\
    \midrule
    & \# Features & 6 & 24 & 132 & 226 & 756 & 4 & 4296 \\
    & \# Classes & 2 & 2 & 10 & 7 & 16 & 2 & 100 \\
    & Missing Values (Ratio) & 0 & 0 & 0.078 & 0.022 & 0.059 & 0 & 0.205 \\
    & Categorical Features (Ratio) & 0 & 0.528 & 1 & 0.474 & 0.39 & 0 & 1 \\
    & Features w/ Missing Values (Ratio) & 0 & 0 & 0.483 & 0.102 & 0.218 & 0 & 0.808 \\
    \bottomrule
    \end{tabular}
}
\label{tab:profile_cls_tabzilla}
\end{table}

\begin{figure}[H]
    \centering
    \begin{subfigure}[b]{\textwidth}
        \centering
        \includegraphics[width=\textwidth]{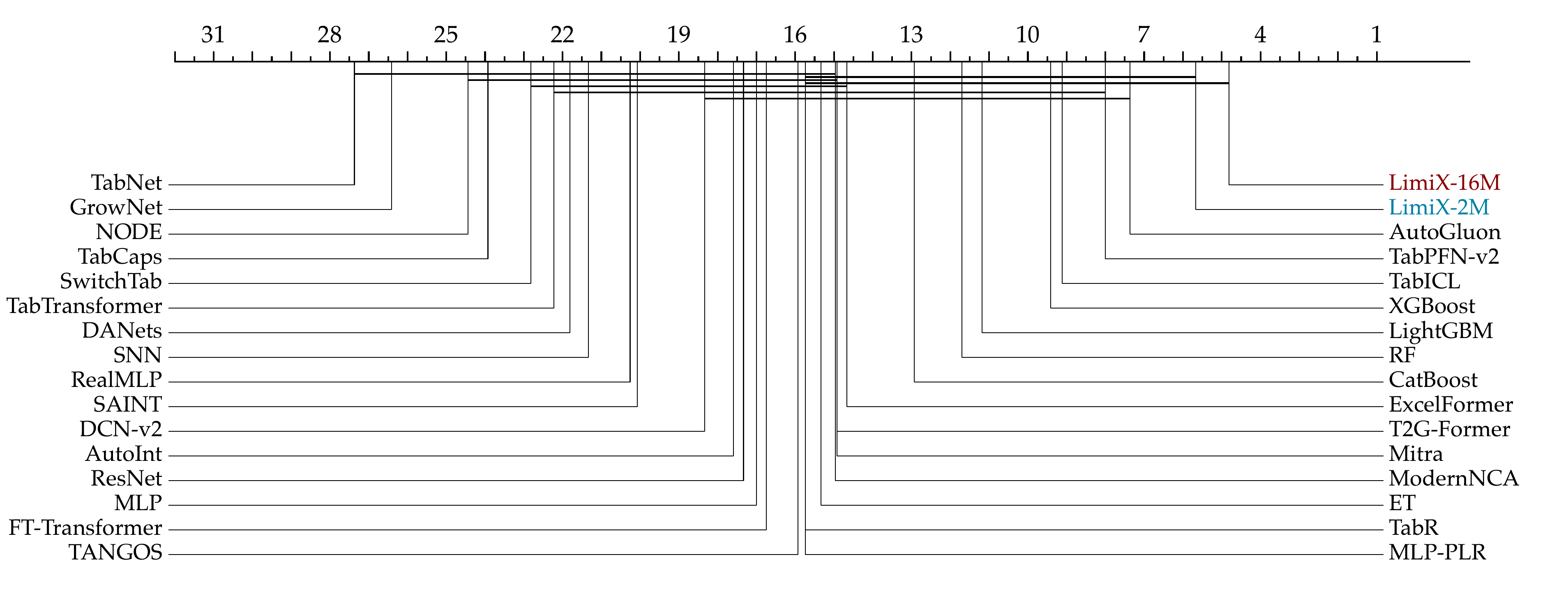}
        \caption{AUC on the TabZilla benchmark.}
        \label{fig:cddtabzilla-sub1}
    \end{subfigure}
    
    \vspace{2em}
    
    \begin{subfigure}[b]{\textwidth}
        \centering
        \includegraphics[width=\textwidth]{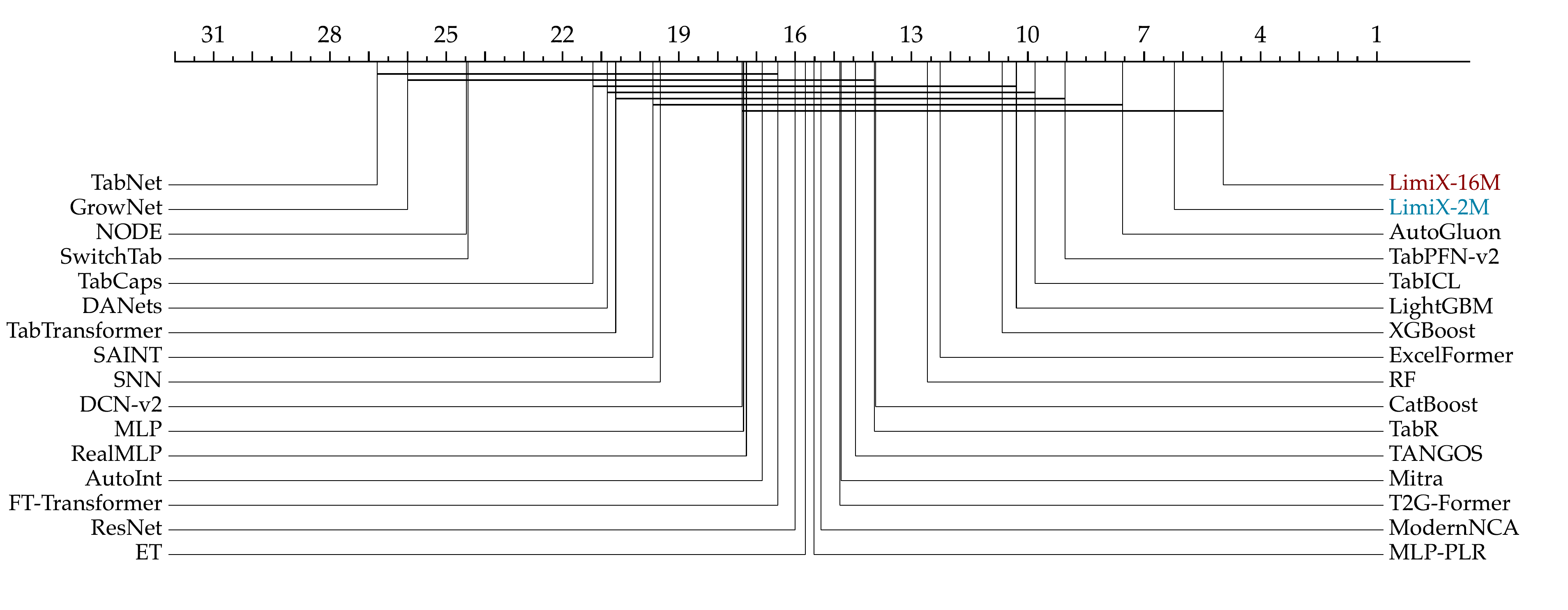}
        \caption{Accuracy on the TabZilla benchmark.}
        \label{fig:cddtabzilla-sub2}
    \end{subfigure}
    
    \vspace{2em}

    \begin{subfigure}[b]{\textwidth}
        \centering
        \includegraphics[width=\textwidth]{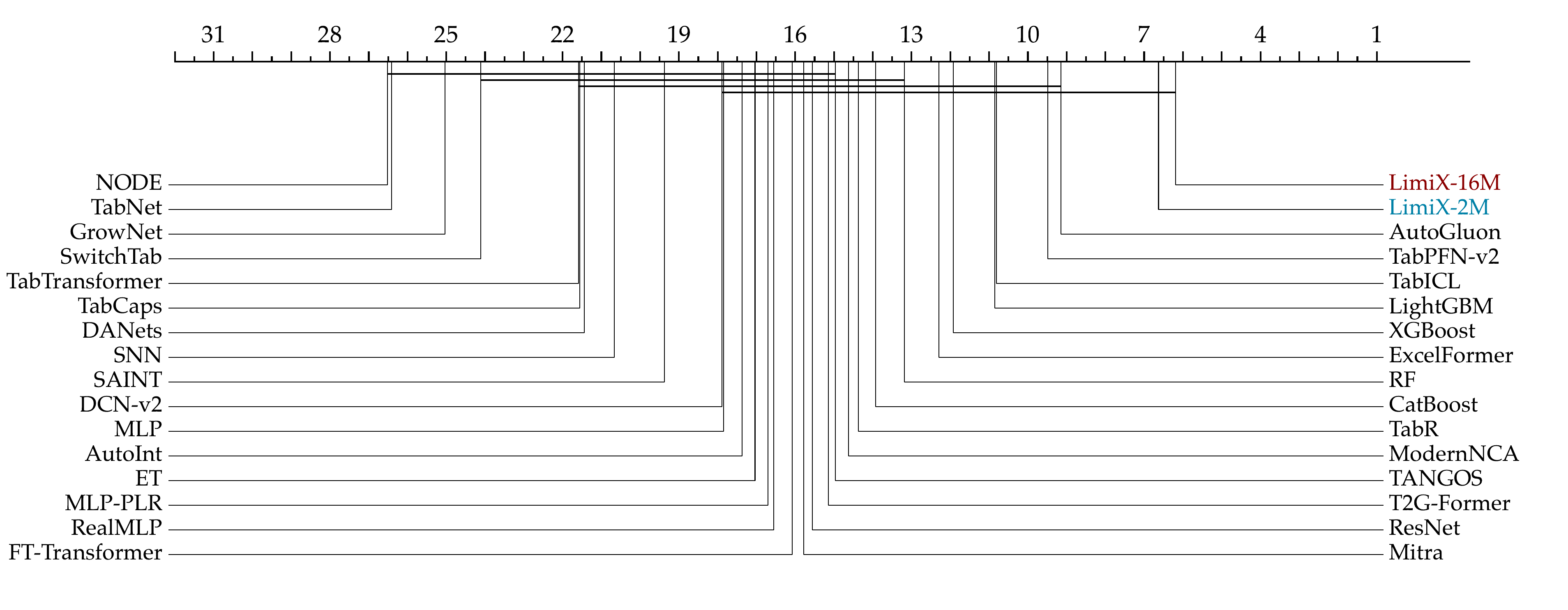}
        \caption{F1-score on the TabZilla benchmark.}
        \label{fig:cddtabzilla-sub3}
    \end{subfigure}
    
    \caption{Critical difference diagrams on the TabZilla benchmark.}
    \label{fig:cddtabzilla}
\end{figure}

\clearpage
\subsection{Regression}

\paragraph{Benchmark.} For the quantitative evaluation of regression performance, we leverage four benchmarks, including three open-source benchmark, TALENT-REG~\citep{liu2024talenttabularanalyticslearning}, PFN-REG~\citep{hollmann2025tabpfn_nature},TabArena-REG~\citep{erickson2025tabarena}, and CTR23~\citep{fischer2023openml_ctr23}.  After applying filters to exclude datasets containing more than 50,000 training samples or 10,000 features, we have 33 datasets from CTR23, 28 from PFN-REG, 13 from TabArena-REG, and 99 from TALENT-REG.

Similar to BCCO-CLS, we introduce BCCO-REG, a balanced regression benchmark comprising 50 datasets. For each dataset, we adopt the provided train-test split when available. If no predefined test set is available, we partition the data into a 70\% training set and a 30\% testing set randomly. Similar to the protocol employed in classification tasks, regression datasets with more than 50,000 training samples or 10,000 features are excluded.

\paragraph{Baselines.} For comparison, we include tree-based methods, NN-based methods, AutoML frameworks, and ICL-based models, which are basically consistent with baselines employed in classification tasks described above. TabICL and TabCaps are excluded because it cannot handle regression tasks. 
We add another baseline, DNNR~\citep{nader2022dnnr}, which is specifically designed for regression.
The model training procedure remains aligned with that employed in the classification experiments.

\paragraph{Metrics.} To evaluate regression performance properly, we employ normalized RMSE and R\textsuperscript{2} as the two evaluation metrics. We also analyze the ranks of models with respect to both metrics.

 \begin{figure}[H]
    \centering
    \includegraphics[width=0.9\linewidth]{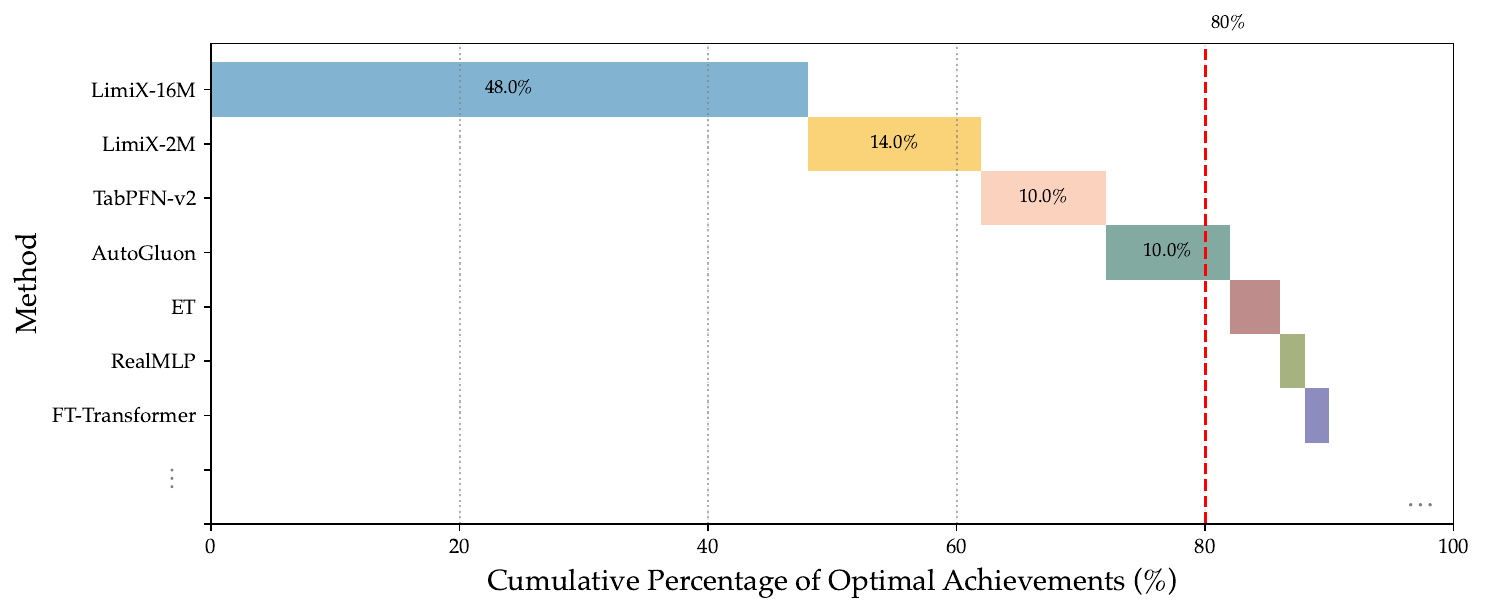}
    \caption{The proportion of models achieving the best R\textsuperscript{2}. The length of each bar represents the proportion of 106 datasets in which a given method achieved the highest R\textsuperscript{2} on the BCCO-REG benchmark. }
    \label{fig:regression_wf}
\end{figure}

\paragraph{Results.} 
As shown by critical difference diagrams, i.e. \Cref{fig:cdd-bcco-reg,fig:cdd-talent-reg,fig:cdd-ctr23-reg,fig:cdd-PFN-reg,fig:cdd_tabarenareg}, \ourL~always achieves the best performance, outperforming strong baselines like AutoGluon, XGBoost, and TabPFN-v2. 
Quantitative results in \Cref{tab:results_bcco_reg,tab:results_reg_pfn,tab:results_talent_reg,tab:results_ctr23,tab:results_reg_tabarena} also confirms that \ourL~achieves state-of-the-art regression performance, leading in both normalized RMSE and R\textsuperscript{2}. On the regression benchmarks, \ourS~exhibits predictive performance comparable to AutoGluon and TabPFN-v2, while requiring substantially less computational resources and runtime.

\paragraph{Subgroup analysis.} 
Similar to the subgroup analysis in classification, for regression, we also use the sample subgroups when building BCCO-REG to perform stratified analyses, whose criteria include the number of training samples, the ratio between the number of samples and features (length-to-width ratio), and the proportion of categorical features. 
\Cref{fig:Subgroup analysis regression} shows that \ourL~outperforms other methods across all subgroups compared with other methods. 
For \Cref{fig:barplot_reg_group-sub2,fig:barplot_reg_group-sub3,fig:barplot_reg_group-sub4}, we can see that performance of Mitra substantially drops for the second subgroup, while \ourL~do not exhibit a large drop.

\begin{table}[H]
\caption{Regression results on the BCCO-REG benchmark. The best scores are shown in bold.}
\label{tab:results_bcco_reg}
\centering
\setlength{\tabcolsep}{14pt} 
\begin{tabular}{l|cc|cc}
\toprule
& \multicolumn{4}{c}{\rule{0pt}{18pt}\raisebox{4pt}{BCCO-REG}} \\ \midrule
& \multicolumn{2}{c|}{Mean} & \multicolumn{2}{c}{Rank} \\
\cmidrule(lr){2-3}\cmidrule(lr){4-5}
\multicolumn{1}{l|}{Model} & R\textsuperscript{2} ($\uparrow$) & RMSE ($\downarrow$) & R\textsuperscript{2} ($\downarrow$) & RMSE ($\downarrow$) \\
\midrule
LimiX-16M            & \textbf{0.794} & \textbf{0.386} & \textbf{3.340} & \textbf{4.120} \\
AutoGluon            & 0.781 & 0.398 & 4.260 & 4.220 \\
TabPFN-v2            & 0.772 & 0.404 & 5.480 & 5.380 \\
LimiX-2M             & 0.785 & 0.392 & 5.720 & 5.600 \\
XGBoost              & 0.764 & 0.415 & 8.520 & 8.560 \\
LightGBM             & 0.715 & 0.423 & 8.760 & 8.720 \\
ET                   & 0.757 & 0.431 & 10.720 & 10.660 \\
CatBoost             & 0.741 & 0.427 & 11.320 & 11.160 \\
RF                   & 0.752 & 0.438 & 12.080 & 12.040 \\
ExcelFormer          & 0.743 & 0.443 & 12.260 & 12.320 \\
RealMLP              & 0.725 & 0.441 & 13.240 & 13.200 \\
T2G-Former           & 0.743 & 0.442 & 13.680 & 13.640 \\
FT-Transformer       & 0.737 & 0.448 & 14.240 & 14.220 \\
DCN-v2               & 0.739 & 0.448 & 14.600 & 14.600 \\
TabR                 & 0.733 & 0.448 & 15.300 & 15.360 \\
ResNet               & 0.720 & 0.468 & 15.940 & 15.900 \\
ModernNCA            & 0.598 & 0.471 & 16.060 & 15.960 \\
TANGOS               & 0.719 & 0.468 & 16.620 & 16.620 \\
Mitra                & 0.667 & 0.474 & 16.620 & 16.540 \\
MLP-PLR              & 0.734 & 0.453 & 16.660 & 16.660 \\
SAINT                & 0.701 & 0.481 & 17.440 & 17.460 \\
MLP                  & 0.701 & 0.487 & 17.860 & 17.920 \\
AutoInt              & 0.724 & 0.465 & 18.120 & 18.040 \\
NODE                 & 0.643 & 0.543 & 21.000 & 20.920 \\
TabNet               & 0.670 & 0.516 & 21.600 & 21.540 \\
DNNR                 & -2.152 & 1.329 & 21.780 & 21.800 \\
SNN                  & 0.434 & 0.720 & 26.200 & 26.200 \\
GrowNet              & 0.201 & 0.864 & 27.840 & 27.820 \\
DANets               & 0.005 & 0.979 & 29.440 & 29.480 \\
TabTransformer       & -0.000 & 0.981 & 29.600 & 29.640 \\
SwitchTab            & 0.001 & 0.981 & 29.700 & 29.700 \\
\bottomrule
\end{tabular}
\end{table}

\begin{table}[H]
\centering
\caption{Statistical profile of the benchmark BCCO-REG, where Q10, Q50, and Q90 correspond to the 10\%, 50\%, and 90\% quantiles, respectively; for categorical feature statistics, we only consider features that are either string-typed or have fewer than 10 unique values.}
\resizebox{0.9\textwidth}{!}{
    \begin{tabular}{cl ccccccc}
    \toprule
    & & \multicolumn{7}{c}{Statistics} \\
    \cmidrule(lr){3-9}
     & Metric   & Q10  & Q50 & Q90 & Mean & Std  & Min & Max \\
    \midrule
    & \# Features & 6 & 12 & 33 & 16 & 14 & 4 & 81 \\
    & Missing Values (Ratio) & 0 & 0 & 0 & 0 & 0 & 0 & 0.001 \\
    & Categorical Features (Ratio) & 0 & 0.16 & 0.65 & 0.242 & 0.264 & 0 & 0.967 \\
    & Features w/ Missing Values (Ratio) & 0 & 0 & 0 & 0.008 & 0.051 & 0 & 0.333 \\
    \bottomrule
    \end{tabular}
}
\label{tab:profile_reg_BCCO}
\end{table}

\clearpage
\begin{table}[H]
\caption{Regression results on the TALENT-REG benchmark. The best scores are shown in bold.}
\label{tab:results_talent_reg}
\centering
\setlength{\tabcolsep}{14pt} 
\begin{tabular}{l|cc|cc}
\toprule
& \multicolumn{4}{c}{\rule{0pt}{18pt}\raisebox{4pt}{TALENT-REG}} \\ \midrule
& \multicolumn{2}{c|}{Mean} & \multicolumn{2}{c}{Rank} \\
\cmidrule(lr){2-3}\cmidrule(lr){4-5}
\multicolumn{1}{l|}{Model} & R\textsuperscript{2} ($\uparrow$) & RMSE ($\downarrow$) & R\textsuperscript{2} ($\downarrow$) & RMSE ($\downarrow$) \\
\midrule
LimiX-16M            & \textbf{0.735} & \textbf{0.433} & \textbf{2.808} & \textbf{3.364} \\
AutoGluon            & 0.722 & 0.448 & 4.818 & 4.929 \\
TabPFN-v2            & 0.695 & 0.465 & 6.192 & 6.101 \\
LimiX-2M             & 0.721 & 0.451 & 6.263 & 6.182 \\
XGBoost              & 0.710 & 0.462 & 7.424 & 7.354 \\
LightGBM             & 0.707 & 0.461 & 8.182 & 8.152 \\
ET                   & 0.696 & 0.476 & 9.970 & 9.919 \\
CatBoost             & 0.700 & 0.471 & 10.424 & 10.354 \\
RF                   & 0.697 & 0.474 & 10.556 & 10.556 \\
ExcelFormer          & 0.653 & 0.512 & 13.313 & 13.354 \\
RealMLP              & 0.656 & 0.510 & 13.444 & 13.444 \\
T2G-Former           & 0.656 & 0.512 & 14.202 & 14.141 \\
TabR                 & 0.651 & 0.516 & 15.283 & 15.242 \\
DCN-v2               & -0.361 & 0.818 & 15.576 & 15.616 \\
FT-Transformer       & 0.648 & 0.519 & 15.848 & 15.798 \\
Mitra                & 0.602 & 0.547 & 15.848 & 15.848 \\
MLP-PLR              & 0.653 & 0.521 & 16.101 & 16.061 \\
ResNet               & 0.562 & 0.550 & 16.616 & 16.586 \\
ModernNCA            & 0.633 & 0.530 & 16.646 & 16.626 \\
TANGOS               & 0.592 & 0.547 & 17.131 & 17.101 \\
MLP                  & 0.556 & 0.564 & 17.596 & 17.636 \\
SAINT                & -1.541 & 0.571 & 18.384 & 18.253 \\
AutoInt              & 0.636 & 0.538 & 19.040 & 19.051 \\
NODE                 & 0.568 & 0.600 & 20.515 & 20.374 \\
TabNet               & 0.576 & 0.586 & 21.747 & 21.687 \\
DNNR                 & -9.172 & 2.528 & 23.505 & 23.545 \\
SNN                  & 0.344 & 0.777 & 25.465 & 25.485 \\
GrowNet              & -0.182 & 0.920 & 27.131 & 27.162 \\
DANets               & 0.005 & 0.998 & 28.434 & 28.455 \\
TabTransformer       & 0.001 & 1.001 & 28.576 & 28.626 \\
SwitchTab            & -0.002 & 1.002 & 28.949 & 28.990 \\
\bottomrule
\end{tabular}
\end{table}

\begin{table}[H]
\centering
\caption{Statistical profile of the benchmark TALENT-REG, where Q10, Q50, and Q90 correspond to the 10\%, 50\%, and 90\% quantiles, respectively; for categorical feature statistics, we only consider features that are either string-typed or have fewer than 10 unique values.}
\resizebox{0.9\textwidth}{!}{
    \begin{tabular}{cl ccccccc}
    \toprule
    & & \multicolumn{7}{c}{Statistics} \\
    \cmidrule(lr){3-9}
     & Metric   & Q10  & Q50 & Q90 & Mean & Std  & Min & Max \\
    \midrule
    & \# Features & 6 & 11 & 79 & 30 & 49 & 4 & 266 \\
    & Missing Values (Ratio) & 0 & 0 & 0 & 0.007 & 0.034 & 0 & 0.283 \\
    & Categorical Features (Ratio) & 0 & 0.131 & 0.667 & 0.238 & 0.283 & 0 & 1 \\
    & Features w/ Missing Values (Ratio) & 0 & 0 & 0.013 & 0.043 & 0.161 & 0 & 0.875 \\
    \bottomrule
    \end{tabular}
}
\label{tab:profile_reg_TALENT-REG}
\end{table}

\clearpage
\begin{table}[H]
\caption{Regression results on the CTR23 benchmark. The best scores are shown in bold.}
\label{tab:results_ctr23}
\centering
\setlength{\tabcolsep}{14pt} 
\begin{tabular}{l|cc|cc}
\toprule
& \multicolumn{4}{c}{\rule{0pt}{18pt}\raisebox{4pt}{CTR23}} \\ \midrule
& \multicolumn{2}{c|}{Mean} & \multicolumn{2}{c}{Rank} \\
\cmidrule(lr){2-3}\cmidrule(lr){4-5}
\multicolumn{1}{l|}{Model} & R\textsuperscript{2} ($\uparrow$) & RMSE ($\downarrow$) & R\textsuperscript{2} ($\downarrow$) & RMSE ($\downarrow$) \\
\midrule
LimiX-16M            & \textbf{0.745} & \textbf{0.477} & \textbf{4.061} & \textbf{4.152} \\
AutoGluon            & 0.725 & 0.497 & 5.303 & 5.212 \\
LimiX-2M             & 0.730 & 0.495 & 6.879 & 6.909 \\
TabPFN-v2            & 0.716 & 0.503 & 7.636 & 7.515 \\
XGBoost              & 0.712 & 0.511 & 8.697 & 8.758 \\
LightGBM             & 0.706 & 0.516 & 9.727 & 9.818 \\
ET                   & 0.697 & 0.535 & 10.515 & 10.576 \\
CatBoost             & 0.700 & 0.528 & 11.212 & 11.212 \\
RF                   & 0.694 & 0.539 & 11.485 & 11.455 \\
ExcelFormer          & 0.665 & 0.556 & 12.515 & 12.606 \\
T2G-Former           & 0.674 & 0.544 & 13.879 & 13.848 \\
DCN-v2               & 0.670 & 0.545 & 14.515 & 14.576 \\
ResNet               & 0.645 & 0.587 & 14.545 & 14.606 \\
RealMLP              & 0.661 & 0.549 & 14.576 & 14.606 \\
FT-Transformer       & 0.667 & 0.549 & 14.848 & 14.939 \\
TabR                 & 0.671 & 0.543 & 14.939 & 14.879 \\
MLP-PLR              & 0.672 & 0.553 & 15.788 & 15.818 \\
MLP                  & 0.608 & 0.623 & 15.848 & 15.848 \\
ModernNCA            & 0.667 & 0.550 & 15.909 & 15.576 \\
Mitra                & 0.624 & 0.583 & 16.727 & 16.848 \\
TANGOS               & 0.642 & 0.586 & 16.727 & 16.758 \\
SAINT                & 0.654 & 0.561 & 16.758 & 16.667 \\
AutoInt              & 0.655 & 0.568 & 17.848 & 17.818 \\
NODE                 & 0.568 & 0.666 & 20.545 & 20.394 \\
TabNet               & 0.605 & 0.623 & 20.576 & 20.636 \\
DNNR                 & -2.969 & 1.651 & 22.636 & 22.667 \\
SNN                  & 0.369 & 0.834 & 25.667 & 25.697 \\
GrowNet              & 0.185 & 0.944 & 27.606 & 27.576 \\
DANets               & 0.001 & 1.052 & 28.970 & 28.970 \\
TabTransformer       & 0.000 & 1.053 & 29.000 & 29.000 \\
SwitchTab            & -0.006 & 1.057 & 30.061 & 30.061 \\
\bottomrule
\end{tabular}
\end{table}

\begin{table}[H]
\centering
\caption{Statistical profile of the benchmark CTR23, where Q10, Q50, and Q90 correspond to the 10\%, 50\%, and 90\% quantiles, respectively; for categorical feature statistics, we only consider features that are either string-typed or have fewer than 10 unique values.}
\resizebox{0.9\textwidth}{!}{
    \begin{tabular}{cl ccccccc}
    \toprule
    & & \multicolumn{7}{c}{Statistics} \\
    \cmidrule(lr){3-9}
     & Metric   & Q10  & Q50 & Q90 & Mean & Std  & Min & Max \\
    \midrule
    & \# Features & 6 & 11 & 39 & 20 & 22 & 5 & 116 \\
    & Missing Values (Ratio) & 0 & 0 & 0 & 0.008 & 0.036 & 0 & 0.209 \\
    & Categorical Features (Ratio) & 0 & 0.143 & 0.842 & 0.312 & 0.345 & 0 & 1 \\
    & Features w/ Missing Values (Ratio) & 0 & 0 & 0 & 0.011 & 0.05 & 0 & 0.286 \\
    \bottomrule
    \end{tabular}
}
\label{tab:profile_reg_CTR23}
\end{table}

\clearpage
\begin{table}[H]
\centering
\renewcommand{\arraystretch}{1.05}
\setlength{\tabcolsep}{14pt} 
\caption{Regression results on the PFN-REG benchmark. The best scores are shown in bold.}
\label{tab:results_reg_pfn}
\begin{tabular}{l|cc|cc}
\toprule
& \multicolumn{4}{c}{\rule{0pt}{18pt}\raisebox{4pt}{PFN-REG}} \\ \midrule
& \multicolumn{2}{c|}{Mean} & \multicolumn{2}{c}{Rank} \\
\cmidrule(lr){2-3}\cmidrule(lr){4-5}
\multicolumn{1}{l|}{Model} & R\textsuperscript{2} ($\uparrow$) & RMSE ($\downarrow$) & R\textsuperscript{2} ($\downarrow$) & RMSE ($\downarrow$) \\
\midrule
LimiX-16M            & 0.692 & 0.461 & \textbf{4.607} & 5.464 \\
LimiX-2M             & \textbf{0.696} & \textbf{0.458} & 5.250 & \textbf{5.250} \\
AutoGluon            & 0.677 & 0.482 & 7.321 & 7.321 \\
TabPFN-v2            & 0.676 & 0.475 & 7.857 & 7.607 \\
XGBoost              & 0.671 & 0.487 & 9.464 & 9.464 \\
RealMLP              & 0.664 & 0.484 & 10.250 & 10.179 \\
LightGBM             & 0.664 & 0.496 & 10.464 & 10.393 \\
CatBoost             & 0.661 & 0.498 & 11.429 & 11.393 \\
ET                   & 0.652 & 0.515 & 11.500 & 11.429 \\
Mitra                & 0.630 & 0.531 & 13.679 & 13.714 \\
RF                   & 0.643 & 0.524 & 13.821 & 13.893 \\
T2G-Former           & 0.640 & 0.503 & 13.964 & 13.857 \\
ExcelFormer          & 0.646 & 0.513 & 14.036 & 14.071 \\
FT-Transformer       & 0.636 & 0.509 & 14.964 & 14.821 \\
MLP                  & 0.577 & 0.576 & 14.964 & 15.107 \\
ModernNCA            & 0.645 & 0.508 & 15.107 & 15.286 \\
ResNet               & 0.599 & 0.553 & 15.286 & 15.321 \\
MLP-PLR              & 0.638 & 0.513 & 15.321 & 15.393 \\
TabR                 & 0.636 & 0.509 & 15.393 & 15.250 \\
TANGOS               & 0.587 & 0.562 & 15.607 & 15.750 \\
DCN-v2               & 0.638 & 0.511 & 15.607 & 15.429 \\
SAINT                & -8.088 & 0.612 & 17.214 & 16.821 \\
AutoInt              & 0.610 & 0.539 & 19.143 & 19.036 \\
NODE                 & 0.496 & 0.666 & 20.964 & 20.786 \\
TabNet               & 0.430 & 0.643 & 22.071 & 22.000 \\
DNNR                 & -7.850 & 2.224 & 23.286 & 23.429 \\
SNN                  & 0.378 & 0.756 & 25.179 & 25.214 \\
GrowNet              & 0.096 & 0.942 & 26.607 & 26.571 \\
DANets               & 0.001 & 0.997 & 27.893 & 27.929 \\
SwitchTab            & -0.026 & 1.011 & 28.393 & 28.429 \\
TabTransformer       & -0.021 & 1.007 & 29.357 & 29.393 \\
\bottomrule
\end{tabular}
\end{table}

\begin{table}[H]
\centering
\caption{Statistical profile of the benchmark PFN-REG, where Q10, Q50, and Q90 correspond to the 10\%, 50\%, and 90\% quantiles, respectively; for categorical feature statistics, we only consider features that are either string-typed or have fewer than 10 unique values.}
\resizebox{0.9\textwidth}{!}{
    \begin{tabular}{cl ccccccc}
    \toprule
    & & \multicolumn{7}{c}{Statistics} \\
    \cmidrule(lr){3-9}
     & Metric   & Q10  & Q50 & Q90 & Mean & Std  & Min & Max \\
    \midrule
    & \# Features & 6 & 16 & 176 & 66 & 93 & 3 & 376 \\
    & Missing Values (Ratio) & 0 & 0 & 0.089 & 0.029 & 0.076 & 0 & 0.335 \\
    & Categorical Features (Ratio) & 0 & 0.209 & 0.949 & 0.355 & 0.381 & 0 & 1 \\
    & Features w/ Missing Values (Ratio) & 0 & 0 & 0.211 & 0.059 & 0.169 & 0 & 0.841 \\
    \bottomrule
    \end{tabular}
}
\label{tab:profile_reg_pfn}
\end{table}

\clearpage
\begin{table}[H]
\centering
\renewcommand{\arraystretch}{1.05}
\setlength{\tabcolsep}{14pt} 
\caption{Regression results on the TabArena-REG benchmark. The best scores are shown in bold.}
\label{tab:results_reg_tabarena}
\begin{tabular}{l|cc|cc}
\toprule
& \multicolumn{4}{c}{\rule{0pt}{18pt}\raisebox{4pt}{TabArena-REG}} \\ \midrule
& \multicolumn{2}{c|}{Mean} & \multicolumn{2}{c}{Rank} \\
\cmidrule(lr){2-3}\cmidrule(lr){4-5}
\multicolumn{1}{l|}{Model} & R\textsuperscript{2} ($\uparrow$) & RMSE ($\downarrow$) & R\textsuperscript{2} ($\downarrow$) & RMSE ($\downarrow$) \\
\midrule
LimiX-16M            & \textbf{0.796} & \textbf{0.406} & \textbf{2.923} & \textbf{2.923} \\
AutoGluon            & 0.791 & 0.414 & 3.154 & 3.154 \\
LimiX-2M             & 0.788 & 0.413 & 3.846 & 3.846 \\
TabPFN-v2            & 0.777 & 0.422 & 5.385 & 5.385 \\
CatBoost             & 0.774 & 0.431 & 6.231 & 6.231 \\
XGBoost              & 0.778 & 0.430 & 6.231 & 6.231 \\
LightGBM             & 0.771 & 0.435 & 6.769 & 6.769 \\
RF                   & 0.758 & 0.456 & 9.769 & 9.769 \\
ET                   & 0.746 & 0.464 & 10.231 & 10.154 \\
TabR                 & 0.729 & 0.476 & 13.692 & 13.692 \\
ExcelFormer          & 0.681 & 0.521 & 13.923 & 13.923 \\
NODE                 & 0.665 & 0.542 & 14.538 & 14.538 \\
TANGOS               & 0.673 & 0.534 & 15.077 & 15.077 \\
DCN-v2               & 0.718 & 0.486 & 15.385 & 15.462 \\
RealMLP              & 0.719 & 0.487 & 15.385 & 15.385 \\
Mitra                & 0.666 & 0.538 & 16.615 & 16.615 \\
ModernNCA            & 0.712 & 0.496 & 17.231 & 17.231 \\
MLP-PLR              & 0.714 & 0.496 & 17.231 & 17.231 \\
AutoInt              & 0.705 & 0.503 & 17.923 & 17.923 \\
ResNet               & 0.687 & 0.521 & 18.538 & 18.538 \\
MLP                  & 0.694 & 0.516 & 18.538 & 18.538 \\
SAINT                & 0.712 & 0.498 & 19.000 & 19.000 \\
T2G-Former           & 0.677 & 0.526 & 19.385 & 19.385 \\
TabNet               & 0.641 & 0.564 & 19.615 & 19.615 \\
FT-Transformer       & 0.626 & 0.572 & 22.077 & 22.077 \\
DNNR                 & -0.368 & 1.058 & 24.538 & 24.538 \\
SNN                  & 0.451 & 0.729 & 25.692 & 25.692 \\
GrowNet              & 0.316 & 0.814 & 27.385 & 27.385 \\
TabTransformer       & 0.016 & 0.993 & 29.333 & 29.333 \\
DANets               & 0.012 & 0.998 & 29.462 & 29.462 \\
SwitchTab            & 0.004 & 1.002 & 29.923 & 29.923 \\
\bottomrule
\end{tabular}
\end{table}

\begin{table}[H]
\centering
\caption{Statistical profile of the benchmark TabArena-REG, where Q10, Q50, and Q90 correspond to the 10\%, 50\%, and 90\% quantiles, respectively; for categorical feature statistics, we only consider features that are either string-typed or have fewer than 10 unique values.}
\resizebox{0.9\textwidth}{!}{
    \begin{tabular}{cl ccccccc}
    \toprule
    & & \multicolumn{7}{c}{Statistics} \\
    \cmidrule(lr){3-9}
     & Metric   & Q10  & Q50 & Q90 & Mean & Std  & Min & Max \\
    \midrule
    & \# Features & 6 & 9 & 67.8 & 92.23 & 280.67 & 5 & 1024 \\
    & Missing Values (Ratio) & 0 & 0 & 0 & 0 & 0 & 0 & 0 \\
    & Categorical Features (Ratio) & 0 & 0.333 & 0.619 & 0.287 & 0.299 & 0 & 1 \\
    & Features w/ Missing Values (Ratio) & 0 & 0 & 0 & 0 & 0 & 0 & 0 \\
    \bottomrule
    \end{tabular}
}
\label{tab:profile_reg_tabarena}
\end{table}

\begin{figure}[H]
    \centering
    \begin{subfigure}[b]{0.9\textwidth}
        \centering
        \includegraphics[width=\textwidth]{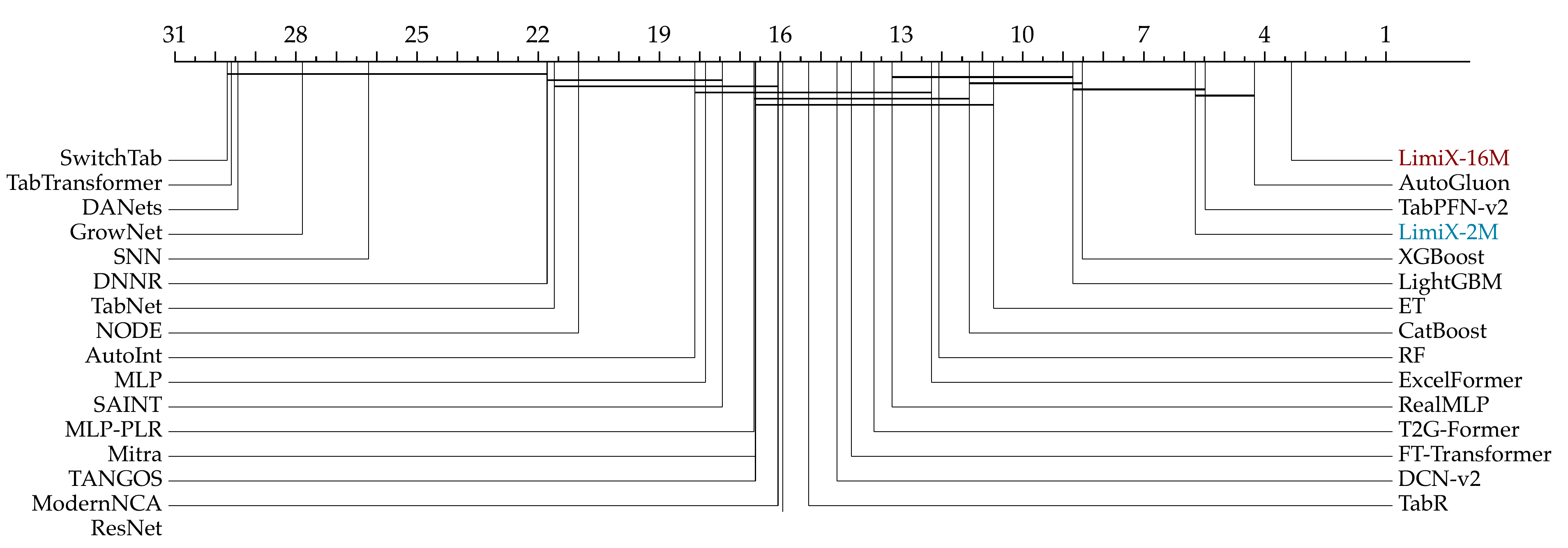}
        \caption{R\textsuperscript{2} on the BCCO-REG benchmark}
        \label{fig:cdd-bcco-reg-sub1}
    \end{subfigure}

    \begin{subfigure}[b]{0.85\textwidth}
        \centering
        \includegraphics[width=\textwidth]{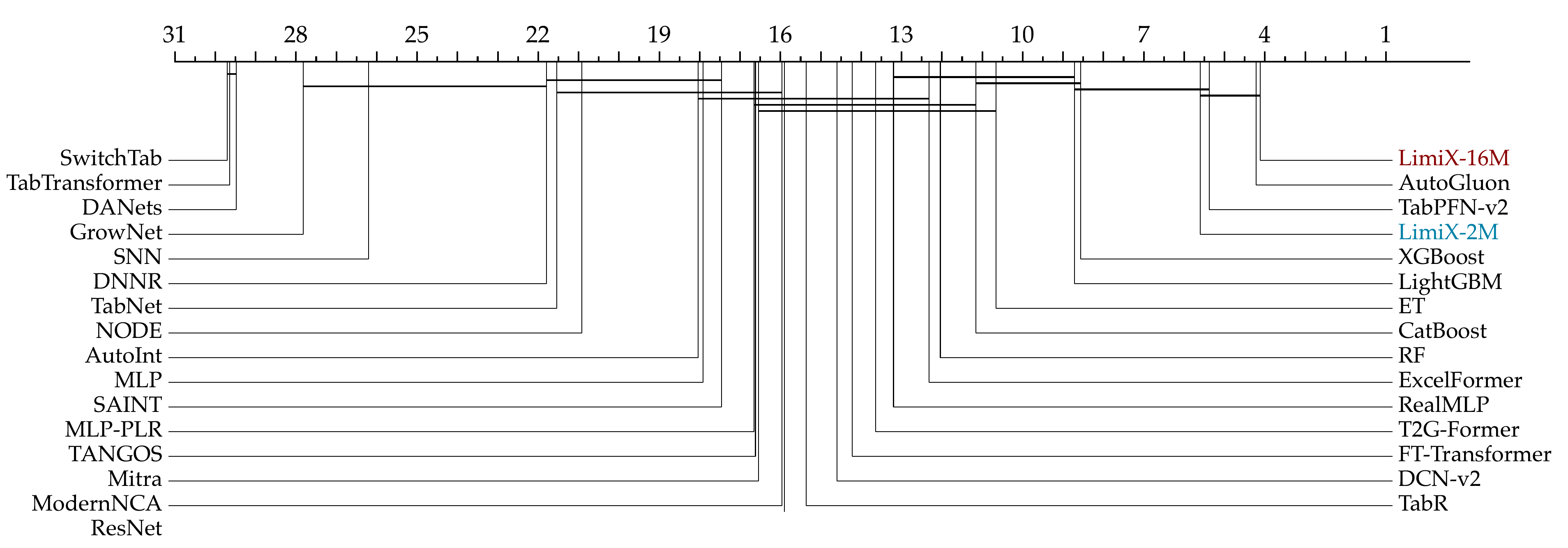}
        \caption{RMSE on the BCCO-REG benchmark}
        \label{fig:cdd-bcco-reg-sub2}
    \end{subfigure}
    
    \caption{Critical difference diagram on the BCCO-REG benchmark}
    \label{fig:cdd-bcco-reg}
\end{figure}

\begin{figure}[H]
    \centering
    \begin{subfigure}[b]{0.85\textwidth}
        \centering
        \includegraphics[width=\textwidth]{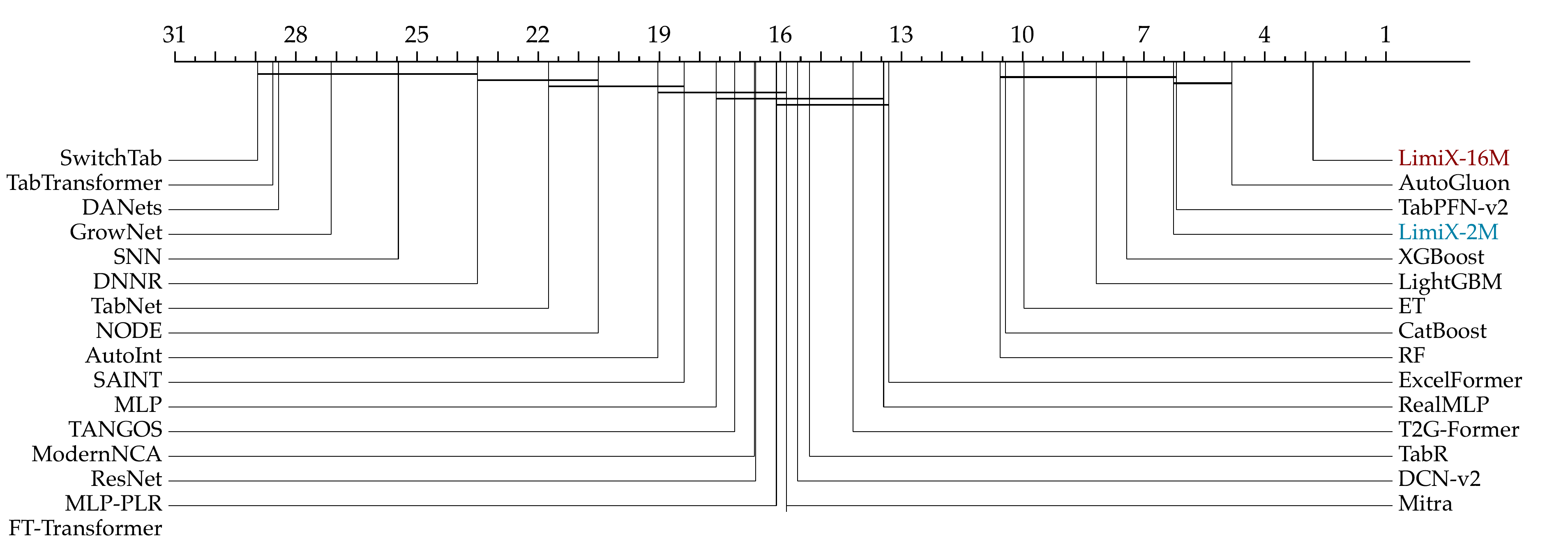}
        \caption{R\textsuperscript{2} on the TALENT-REG benchmark}
        \label{fig:cdd-talent-reg-sub1}
    \end{subfigure}

    \begin{subfigure}[b]{0.9\textwidth}
        \centering
        \includegraphics[width=\textwidth]{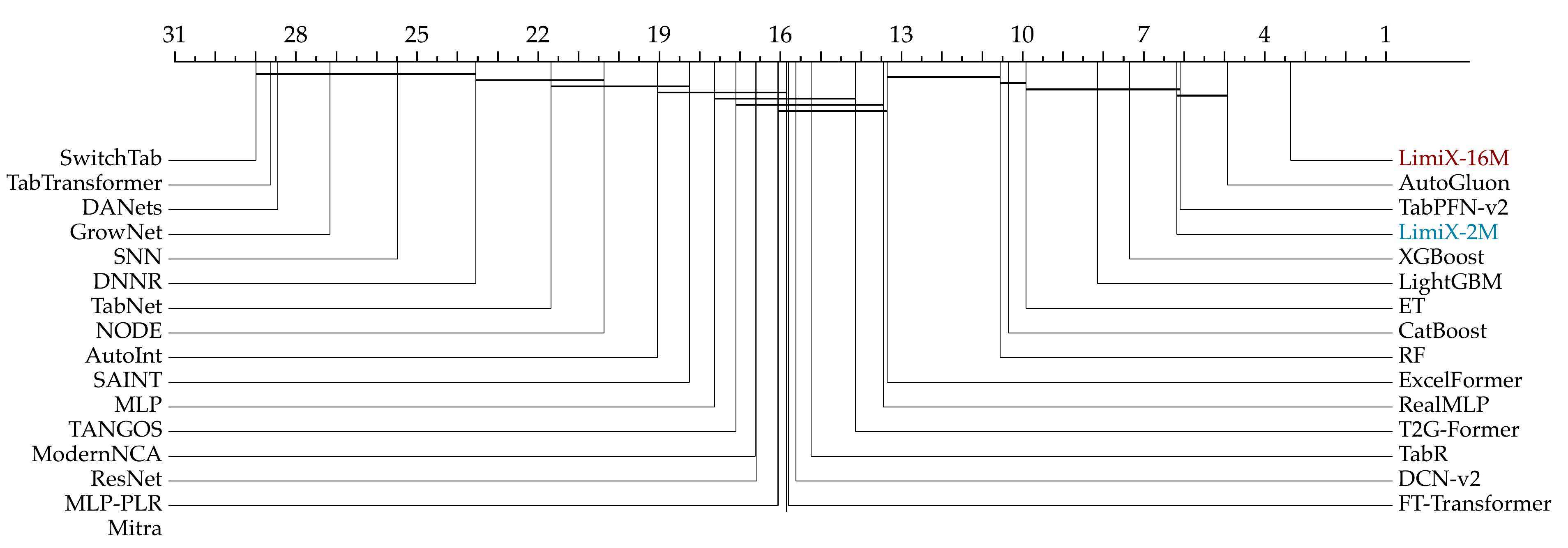}
        \caption{RMSE on TALENT-REG benchmark}
        \label{fig:cdd-talent-reg-sub2}
    \end{subfigure}
    
    \caption{Critical difference diagram on the TALENT-REG benchmark}
    \label{fig:cdd-talent-reg}
\end{figure}

\begin{figure}[H]
    \centering
    \begin{subfigure}[b]{0.9\textwidth}
        \centering
        \includegraphics[width=\textwidth]{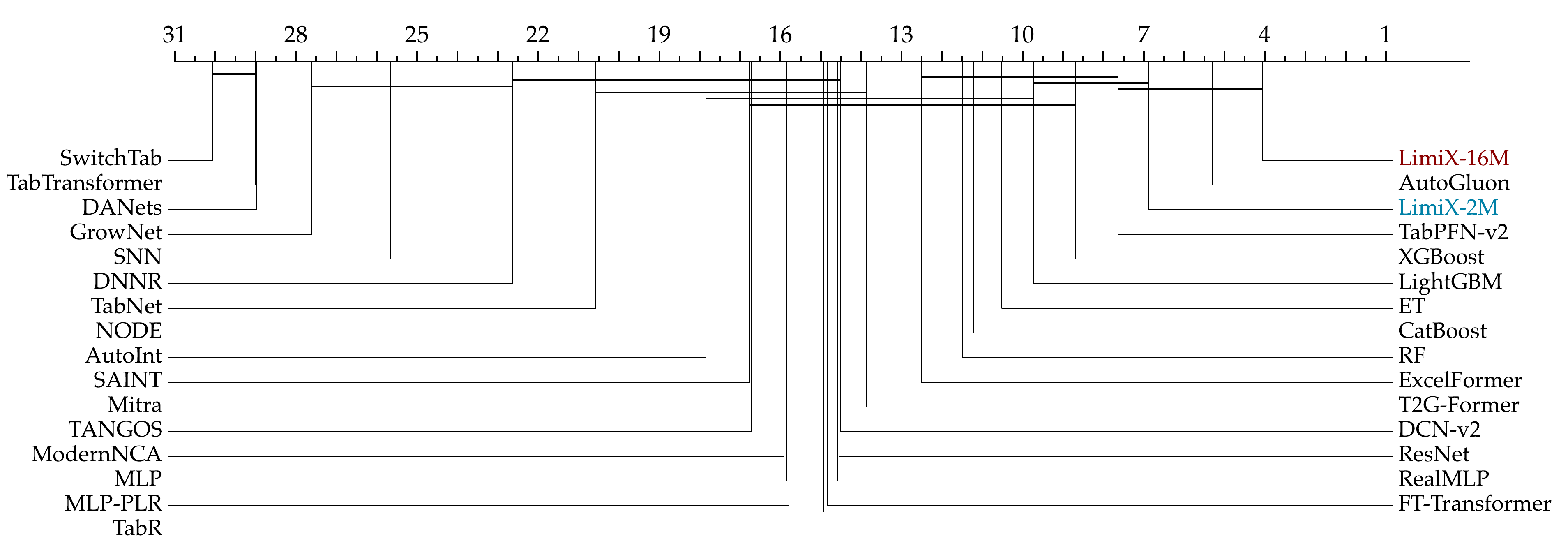}
        \caption{R\textsuperscript{2} on the CTR23 benchmark}
        \label{fig:cdd-ctr23-sub1}
    \end{subfigure}
    \begin{subfigure}[b]{0.9\textwidth}
        \centering
        \includegraphics[width=\textwidth]{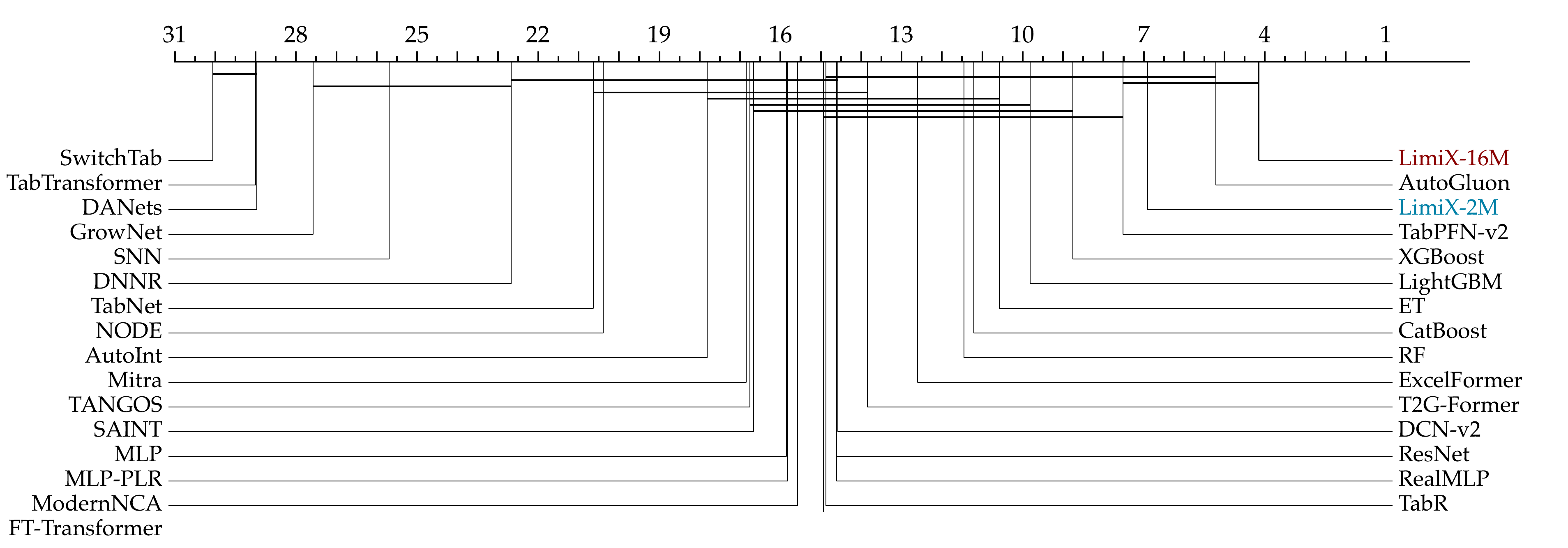}
        \caption{RMSE on the CTR23 benchmark}
        \label{fig:cdd-ctr23-sub2}
    \end{subfigure}
    \caption{Critical difference diagram on the CTR23 benchmark}
    \label{fig:cdd-ctr23-reg} 
\end{figure}

\begin{figure}[H]
    \centering
    \begin{subfigure}[b]{0.85\textwidth}
        \centering
        \includegraphics[width=\textwidth]{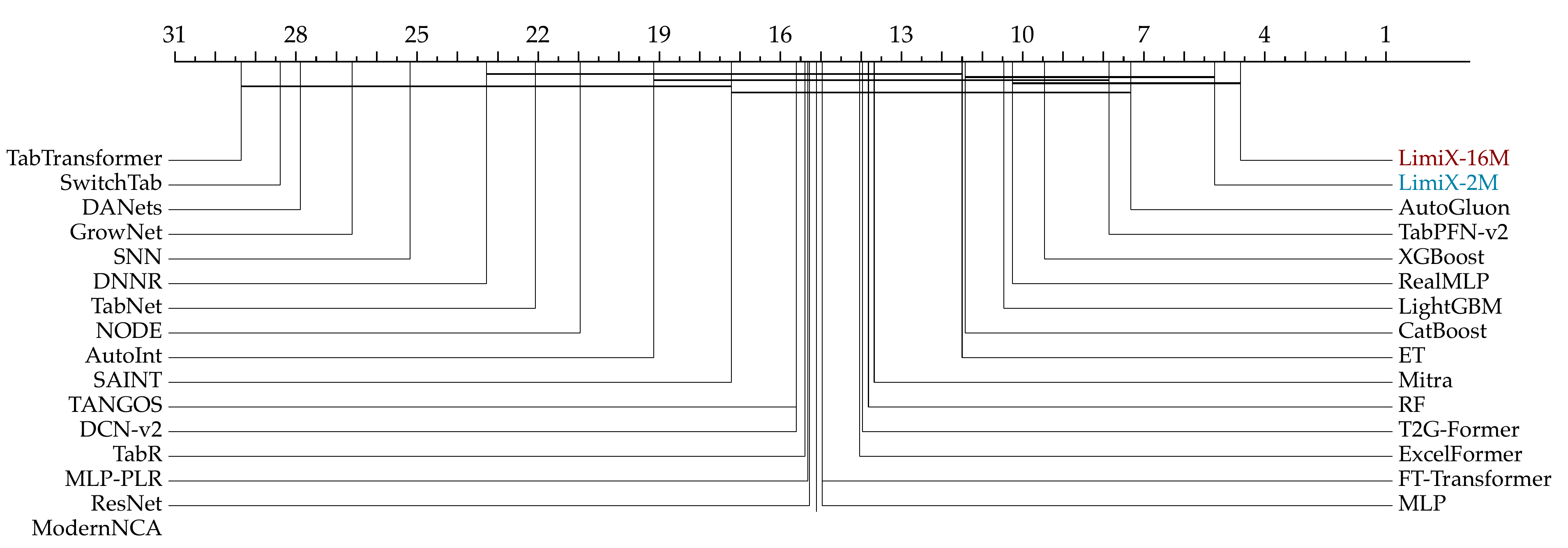}
        \caption{R\textsuperscript{2} on the PFN-REG benchmark}
        \label{fig:cdd-PFN-reg-sub1}
    \end{subfigure}
    \begin{subfigure}[b]{0.85\textwidth}
        \centering
        \includegraphics[width=\textwidth]{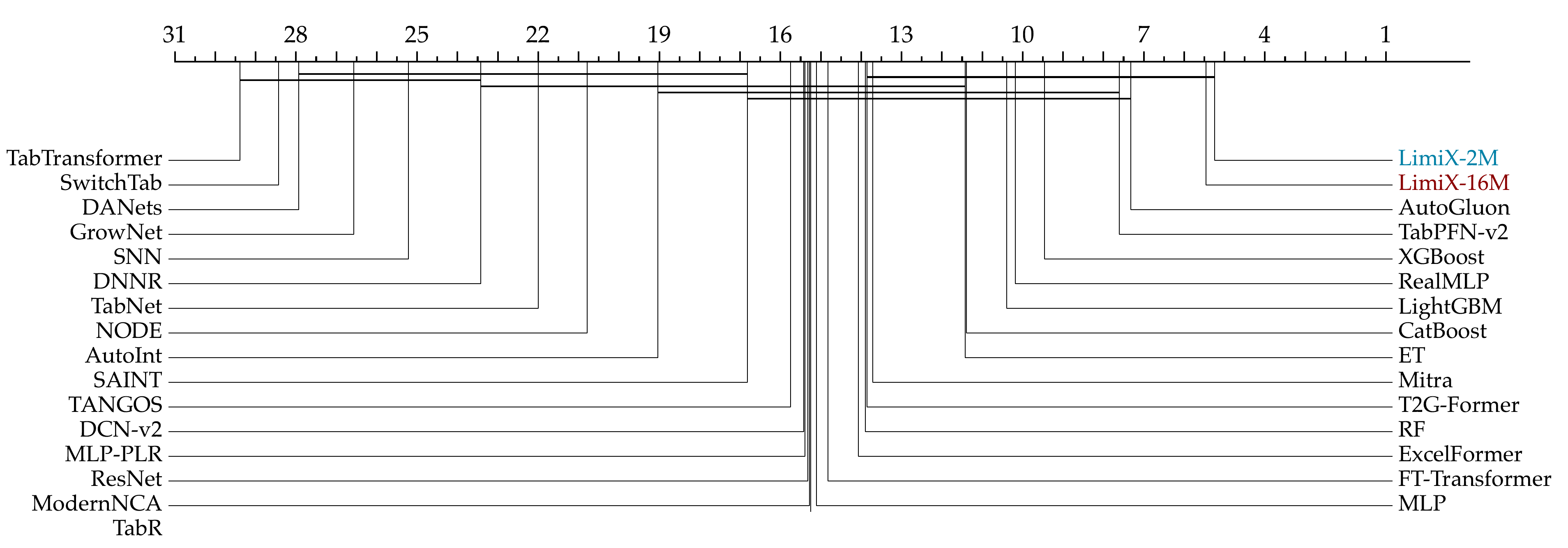}
        \caption{RMSE on PFN-REG benchmark}
        \label{fig:cdd-PFN-reg-sub2}
    \end{subfigure}
    \caption{Critical difference diagram on the PFN-REG benchmark}
    \label{fig:cdd-PFN-reg} 
\end{figure}

\begin{figure}[H]
    \centering
    \begin{subfigure}[b]{0.9 \textwidth}
        \centering
        \includegraphics[width=\textwidth]{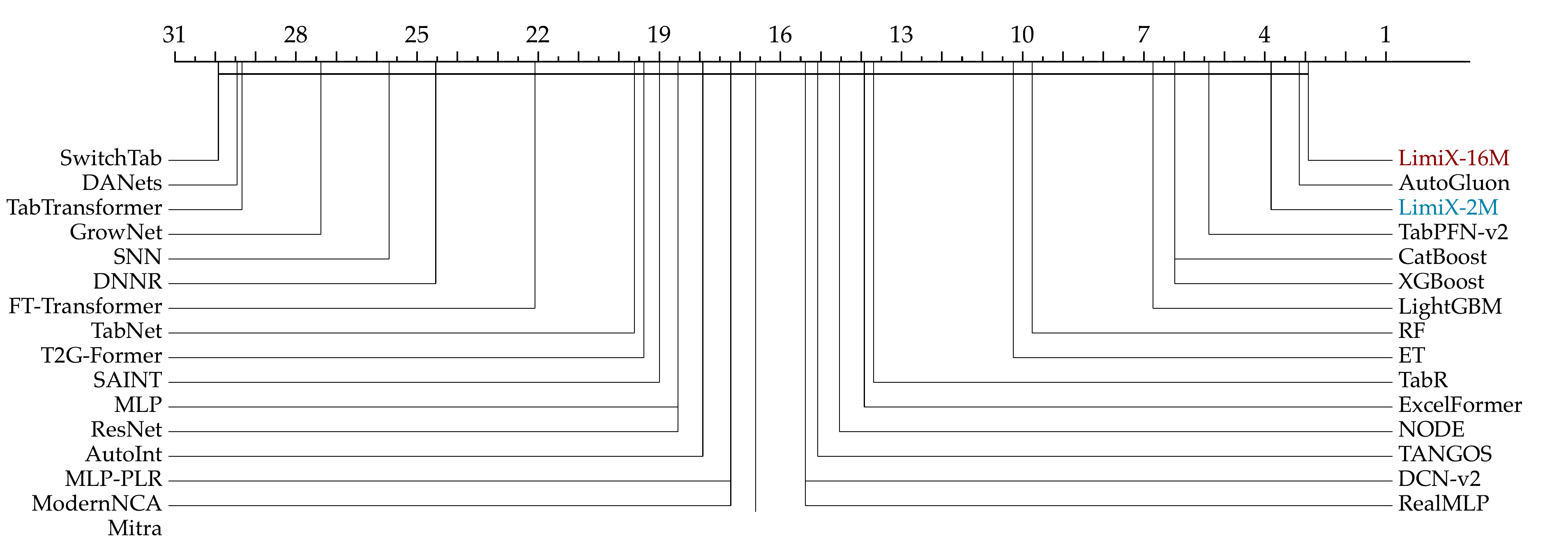}
        \caption{R\textsuperscript{2} on the TabArena-REG benchmark.}
        \label{fig:cdd106-sub1}
    \end{subfigure}
    
    \vspace{2em}
    
    \begin{subfigure}[b]{0.9 \textwidth}
        \centering
           \includegraphics[width=\textwidth]{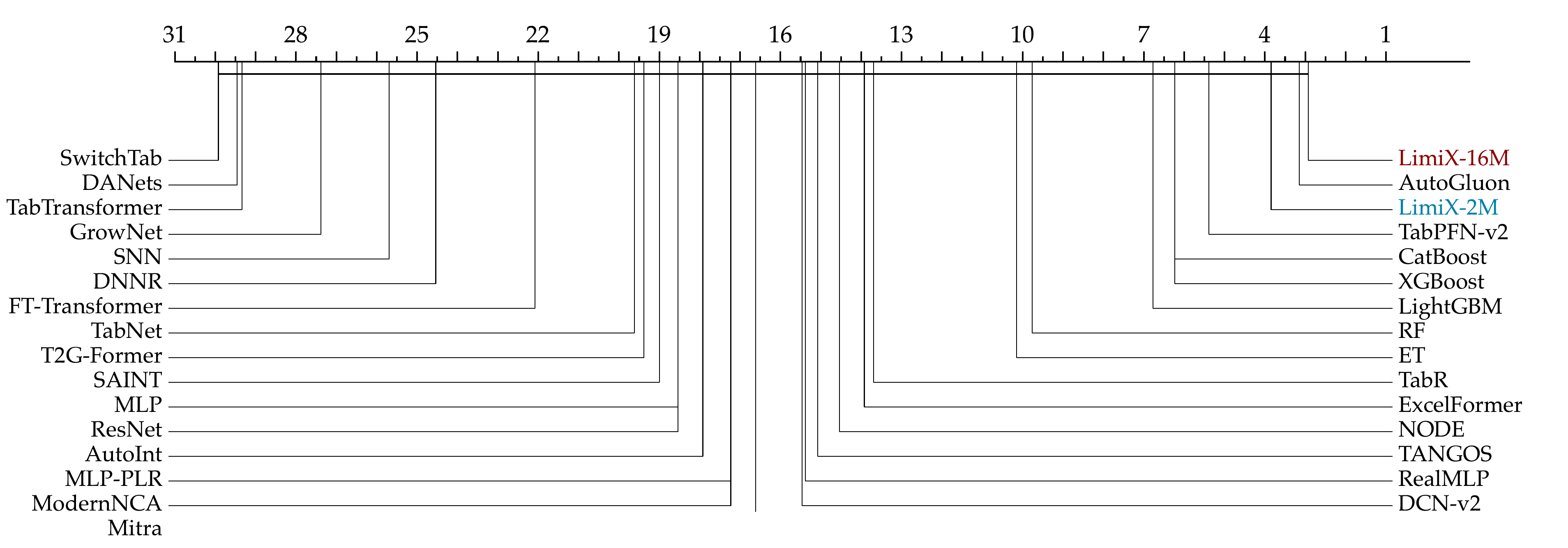}
        \caption{RMSE on the TabArena-REG benchmark.}
        \label{fig:cdd106-sub2}
    \end{subfigure}
    
    \caption{Critical difference diagrams on TabArena-REG benchmark.}
    \label{fig:cdd_tabarenareg}
\end{figure}

\clearpage

\begin{figure}[H]
    \centering
    \begin{subfigure}[b]{0.29\textwidth}
        \centering
        \includegraphics[width=\textwidth]{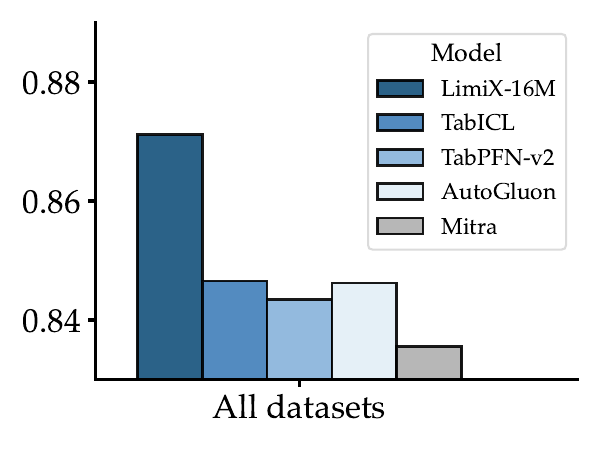}
        \caption{All datasets}
        \label{fig:barplot_group-sub1}
    \end{subfigure}
    \hfill
    \begin{subfigure}[b]{0.29\textwidth}
        \centering
        \includegraphics[width=\textwidth]{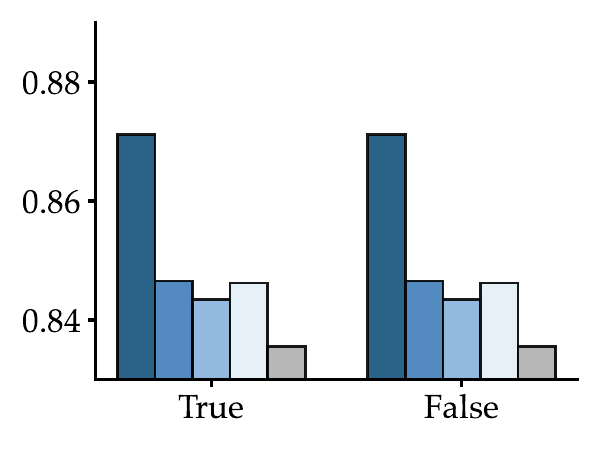}
        \caption{Binary classification}
        \label{fig:barplot_group-sub2}
    \end{subfigure}
    \hfill
    \begin{subfigure}[b]{0.36\textwidth}
        \centering
        \includegraphics[width=\textwidth]{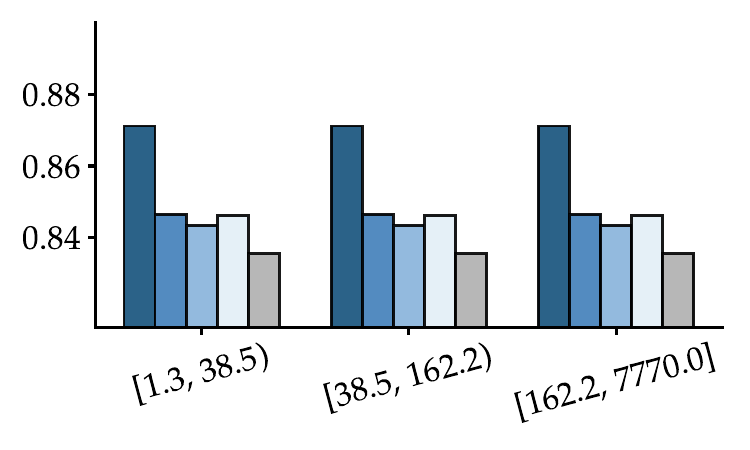}
        \caption{Length-to-width ratio}
        \label{fig:barplot_group-sub3}
    \end{subfigure}
    \vspace{0.5cm}
    \begin{subfigure}[b]{0.29\textwidth}
    \centering
    \includegraphics[width=\textwidth]{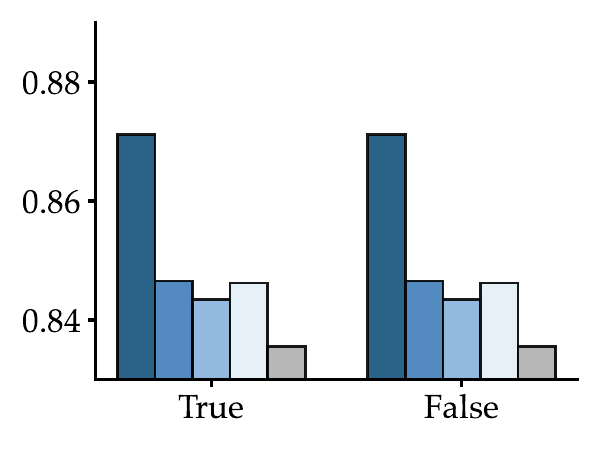}
    \caption{Presence of missing values}
    \label{fig:barplot_group-sub5}
    \end{subfigure}
    \hfill
    \begin{subfigure}[b]{0.29\textwidth}
    \centering
    \includegraphics[width=\textwidth]{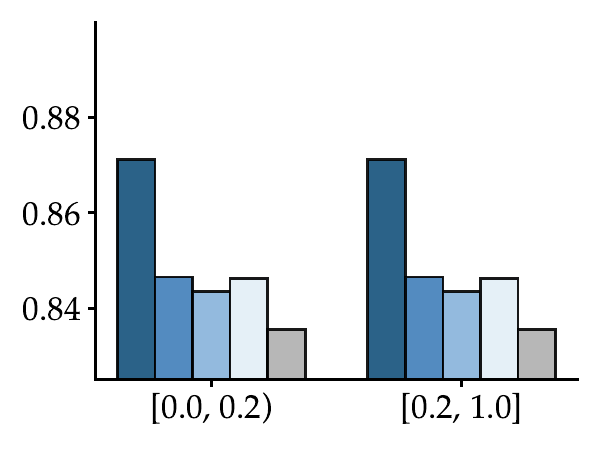}
    \caption{Categorical-feature proportion}
    \label{fig:barplot_group-sub6}
    \end{subfigure}
    \hfill
    \begin{subfigure}[b]{0.36\textwidth}
        \centering
        \includegraphics[width=\textwidth]{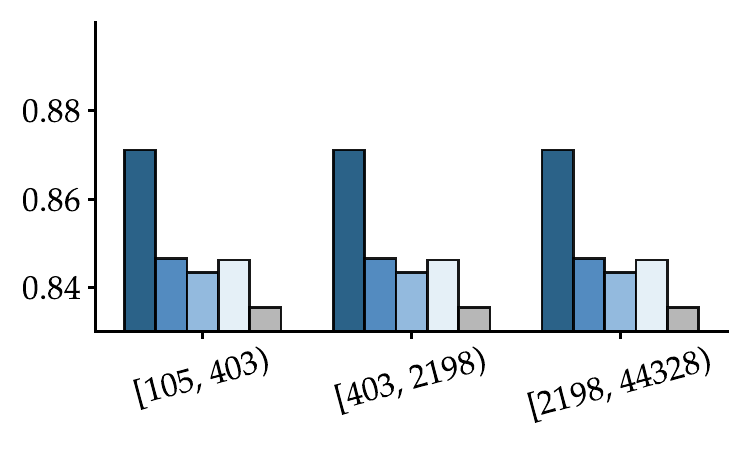}
        \caption{Number of training samples}
        \label{fig:barplot_group-sub4}
    \end{subfigure}
    \caption{AUC on subsets with various sample size, number of classes, categorical–to-numerical feature ratios, missing values, and sample-to-feature ratios.}
    \label{fig:Subgroup analysis}
\end{figure}

\begin{figure}[H]
    \centering
    \begin{subfigure}[b]{0.35\textwidth}
        \centering
        \includegraphics[width=\textwidth]{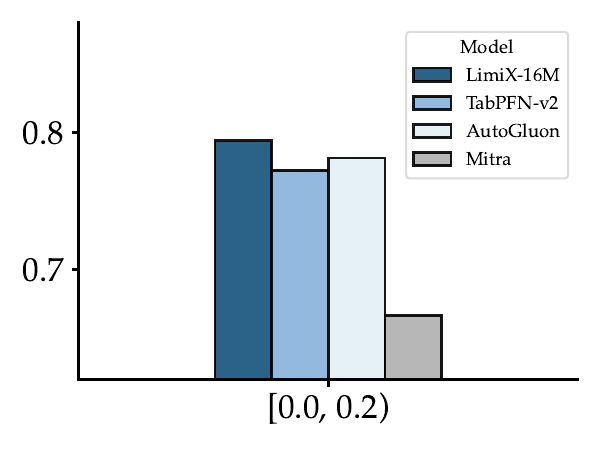}
        \caption{All datasets}
        \label{fig:barplot_reg_group-sub1}
    \end{subfigure}
    \hspace{20pt}
    \begin{subfigure}[b]{0.35\textwidth}
        \centering
        \includegraphics[width=\textwidth]{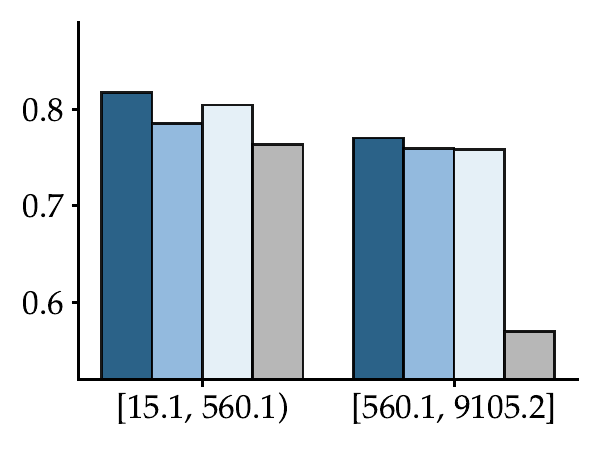}
        \caption{Length-to-width ratio}
        \label{fig:barplot_reg_group-sub2}
    \end{subfigure}
    \hfill
    \begin{subfigure}[b]{0.35\textwidth}
    \centering
    \includegraphics[width=\textwidth]{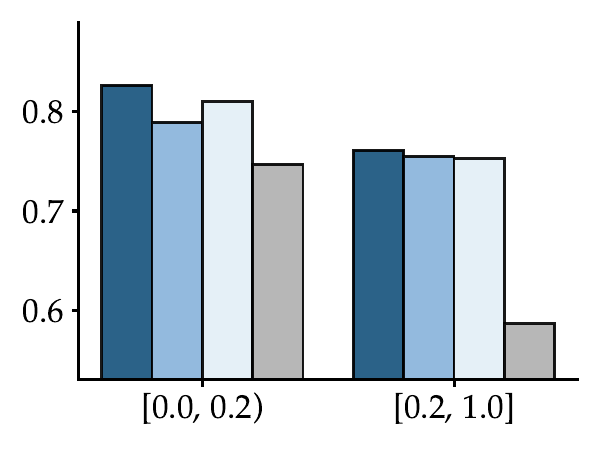}
    \caption{Categorical-feature proportion}
    \label{fig:barplot_reg_group-sub3}
    \end{subfigure}
    \hspace{20pt}
    \begin{subfigure}[b]{0.35\textwidth}
        \centering
        \includegraphics[width=\textwidth]{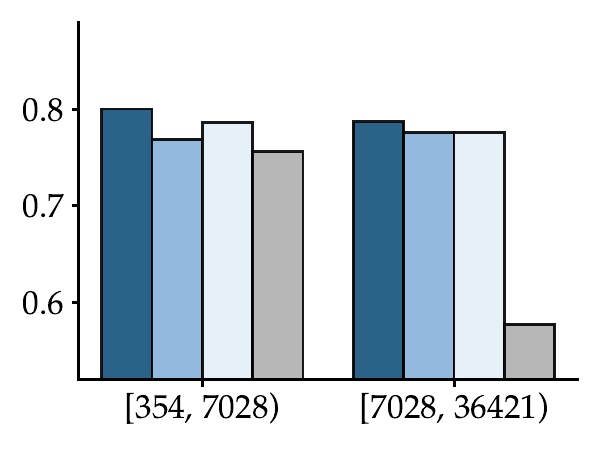}
        \caption{Number of training samples}
        \label{fig:barplot_reg_group-sub4}
    \end{subfigure}
    \caption{R\textsuperscript{2} on subsets with various sample size, categorical–to-numerical feature ratios, and sample-to-feature ratios.}
    \label{fig:Subgroup analysis regression}
\end{figure}

\clearpage
\subsection{Missing Value Imputation}

The ubiquity of missing values in tabular data is harmful to downstream tasks or statistical analyses~\citep{lin2020missing}. 
Meanwhile, the problem of missing value imputation poses a bigger challenge compared with standard classification and regression tasks since it requires modeling of the joint distribution of $P(X,Y)$ instead of $P(Y|X)$ only. 
A noteworthy capability of \ourA~is missing value imputation, which naturally stems from the mask modeling during pretraining. 
Although there are existing deep learning approaches for missing value imputation~\citep{jarrett2022hyperimpute,zhang2025diffputer}, they all require additional training on unseen datasets. 
In contrast, \ourA~are the first to perform missing value imputation on unseen datasets via in-context learning without additional training. This brings great convenience to the usage in downstream tasks.

We conduct experiments on the following datasets: Analcatdata BroadwayMult~\citep{simonoff2003analyzing,openml_analcatdata}, Early Stage Diabetes 2020~\citep{islam2019likelihood,early_stage_diabetes_risk_prediction_529}, Forty Soybean Cultivars from Subsequent Harvests~\citep{de2023dataset,forty_soybean_cultivars_from_subsequent_harvests_913}, HCC survival~\citep{santos2015new,hcc_survival_423}, Vehicle~\citep{siebert1987vehicle,statlog_(vehicle_silhouettes)_149}, Eucalyptus~\citep{openml_eucalyptus}, and Z-Alizadeh Sani~\citep{z-alizadeh_sani_412}.
To simulate the scenario of missing value imputation, we randomly mask a fixed proportion $\alpha$ of cells in $\rvx_{te}$ and try to recover values of these cells. In our experiments, we set $\alpha=0.05$.
Baselines include: Mean/Mode (use mean for continuous features and mode for categorical features), k-nearest neighbors (KNN), MICE~\citep{van2011mice}, MissForest~\citep{stekhoven2012missforest}, GAIN~\citep{yoon2018gain}, MIWAE~\citep{mattei2019miwae}, and HyperImpute~\citep{jarrett2022hyperimpute}. 
For continuous features, we normalize each feature using MinMaxScaler and calculate RMSE between imputed values and ground truth values after normalization as the evaluation metric. 
For categorical features, the error rate is used as the evaluation metric. 
\Cref{tab:imputation} shows that \ourL~consistently outperforms previous baselines, all of which require additional training or fitting.

\begin{table}[H]
\centering
\caption{Evaluation of missing value imputation. For continuous features, we calculate RMSE. For categorical features, we calculate classification error. Our method outperforms previous missing value imputation methods. }
\label{tab:imputation}
\resizebox{\textwidth}{!}{%
\begin{tabular}{@{}c|ccccc|cc@{}}
\toprule
Metric      & \multicolumn{5}{c|}{Regression RMSE (↓)}                             & \multicolumn{2}{c}{Classification Error (↓)} \\ \midrule
Method      & Analcatdata & Early Diabetes & Harvests & Hcc Survival & Vehicle & Eucalyptus      & Z-Alizadeh Sani      \\ \midrule
Mean/Mode   & 0.321       & 0.235          & 0.179    & 0.351        & 0.209   & 0.804           & 0.200                  \\
KNN         & 0.358       & 0.205          & 0.158    & 0.246        & 0.083   & 0.275           & 0.206                  \\
MICE        & 0.294       & 0.244          & 0.155    & 0.234        & 0.102   & 0.627           & 0.181                  \\
MissForest  & 0.203       & 0.223          & 0.136    & 0.215        & 0.107   & 0.137           & 0.156                  \\
GAIN        & 0.299       & 0.328          & 0.196    & 0.413        & 0.102   & 0.647           & 0.175                  \\
MIWAE       & 0.561       & 0.478          & 0.322    & 0.639        & 0.415   & 0.706           & 0.294                  \\
HyperImpute & 0.272       & 0.270          & 0.152    & 0.297        & 0.086   & 0.647           & 0.194                  \\
\ourL       & \textbf{0.194}       & \textbf{0.161}          & \textbf{0.104}    & \textbf{0.188}        & \textbf{0.064}   & \textbf{0.118}           & \textbf{0.131}                  \\ \bottomrule
\end{tabular}%
}
\end{table}

\clearpage
\subsection{Robustness}

Neural networks have been found to be vulnerable to various types of perturbations and attacks~\citep{carlini2017towards,akhtar2018threat}. 
In order to investigate the robustness of \ours, following~\citet{hollmann2025tabpfn_nature}, we conduct experiments under two types of controlled perturbations: adding uninformative features and outliers. 

For uninformative features, we randomly select columns from the original dataset, shuffle each column, and concatenate them to the original dataset. 
For outliers, we multiply a outlier coefficient to each cell value in the original dataset with a probability of 2\%. The outlier coefficient is randomly chosen between 0 and the outlier factor.

In \Cref{fig:robustness-cls}, we show Normalized AUC in classification tasks under the two types of perturbations. We compare \ours~with TabPFN-v2, TabICL, and CatBoost.
The left figure shows that even when adding up to $90\%$ uninformative features, normalized AUC of \ours~remains nearly unchanged. In contrast, the performance of TabICL and CatBoost drops significantly. 
This indicates that \ours~is much more robust than TabICL and CatBoost. 
The right figure shows that \ours~consistently outperforms other methods regardless of the outlier factor. 
In \Cref{fig:robustness-reg}, we show RMSE in regression tasks under perturbations. In the left figure of adding uninformative features, a similar phenomenon is observed as that in classification tasks. In the right figure of adding outliers, we find that RMSE of TabPFN-v2 rapidly increases when the outlier factor grows from 100 to 10000, while \ours~do not. This demonstrates the superior robustness of \ourL~compared with TabPFN-v2 in regression tasks. 
Overall, \ours~exhibits superior robustness relative to the baselines. 

\begin{figure}[H]
    \centering
    \includegraphics[width=0.9\linewidth]{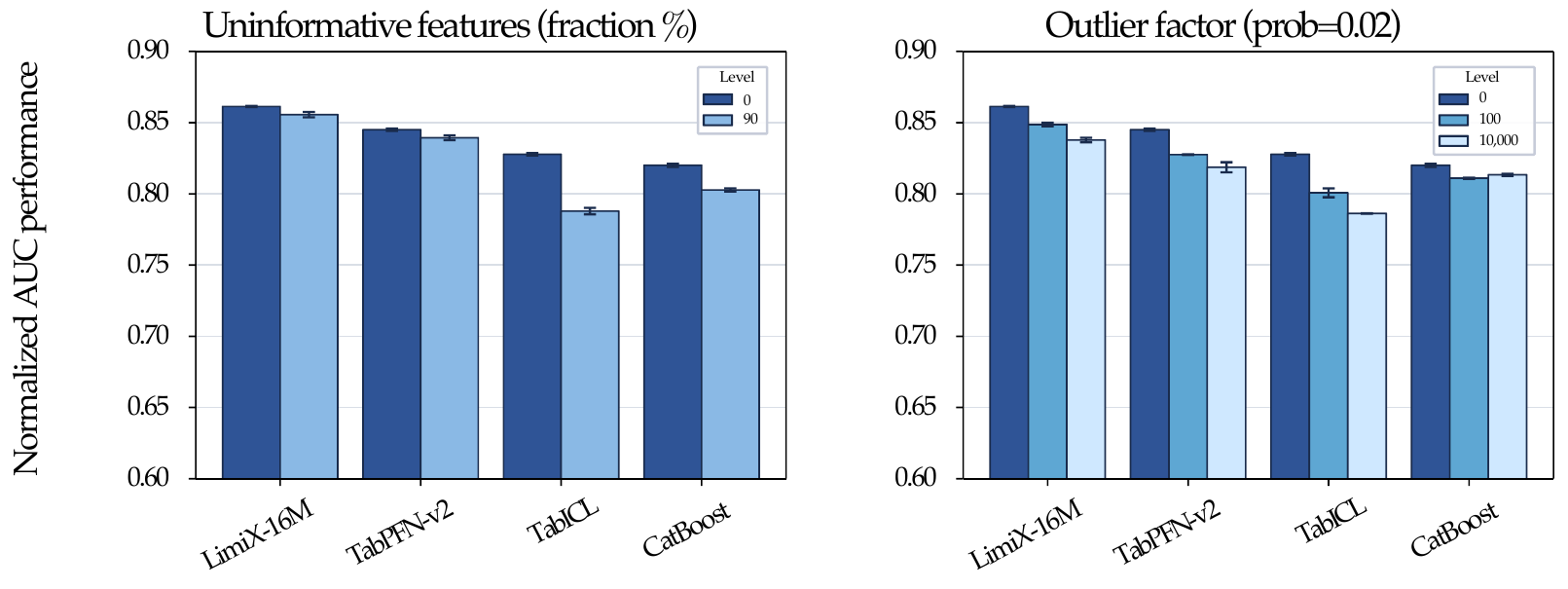}
    \caption{Robustness analysis in classification Tasks. \ours~consistently exhibits the best performance and superior robustness under perturbations.}
    \label{fig:robustness-cls}
\end{figure}

\begin{figure}[H]
    \centering
    \includegraphics[width=0.9\linewidth]{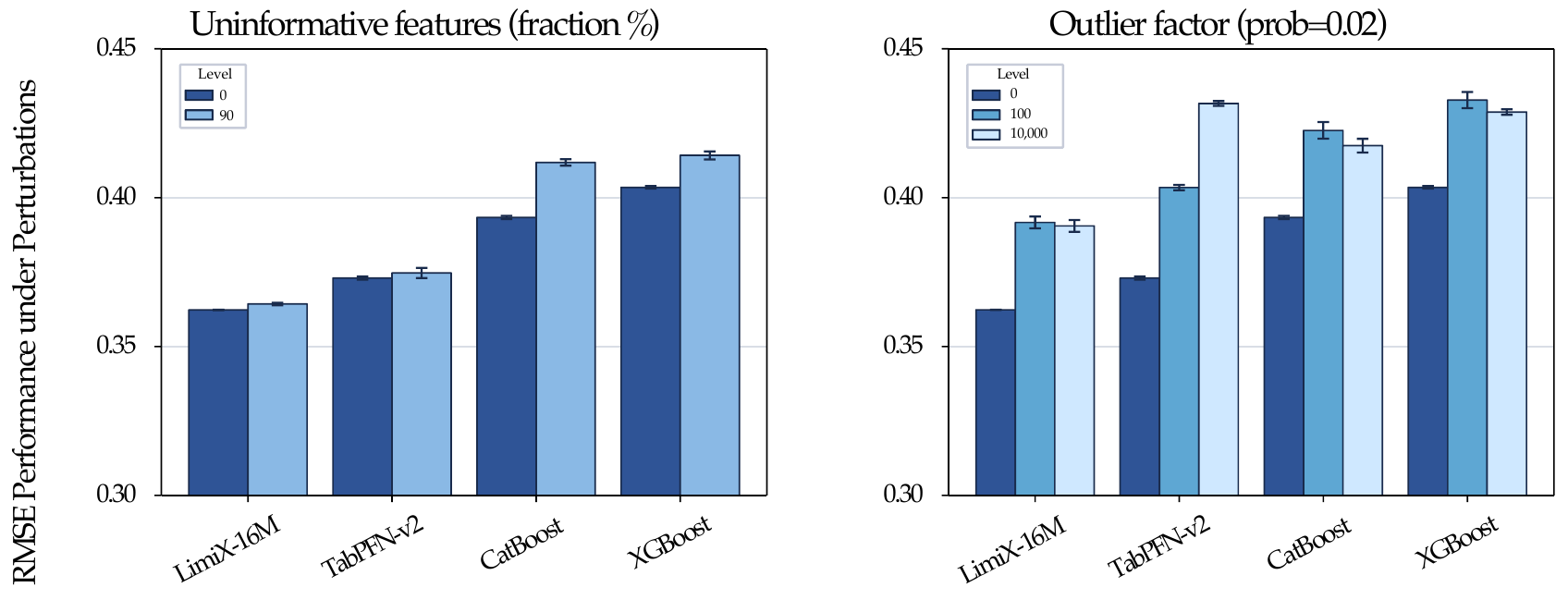}
    \caption{Robustness analysis in regression tasks. \ours~consistently exhibits the best performance and superior robustness under perturbations. }
    \label{fig:robustness-reg}
\end{figure}

\subsection{Embedding}

Pretrained models~\citep{devlin2019bert, oquab2023dinov2} are generally expected to provide effective and transferable feature representations that can be leveraged for downstream tasks. 
To assess the quality of embeddings extracted by \ourL, we conduct experiments on six datasets of various sample sizes and numbers of categories. We compare embeddings of \ourL~with those of MLP~\citep{goodfellow2016deep, gorishniy2021revisiting}, ResNet~\citep{he2016deepresnet, gorishniy2021revisiting}, ModernNCA~\citep{ye2024revisitingmodernnca}, TabPFN-v2~\citep{hollmann2025tabpfn_nature}, and TabICL~\citep{qu2025tabicl}. 
For each model, we treat representations prior to the classification head as embeddings. 
\Cref{fig:embed_tsne} shows t-SNE visualization of embeddings. We can see that embeddings of different categories extracted by \ourL~are more separated than those extracted by other models.

To further evaluate the quality of embeddings extracted by \ourL, we additionally conduct linear probing experiments, which is widely adopted in the analysis of feature representations~\citep{kumar2022fine}. 
The experiments are conducted on BCCO-CLS. 
\Cref{tab:linear prob} shows that embeddings of \ourL~achieve the highest average AUC and ranks among three ICL-based models. 
Overall, \ourL~consistently outperforms baselines in both qualitative and quantitative experiments.

\begin{figure}[H]
    \centering
    \includegraphics[width=1.0\linewidth]{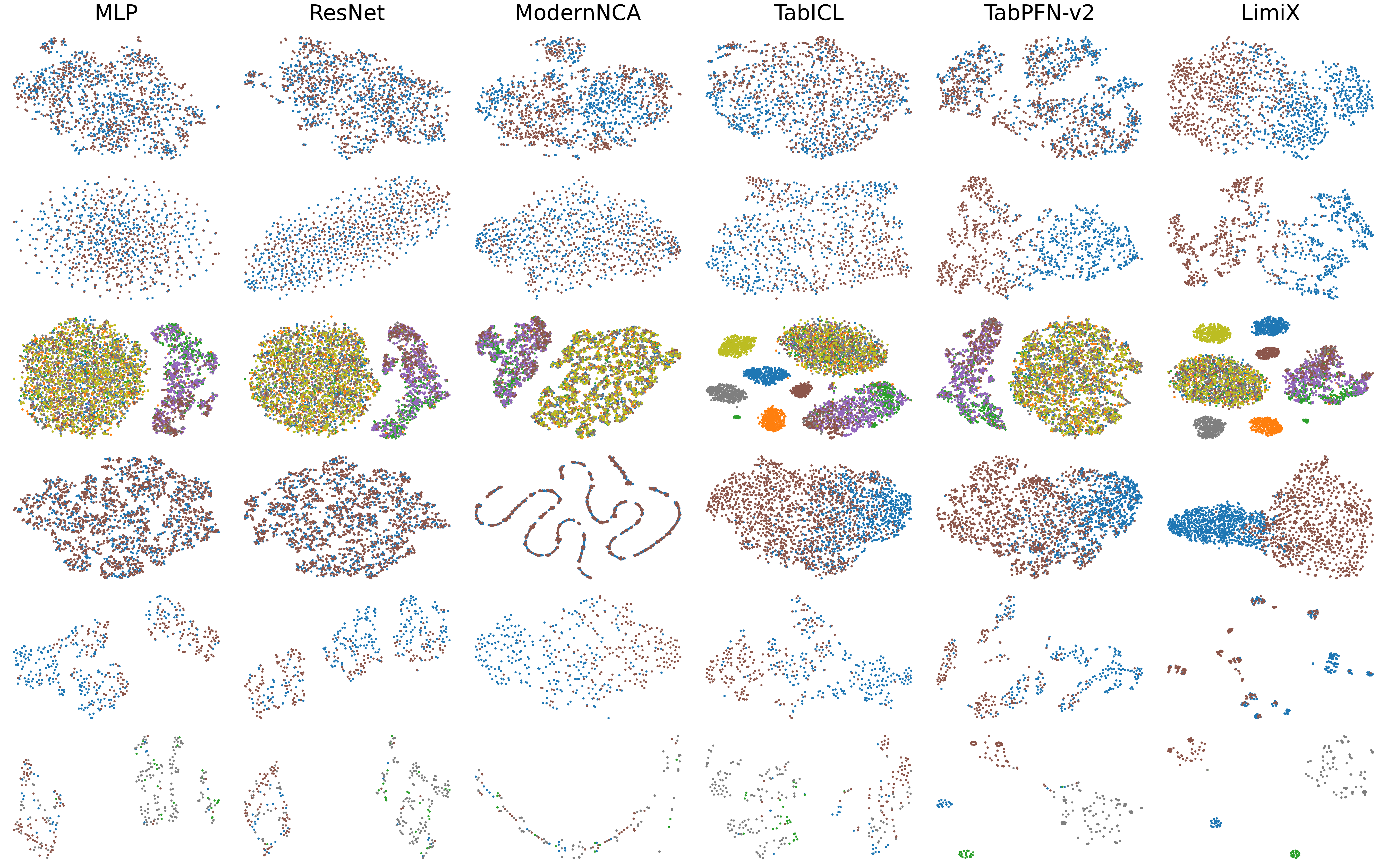}
    \caption{t-SNE visualization of embeddings. Different colors represent different categories. Embeddings extracted by \ourL~are better separated between categories than other models.}
    \label{fig:embed_tsne}
\end{figure}

\begin{table}[H]
\centering
\caption{Results of linear probing on BCCO-CLS. \ourL~outperforms other two ICL-based models, indicating the strongest embedding capability.}
\begin{tabular}{@{}c|cc@{}}
\toprule
Model     & AUC(↑) & Rank(↓) \\ \midrule
TabPFN-v2 & 0.832  & 2.189   \\
TabICL    & 0.838  & 1.981   \\
\ourL     & \textbf{0.850}  & \textbf{1.792}   \\ \bottomrule
\end{tabular}
\label{tab:linear prob}
\end{table}

\clearpage
\subsection{Fine-tuning}
\label{sec:ft}

Long contexts increase memory cost and can degrade optimization during fine-tuning. Therefore, it is common practice to subsample the training corpus to a manageable budget, usually via random or KNN selection~\citep{thomas2024retrieval,xu2024mixture}. Based on our attention-based retrieval strategy, we adopt a retrieval-guided downsampling scheme that concentrates on locally most relevant patterns within each dataset, thus improving sample efficiency and prediction performance.

For each dataset, we split the original training set into a retrieval pool and a query set. For each instance in the query set, we select samples from the retrieval pool via the strategy proposed in \Cref{sec:retrieval} as in-context samples and construct an episode of in-context learning. 
Fine-tuning then proceeds over these episodes rather than over full, unfiltered contexts.
This strategy substantially reduces the number of epochs required to achieve good performance. 
Note that the retrieved in-context sets may contain duplicated samples, which leads to a trade-off between relevance and diversity. 
We empirically find that the retrieved context size controls the balance between computation efficiency and prediction performance.

We compare \ourL~with some representative baselines that are trained or fine-tuned on real datasets, including TabDPT~\citep{ma2024tabdpt}. 
For the fine-tuning of TabPFN-v2, we adopt the strategy and hyperparameter configurations from its official repository\footnote{\url{https://github.com/PriorLabs/TabPFN}}.
As shown in \Cref{tab:classification tasks of sft,tab:regression tasks of sft}, \ourL~significantly outperforms the others across various metrics in most benchmarks after fine-tuning. Meanwhile, we illustrate the change of AUC for \ourL~on some datasets of BCCO-CLS before and after fine-tuning. From \Cref{fig:sft}, a significant improvement can be observed in most cases. 

\begin{figure}[H]
    \centering
    \includegraphics[width=0.95\linewidth]{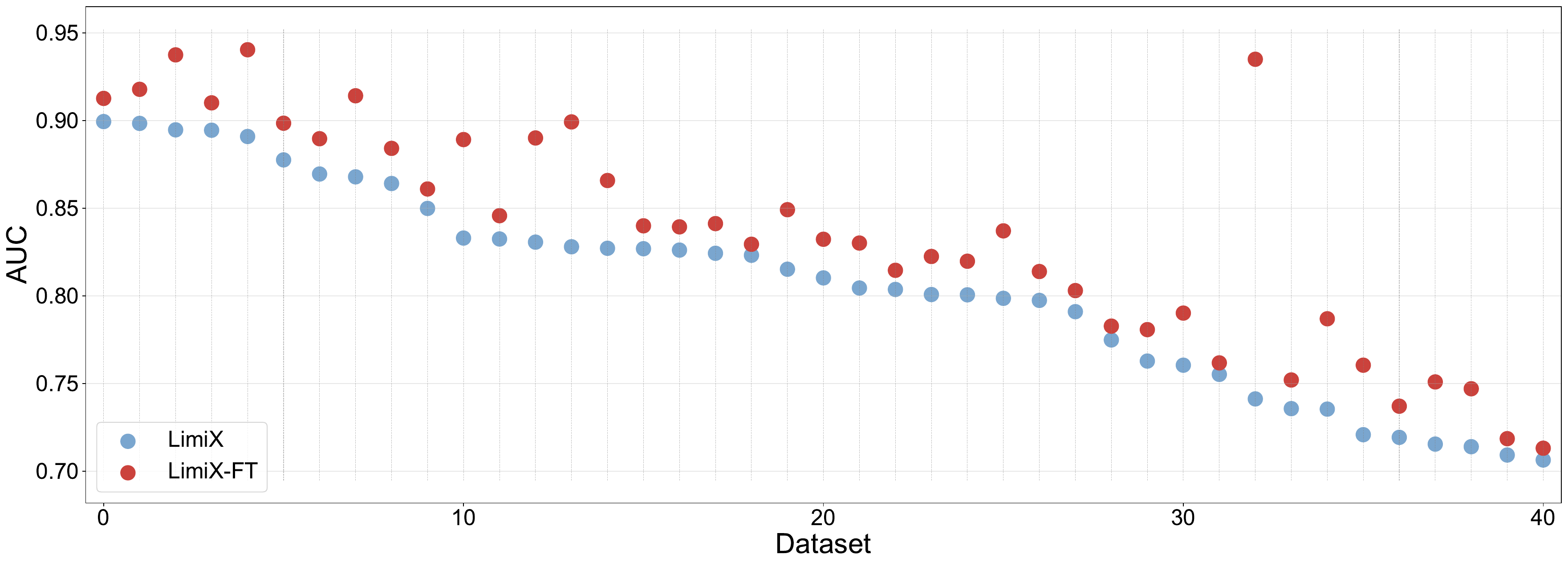}
    \caption{The improvement of fine-tuning in AUC on the BCCO-CLS benchmark.}
    \label{fig:sft}
\end{figure}

\clearpage

\begin{table}[H]
\centering
        \caption{Performance comparison of representative methods on benchmarks of classification. 
        "FT" indicates the model performance after fine-tuning. 
        TabDPT-NE represents the performance of TabDPT without ensemble while its default version has an ensemble size of 8.  The best scores are shown in bold.}
\label{tab:classification tasks of sft}
\resizebox{\linewidth}{!}{%
\begin{tabular}{l|c|ccc|ccc}
\toprule
Benchmark & Method & AUC (↑) & Acc. (↑) & F1 (↑) & \makecell{AUC-Rank} (↓) & \makecell{Acc.-Rank} (↓) & \makecell{F1-Rank} (↓) \\
\midrule
\multirow{8}{*}{BCCO-CLS} 
&XGBoost & 0.834 & 0.762 & 0.674 & 5.811 & 5.660 & 5.509 \\
&CatBoost & 0.829 & 0.757 & 0.664 & 6.264 & 6.000 & 5.802 \\ \cmidrule(l){2-8} 
&TabDPT-NE & 0.841 & 0.774 & 0.685 & 5.085 & 4.679 & 4.783 \\
&TabDPT & 0.846 & 0.777 & 0.687 & 4.453 & 4.358 & 4.604 \\ \cmidrule(l){2-8} 
&TabPFN-v2 & 0.843 & 0.772 & 0.679 & 5.170 & 5.208 & 5.274 \\
&TabPFN-v2-FT & 0.842 & 0.773 & 0.678 & 5.198 & 4.925 & 5.330 \\ \cmidrule(l){2-8} 
&\ourL & 0.871 & 0.804 & 0.731 & 2.538 & 2.575 & 2.651 \\
&\ourL-FT& \textbf{0.873} & \textbf{0.806} & \textbf{0.733} & \textbf{1.396} & \textbf{1.434} & \textbf{1.481} \\ 
\midrule
\multirow{8}{*}{TALENT-CLS} 
&XGBoost & 0.881 & 0.837 & 0.713 & 5.603 & 5.564 & 5.402 \\
&CatBoost & 0.876 & 0.828 & 0.704 & 6.318 & 6.229 & 6.000 \\ \cmidrule(l){2-8} 
&TabDPT-NE& 0.891 & 0.846 & 0.719 & 4.955 & 4.754 & 4.933 \\
&TabDPT  & 0.893 & 0.849 & 0.723 & 4.179 & 4.207 & 4.413\\ \cmidrule(l){2-8} 
&TabPFN-v2& 0.895 & 0.850 & 0.727 & 4.436 & 4.464 & 4.570\\
&TabPFN-v2-FT & 0.889 & 0.842 & 0.718 & 5.212 & 5.156 & 5.106\\ \cmidrule(l){2-8} 
&\ourL & 0.903 & 0.861 & 0.752 & 3.061 & 2.860 & 2.933\\
&\ourL-FT & \textbf{0.904} & \textbf{0.863} & \textbf{0.755} & \textbf{1.726} & \textbf{1.598} & \textbf{1.777}\\
\midrule
\multirow{8}{*}{OpenML-CC18} 
&XGBoost & 0.929 & 0.879 & 0.775 & 5.177 & 5.194 & 5.500 \\
&CatBoost & 0.926 & 0.870 & 0.770 & 6.548 & 6.032 & 6.274 \\ \cmidrule(l){2-8} 
&TabDPT-NE & 0.927 & 0.881 & 0.786 & 4.677 & 4.758 & 4.984 \\
&TabDPT & 0.930 & 0.885 & 0.799 & 3.919 & 3.919 & 4.161 \\ \cmidrule(l){2-8} 
&TabPFN-v2& 0.929 & 0.886 & 0.790 & 5.065 & 4.468 & 4.613 \\
&TabPFN-v2-FT & 0.933 & 0.889 & \textbf{0.883} & 4.161 & 3.823 & 3.032 \\ \cmidrule(l){2-8} 
&\ourL & 0.939 & 0.893 & 0.811 & 2.952 & 3.677 & 3.613 \\
&\ourL-FT & \textbf{0.941} & \textbf{0.894} & 0.813 & \textbf{1.500} & \textbf{2.306} & \textbf{2.435} \\ 
\midrule
\multirow{8}{*}{TabArena} 
&XGBoost & 0.838 & 0.867 & 0.567 & 5.424 & 6.030 & 5.788\\
&CatBoost & 0.835 & 0.867 & 0.574 & 6.242 & 5.545 & 5.121 \\ \cmidrule(l){2-8} 
&TabDPT-NE & 0.891 & 0.846 & 0.719 & 4.950 & 4.749 & 4.927 \\
&TabDPT & 0.839 & 0.870 & 0.580 & 4.970 & 4.636 & 4.667 \\ \cmidrule(l){2-8} 
&TabPFN-v2 & 0.838 & 0.872 & 0.589 & 4.939 & 4.303 & 4.818 \\
&TabPFN-v2-FT & 0.850 & 0.874 & \textbf{0.713} & 3.697 & 4.273 & \textbf{2.682} \\ \cmidrule(l){2-8} 
&\ourL & 0.849 & 0.877 & 0.597 & 3.242 & 3.030 & 4.212 \\
&\ourL-FT& \textbf{0.851} & \textbf{0.878} & 0.600 & \textbf{2.030} & \textbf{1.697} & 2.939 \\ 
\midrule
\multirow{8}{*}{TabZilla} 
&XGBoost & 0.929 & 0.863 & 0.789 & 5.185 & 5.148 & 5.481\\
&CatBoost & 0.922 & 0.848 & 0.780 & 6.556 & 6.037 & 6.074 \\ \cmidrule(l){2-8} 
&TabDPT-NE & 0.932 & 0.880 & 0.824 & 4.333 & 3.963 & 4.037 \\
&TabDPT & 0.934 & 0.882 & 0.824 & 3.815 & 3.815 & 4.148\\ \cmidrule(l){2-8} 
&TabPFN-v2 & 0.929 & 0.863 & 0.797 & 5.111 & 4.815 & 5.185 \\
&TabPFN-v2-FT & 0.932 & 0.876 & \textbf{0.847} & 4.185 & 4.556 & 3.318 \\ \cmidrule(l){2-8}
&\ourL  & 0.943 & 0.885 & 0.836 & 3.556 & 3.444 & 3.889 \\
&\ourL-FT & \textbf{0.944} & \textbf{0.887} & 0.838 & \textbf{2.074} & \textbf{2.407} & \textbf{3.136} \\  
\midrule
\multirow{8}{*}{PFN-CLS} 
&XGBoost & 0.898 & 0.831 & 0.733 & 6.000 & 5.138 & 5.103 \\
&CatBoost & 0.895 & 0.819 & 0.720 & 6.310 & 5.828 & 6.138 \\ \cmidrule(l){2-8} 
&TabDPT-NE& 0.891 & 0.822 & 0.727 & 5.034 & 5.483 & 5.828 \\
&TabDPT & 0.896 & 0.833 & 0.742 & 4.655 & 4.103 & 4.448 \\ \cmidrule(l){2-8} 
&TabPFN-v2 & 0.910 & 0.845 & 0.756 & 4.138 & 4.655 & 4.552\\
&TabPFN-v2-FT & 0.900 & 0.837 & 0.759 & 5.379 & 5.310 & 4.828 \\ \cmidrule(l){2-8} 
&\ourL& 0.923 & 0.862 & 0.786 & 2.241 & 2.690 & 2.759 \\
&\ourL-FT & \textbf{0.924} & \textbf{0.864} & \textbf{0.788} & \textbf{1.276} & \textbf{1.517} & \textbf{1.724} \\
\bottomrule
\end{tabular}}%
\end{table}

\begin{table}[H]
    \centering
        \caption{Performance comparison of representative methods on benchmarks of regression. "FT" indicates the model performance after fine-tuning. 
        TabDPT-NE represents the performance of TabDPT without ensemble while its default version has an ensemble size of 8. 
        The best scores are shown in bold.}
    \label{tab:regression tasks of sft}%
    \begin{tabular}{l|c|cc|cc}
\toprule
Benchmark& Method & R\textsuperscript{2} (↑) & RMSE (↓) &  R\textsuperscript{2}-Rank (↓) & RMSE-Rank (↓) \\ 

\midrule
\multirow{8}{*}{BCCO-REG} 
& XGBoost & 0.764 & 0.415 & 5.820 & 5.740 \\
& CatBoost & 0.741 & 0.427 & 6.600 & 6.460 \\ \cmidrule(l){2-6} 
& TabDPT-NE & 0.769 & 0.410 & 5.220 & 5.060 \\
& TabDPT & 0.772 & 0.406 & 4.460 & 4.260 \\ \cmidrule(l){2-6}
& TabPFN-v2 & 0.772 & 0.404 & 4.600 & 4.280 \\
& TabPFN-v2-FT & 0.777 & 0.399 & 4.020 & 3.720 \\ \cmidrule(l){2-6}
& \ourL & 0.794 & 0.386 & 3.320 & \textbf{2.860} \\
& \ourL-FT & \textbf{0.796} & \textbf{0.385} & \textbf{1.960} & 2.880 \\
 \midrule
 \multirow{8}{*}{TALENT-REG}
& XGBoost & 0.710 & 0.462 & 5.545 & 5.404 \\
& CatBoost & 0.700 & 0.471 & 6.485 & 6.434 \\ \cmidrule(l){2-6}
& TabDPT-NE & 0.709 & 0.461 & 5.061 & 4.929 \\
& TabDPT & 0.711 & 0.458 & 4.263 & 4.081 \\ \cmidrule(l){2-6}
& TabPFN-v2 & 0.695 & 0.465 & 4.980 & 4.747 \\
& TabPFN-v2-FT & 0.702 & 0.459 & 4.879 & 4.657 \\ \cmidrule(l){2-6}
& \ourL & 0.735 & 0.433 & 2.970 & 2.566 \\
& \ourL-FT & \textbf{0.737} & \textbf{0.417} & \textbf{1.818} & \textbf{2.177} \\
 \midrule
 \multirow{8}{*}{CTR23}
& XGBoost & 0.712 & 0.511 & 5.636 & 5.606 \\
& CatBoost & 0.700 & 0.528 & 6.273 & 6.212 \\ \cmidrule(l){2-6}
& TabDPT-NE & 0.728 & 0.500 & 4.909 & 4.818 \\
& TabDPT & 0.731 & 0.498 & 4.273 & 4.121 \\ \cmidrule(l){2-6}
& TabPFN-v2 & 0.716 & 0.503 & 4.545 & 4.394 \\
& TabPFN-v2-FT & 0.722 & 0.498 & 4.636 & 4.545 \\ \cmidrule(l){2-6}
& \ourL & 0.745 & 0.477 & 3.576 & 2.818 \\
& \ourL-FT & \textbf{0.748} & \textbf{0.473} & \textbf{2.152} & \textbf{2.437} \\
 \midrule
 \multirow{8}{*}{PFN-REG} 
& XGBoost & 0.671 & 0.487 & 4.929 & 4.714 \\
& CatBoost & 0.661 & 0.498 & 6.214 & 6.071 \\ \cmidrule(l){2-6}
& TabDPT-NE & 0.672 & 0.491 & 5.393 & 5.071 \\
& TabDPT & 0.675 & 0.484 & 4.607 & 4.500 \\ \cmidrule(l){2-6}
& TabPFN-v2 & 0.676 & 0.475 & 4.393 & 4.143 \\
& TabPFN-v2-FT & 0.687 & 0.466 & 4.179 & 3.857 \\ \cmidrule(l){2-6}
& \ourL & 0.692 & \textbf{0.461} & 3.857 & \textbf{3.571} \\
& \ourL-FT & \textbf{0.695} & 0.466 & \textbf{2.429} & 4.071 \\

\bottomrule
\end{tabular}%
\end{table}

\clearpage
\subsection{Data Generation}

Data generation is a challenging and meaningful task across multiple modalities since it needs to capture the joint distribution of $P(X,Y)$ compared with prediction tasks that focus on modeling $P(Y|X)$. 
While there is abundant literature on the generation of images~\citep{croitoru2023diffusion}, videos~\citep{xing2024survey}, and text~\citep{li2024pre}, relatively less attention has been paid to tabular data generation. 
However, data generation is even more critical for tabular data due to its scarcity and possible privacy issues~\citep{hernandez2022synthetic}. 
As a foundation model, \ourA~have the ability of tabular data generation given an unseen real dataset. 
A significant advantage is that it does not require additional training of generative models. 
Following TabPFN-v2~\citep{hollmann2025tabpfn_nature}, firstly we conduct data generation in an iterative style. 
For the first column, we sample from the empirical categorical distribution if it is a categorical feature, or we use the original first column with added random noise if it is a continuous feature. 
For other columns, we generate $j^{th}$ column based on the $j-1$ columns of real data (treated as $\rvx^{ct}$) and generated data (treated as $\bfxtest$). 
We iterate this process until the last column is generated.
Then for \ourA, we randomly mask a proportion of cell values of generated data and conduct missing value imputation for multiple times so that we can leverage \ours's capability of modeling the joint distribution. 
We conduct experiments on the following datasets: Early Stage Diabetes 2020~\citep{islam2019likelihood,early_stage_diabetes_risk_prediction_529}, Vertebral Column~\citep{vertebral_column_212}, Seeds~\citep{openml_seeds}, Wine~\citep{openml_wine}, and Grub Damage~\citep{openml_grub}.
For evaluation, we use Trend and Shape to evaluate fidelity, which are proposed by SDMetrics\footnote{\url{https://docs.sdv.dev/sdmetrics}}. 
We also calculate the AUC of XGBoost on a hold-out test dataset using generated data or real data. 
From~\Cref{tab:gen}, we find that \ourL~outperforms TabPFN-v2 in terms of all metrics on the five classification tasks. On the dataset of Grub Damage, the prediction performance using data generated by \ourL~is even higher than using real data. 
The results show that \ourA~are capable of capturing dependencies between features of $X$, which is brought by the pretraining strategy of mask prediction. In contrast, TabPFN-v2 is only capable of modeling $P(Y|X)$.

\begin{table}[H]
\centering
\caption{Evaluation of data generation on five classification tasks. "Trend" and "Shape" are two fidelity metrics. "AUC" measures the classification performance of XGBoost using generated (or real) data. \ourL~consistently outperforms TabPFN-v2 in all three metrics.  }
\label{tab:gen}%
\begin{tabular}{@{}c|c|ccccc@{}}
\toprule
Metric                 & Method    & Early Diabetes & Vertebra & Seeds & Wine  & Grub Damage \\ \midrule
\multirow{2}{*}{Trend (↑)} & TabPFN-v2 & 0.797          & 0.580   & 0.696 & 0.686 & 0.486       \\
                       & \ourL     & 0.804          & 0.591   & 0.699 & 0.688 & 0.673       \\ \midrule
\multirow{2}{*}{Shape (↑)} & TabPFN-v2 & 0.889          & 0.754   & 0.739 & 0.622 & 0.635       \\
                       & \ourL     & 0.902          & 0.763   & 0.768 & 0.646 & 0.762       \\ \midrule
\multirow{3}{*}{AUC (↑)}   & TabPFN-v2 & 0.839          & 0.652   & 0.861 & 0.670 & 0.695       \\
                       & \ourL     & 0.879          & 0.783   & 0.932 & 0.912 & 0.727       \\ \cmidrule(l){2-7} 
                       & Real      & 0.969          & 0.915   & 0.982 & 0.996 & 0.710       \\ \bottomrule
\end{tabular}%

\end{table}

\clearpage
\subsection{Out-of-Distribution Generalization}
\label{sec:ood}

In real-world applications, tabular data is often subject to distribution shifts between training data and the test data encountered during deployment. 
For instance, a machine learning model may be trained on patient records collected from one hospital but deployed on data from a different hospital. Such distribution shifts arising from domain variation can lead to substantial degradation in model performance. This challenge is commonly referred to as the Out-of-Distribution (OOD) generalization problem, and has been the focus of extensive research in the machine learning community~\citep{liu2021towards,yu2024survey}. 

We evaluate the OOD generalization performance of various tabular models on 10 public datasets drawn from the TableShift~\citep{gardner2023benchmarking}, which is a benchmark for distribution shifts in tabular data. It contains 15 binary classification tasks in total, covering a wide range of data sources, including finance, medical diagnosis, policy, etc. More details can be found in \Cref{appendix:tableshift}. 
To ensure fair comparisons, for each dataset, if the number of training or test samples exceeds 10,000, we randomly subsample 10,000 instances. Each experiment is repeated five times and we report the average.

From~\Cref{table:tableshift_results}, we can see that \ours~achieves state-of-the-art performance, securing top ranks in both ID (In-Distribution) and OOD evaluations. 
This could be attributed to \ours's integration of causal data with a Context-Conditional Masked Modeling framework. This strategy enables the model to capture robust causal relationships rather than superficial correlations, leading to significantly enhanced generalization on OOD data. 
As for baselines, models that are also based on in-context learning (ICL), notably TabICL and TabPFN-v2, also deliver competitive results. This indicates that the ICL mechanism could better capture latent invariant patterns in data, thereby endowing models with strong generalization potential.

\begin{table}[H]
\centering
\caption{Average AUC and ranks on TableShift. We can see that \ours~consistently outperforms all baselines in terms of ID or OOD performance.}
\resizebox{0.8\textwidth}{!}{%
\begin{tabular}{lcccc}
\toprule
Model & ID\_AUC ($\uparrow$) & OOD\_AUC ($\uparrow$) & ID\_Rank ($\downarrow$) & OOD\_Rank ($\downarrow$) \\
\midrule
\ourL & \textbf{0.848} & \textbf{0.806} & \textbf{2.5} & \textbf{1.3} \\ 
TabICL & 0.847 & 0.799 & 4.1 & 3.9 \\
AutoGluon & 0.842 & 0.797 & 3.5 & 4.0 \\
TabPFN-v2 & 0.841 & 0.797 & 6.4 & 5.2 \\
CatBoost & 0.840 & 0.793 & 4.1 & 5.9 \\
LightGBM & 0.836 & 0.790 & 6.1 & 7.3 \\
MLP & 0.839 & 0.792 & 7.5 & 7.6 \\
FT-Transformer & 0.840 & 0.789 & 7.8 & 8.3 \\
ResNet & 0.837 & 0.789 & 8.6 & 8.3 \\
NODE & 0.836 & 0.789 & 9.2 & 9.2 \\
XGBoost & 0.830 & 0.783 & 9.9 & 9.5 \\
TabDPT & 0.822 & 0.763 & 12.6 & 12.7 \\
TabR & 0.820 & 0.767 & 13.2 & 13.0 \\
TANGOS & 0.801 & 0.752 & 13.8 & 13.6 \\
ModernNCA & 0.801 & 0.757 & 15.4 & 14.6 \\
\bottomrule
\label{table:tableshift_results}
\end{tabular}%
}
\end{table}

\section{Scaling law}

Recent advances in large language models (LLMs) reveal a class of empirical regularities known as \textit{scaling laws}, which describe predictable relationships between model performance and key resource dimensions such as parameter count, dataset size, and compute budget \citep{kaplan2020scaling,hoffmann2022training}. 
These scaling laws enable researchers to forecast model behavior and guide efficient resource allocation when scaling up architectures. 

However, most prior studies focus on large language models (LLMs), analyzing how loss scales with model parameters and compute under fixed or approximately fixed data distributions. Despite the growing importance of large structured-data models (LDMs), their scaling behavior remains largely unexplored. 

Understanding the quantitative interaction between loss, data, and model parameters informs optimal trade-offs between data acquisition and model expansion. Motivated by this, this work systematically investigates the scaling laws for LDMs, focusing on the empirical relationships among model size, dataset size, and various performance metrics, including Log-Loss, AUC, F1, and accuracy for classification, as well as RMSE-Loss and $R^2$ for regression tasks. Our primary focus is on evaluating these metrics on downstream tasks, as they directly reflect the model's practical effectiveness and its ability to generalize to real-world applications. By analyzing performance on tasks like classification and regression, we aim to identify the optimal balance between model size, data, and task-specific performance.

For the experimental setup, we construct a series of models with parameter counts ranging from 1.05M to 16.53M (and up to 1.6G in extrapolated fits), trained on datasets spanning $1 \times 10^{11}$ to $1.7 \times 10^{12}$ tokens. Each configuration maintains consistent architecture and optimization settings, isolating the effects of data and parameter scaling. The downstream benchmark we used is BCCO-CLS and BCCO-REG. For this experiment, the evaluation was performed without the retrieval procedure, which may lead to minor deviations in performance compared to the results reported in the experimental section.

\subsection{Scaling with Dataset Size and Model Parameters}

Following previous findings in large language models (LLMs), we assume that both the loss and performance metrics of our model follow a power-law relationship with respect to computational resources, including dataset size and model parameter count. This assumption reflects the empirical regularity that as data volume or model capacity increases, the model’s loss tends to decrease while its predictive performance improves in a predictable manner. Formally, we hypothesize that each evaluation metric $M$ (e.g., Log-Loss, RMSE-Loss, AUC, ACC, or $R^2$) can be expressed as a power-law function of the resource variable $N$:

\begin{equation}
M = a \cdot N^{c} + b,
\end{equation}

where $a$, $b$, and $c$ are constants to be determined. To estimate these parameters, we employ an ordinary least squares (OLS) regression over measurements collected across multiple scales of $N$. For each metric and task configuration, the fitted parameters and the coefficient of determination ($R^2$) are reported to assess the strength and consistency of the scaling behavior. This procedure enables systematic quantification of how both loss reduction and performance improvement scale with resource growth, facilitating direct comparisons between different model sizes and data regimes.

Figure~\ref{fig:loss_vs_tokens} illustrates the empirical loss scaling with dataset size across different model configurations. All models achieve highly consistent fits with $R^2 > 0.99$, confirming the robustness of the scaling law. As dataset size increases, the loss decreases rapidly at smaller scales and then transitions to a more gradual improvement trend. Larger models achieve consistently lower loss values and exhibit smoother, more stable scaling behavior as additional data are introduced.

\begin{figure}[ht]
    \centering
    
    \begin{subfigure}[b]{0.48\textwidth}
        \centering
        \includegraphics[width=\textwidth]{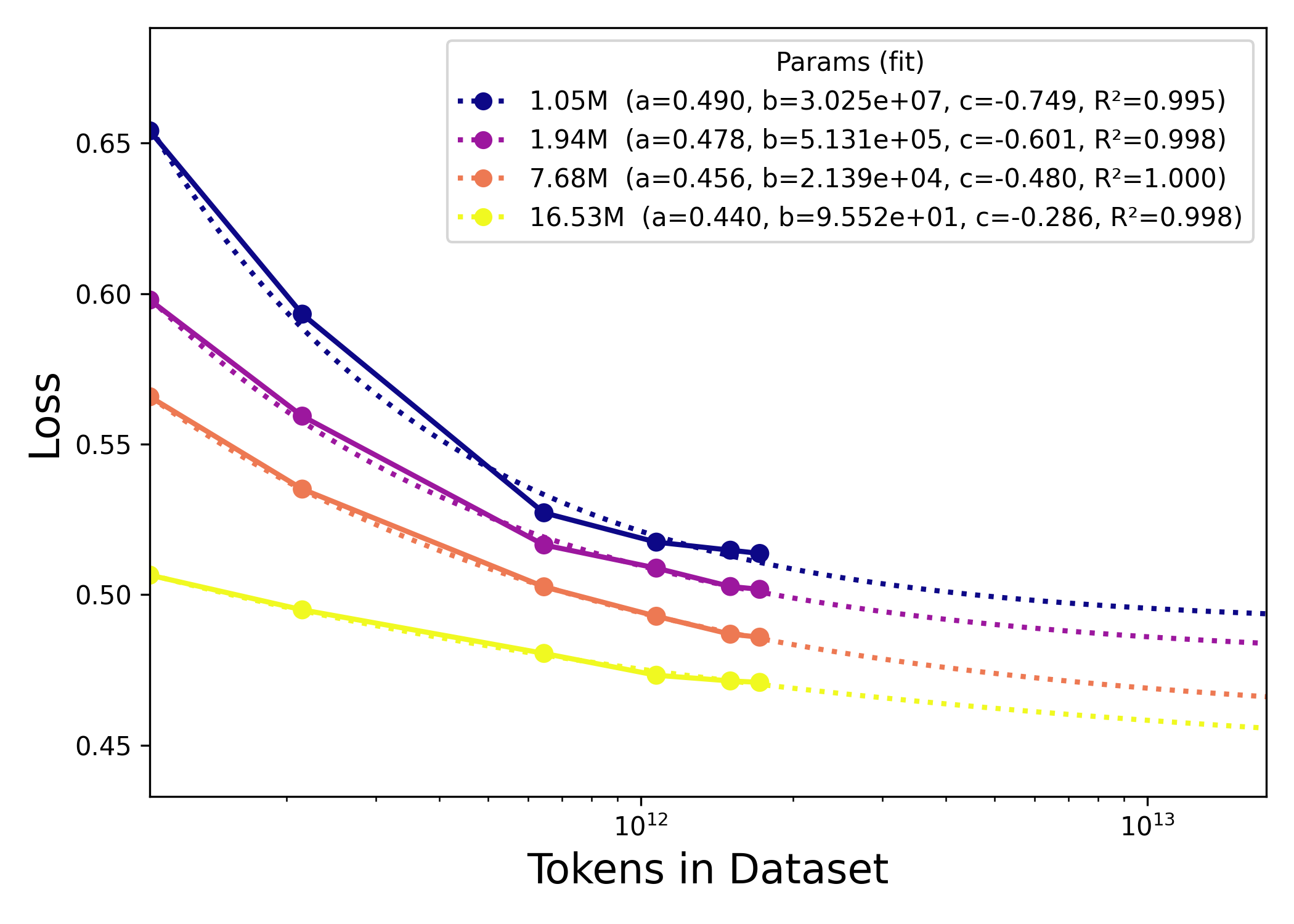}
        \caption{Log-Loss versus dataset size on BCCO-CLS}
        \label{fig:bcco_cls_loss}
    \end{subfigure}
    \hfill
    \begin{subfigure}[b]{0.48\textwidth}
        \centering
        \includegraphics[width=\textwidth]{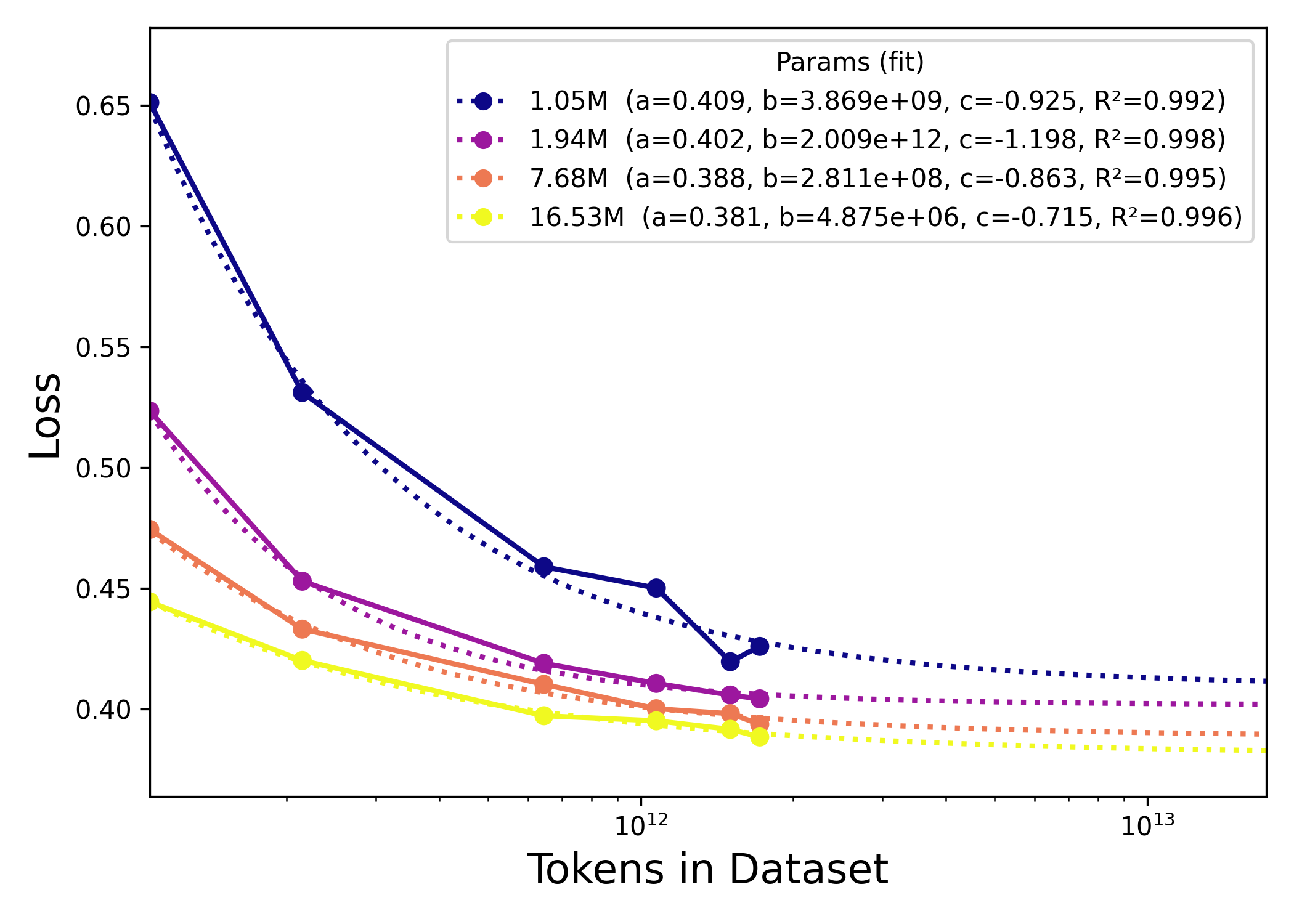}
        \caption{RMSE-Loss versus dataset size on BCCO-REG}
        \label{fig:bcco_reg_loss}
    \end{subfigure}

    \caption{Loss versus dataset size for models of different parameter counts. Dotted lines indicate fitted power-law relations.}
    \label{fig:loss_vs_tokens}
\end{figure}

Furthermore, we extend our analysis to examine how downstream task loss and performance scale with model parameters when unconstrained by dataset size or computational budget. As shown in Figure~\ref{fig:loss_vs_params}, multiple performance metrics exhibit clear dependencies on the model parameter count $N$. For the classification task, the empirical fits achieve $R^2 > 0.9$ for Log-Loss, AUC, and ACC, confirming the robustness of the scaling law, while the fit for F1 yields a slightly lower $R^2$ of 0.68. For the regression task, both RMSE-Loss and the $R^2$ score exhibit strong correlations with model size, with empirical fits achieving $R^2 > 0.9$.

The empirical fits are reported below, with only those achieving $R^2 > 0.9$ shown:

\begin{equation}
\begin{aligned}
    \text{Log-Loss} &= -132.43 + 133.12N^{-0.000}, & (R^2=0.9646)\\
    \text{AUC} &= 1.30 - 0.56N^{-0.014}, & (R^2=0.9044)\\
    \text{ACC} &= 0.91 - 0.31N^{-0.058}, & (R^2=0.9933)\\
    \text{RMSE-loss} &=-30.63+31.13N^{-0.000}, & (R^2=0.9890) \\
    \text{$R^2$} &=26.33-25.65N^{-0.000}, & (R^2=0.9890) \\
\end{aligned}
\end{equation}
 
Loss decays sharply with increasing $N$, while AUC, ACC and F1 improve gradually, indicating that while error convergence is achieved relatively early, fine-grained discriminative improvements require substantially larger models.

\begin{figure}[ht]
    \centering
    
    \begin{subfigure}[b]{0.48\textwidth}
        \centering
        \includegraphics[width=\textwidth]{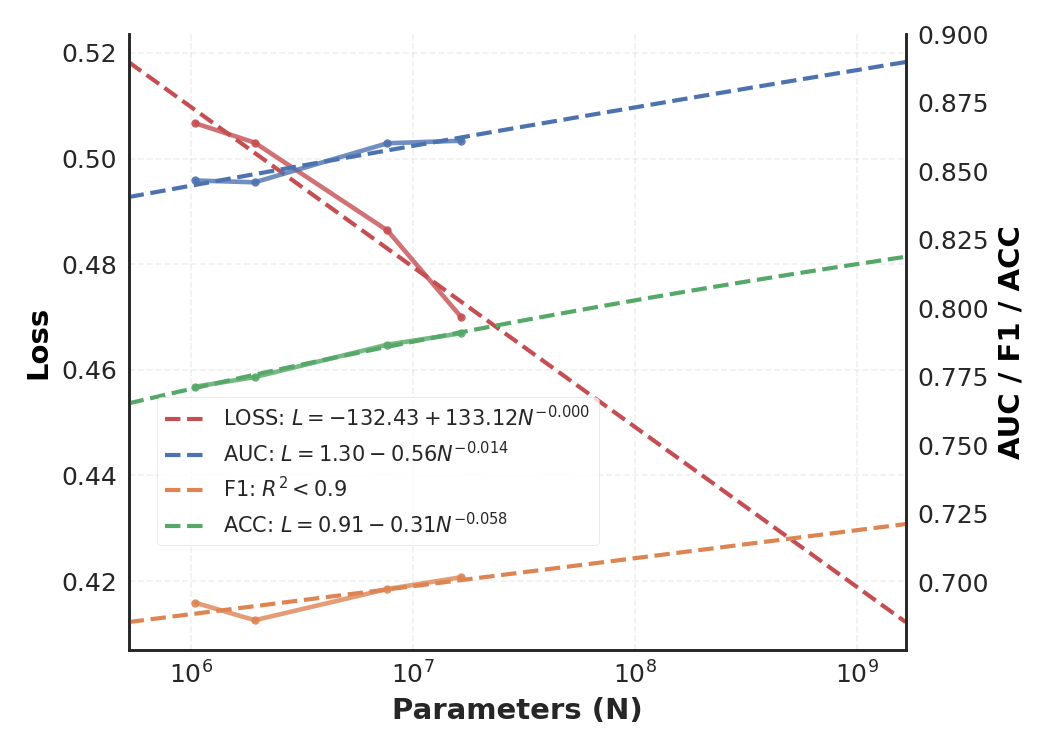}
        \caption{Log-Loss and performance versus parameter count on BCCO-CLS}
        \label{fig:loss_vs_params_cls}
    \end{subfigure}
    \hfill
    \begin{subfigure}[b]{0.48\textwidth}
        \centering
        \includegraphics[width=\textwidth]{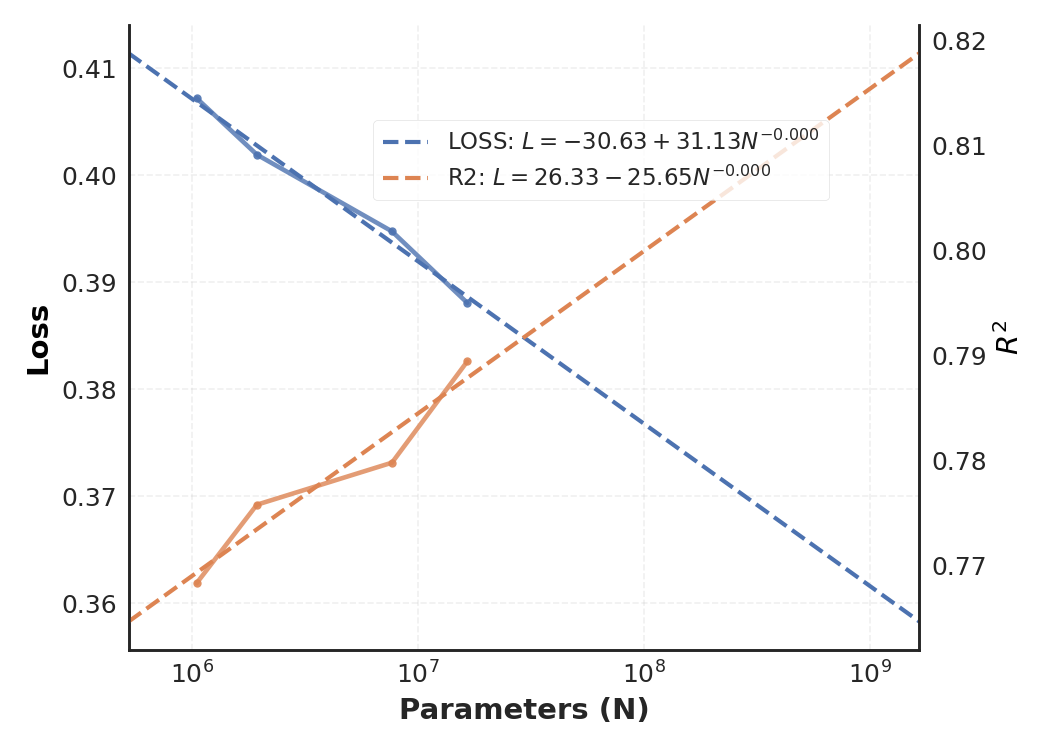}
        \caption{RMSE-Loss and performance versus parameter count on BCCO-REG}
        \label{fig:loss_vs_params_reg}
    \end{subfigure}

    \caption{Joint scaling of Loss, AUC, F1, ACC, and $R^2$ with respect to parameter count. Real and fitted trends are shown as solid and dashed lines, respectively.}
    \label{fig:loss_vs_params}
\end{figure}

\subsection{Conclusions and Insights}

This study provides the first systematic characterization of scaling laws in \textbf{LDMs}. Our key findings are summarized as follows:

\textbf{Loss versus Dataset Size}: both the classification (BCCO-CLS) and regression (BCCO-REG) tasks exhibit a consistent power-law relationship between loss and dataset size. When the token count increases from approximately $10^{11}$ to $10^{13}$, the loss decreases systematically across all model configurations. Smaller models (e.g., 1.05M parameters) display a sharper initial decline in loss, while larger models (e.g., 16.53M parameters) achieve overall lower loss levels and maintain more stable improvement trends as data volume grows. The fitting results demonstrate exceptionally high explanatory power ($R^2 > 0.99$), confirming the robustness of the scaling behavior. These observations suggest that increasing the dataset size consistently enhances model performance, with improvement rates varying according to model capacity.

\textbf{Loss and Performance versus Model Parameters} As the number of model parameters increases from 1.05M to approximately one billion, the loss decreases markedly while overall predictive performance improves. In classification tasks, loss declines rapidly at smaller scales and stabilizes at larger ones, while AUC and ACC values continue to increase consistently. Similarly, in regression tasks, RMSE decreases and the coefficient of determination ($R^2$) rises with model size, both following strong power-law trends ($R^2 > 0.98$). These results suggest that scaling up the model enhances predictive stability and generalization. However, the diminishing rate of improvement in loss relative to performance metrics implies that additional parameters primarily contribute to refined discriminative capability rather than further error convergence.

Overall, the scaling behavior of the LDM follows a predictable power-law trend analogous to that observed in large language models. Both increased dataset size and model capacity yield systematic and highly correlated reductions in loss ($R^2 > 0.99$). Nevertheless, the returns on scaling diminish progressively, as larger models require substantially greater data and computational resources to achieve marginal gains in performance. These results highlight the importance of balancing dataset expansion and model scaling to optimize the trade-off between predictive accuracy and computational efficiency.

\section{Conclusion}

In this work, we present \ours, a unified family of our LDM series, and release its first two variants, \ourL~and \ourS. \ourA~treat structured-data inputs as a joint distribution over variables and missingness, so that classification, regression, missing value imputation, data generation, and sample selection for interpretability, can all be expressed as conditional queries to a single model. Methodologically, \ourA~adopt a lightweight, scalable architecture that models causal relations among variables while jointly capturing dependencies across features and samples. 
Meanwhile, \ourA~combine masked joint-distribution pretraining with an episodic, context-conditional objective of per-dataset adaptation for the versatility in downstream tasks. The pretraining data for \ourA~is generated with hierarchical SCMs via graph-aware and solvability-aware sampling. Attention-guided retrieval of \ourA~further supports efficient inference-time ensemble and fine-tuning if desired. Empirically, experiments across 11 large structured-data benchmarks, spanning wide ranges of sample sizes, feature dimensions, number of classes, categorical-to-numerical feature ratios, missingness, and sample-to-feature ratios, confirm the effectiveness of \ourA~. As a single model, \ourL~consistently surpasses competitive baselines on various downstream tasks. 
We present the first systematic investigation of scaling laws in LDMs, revealing a predictable power-law relationship between scale and performance, analogous to that observed in large language models.

\clearpage
\section{Contribution}

\textbf{Project Design and Lead}

Xingxuan Zhang, Peng Cui

\vspace{15pt}

\textbf{Core Contributors}

Gang Ren, Han Yu, Hao Yuan, Hui Wang, Jiansheng Li, Jiayun Wu, Lang Mo, Li Mao, Mingchao Hao, Ningbo Dai, Renzhe Xu, Shuyang Li, Tianyang Zhang, Yue He, Yuanrui Wang, Yunjia Zhang, Zijing Xu

\vspace{15pt}

\textbf{Contributors}

Dongzhe Li, Fang Gao, Hao Zou, Jiandong Liu, Jiashuo Liu, Jiawei Xu, Kaijie Cheng, Kehan Li, Linjun Zhou, Qing Li, Shaohua Fan, Xiaoyu Lin, Xinyan Han, Xuanyue Li, Yan Lu, Yuan Xue, Yuanyuan Jiang, Zimu Wang, Zhenlei Wang

\clearpage

\bibliography{biblio}
\bibliographystyle{colm2024_conference}

\clearpage
\appendix
\section{Experimental Details}
\label{appendix:exp-details}

\subsection{Details of Datasets with Distribution Shifts}
\label{appendix:tableshift}

The 10 public datasets we adopt from TableShift include Voting\footnote{\url{https://electionstudies.org/}},   Unemployment\footnote{\url{https://www.census.gov/programs-surveys/acs}\label{census}}, Sepsis~\citep{reyna2020early}, Public Health Insurance\footref{census}, Income\footref{census}, Hospital Readmission~\citep{strack2014impact}\footnote{\url{https://archive.ics.uci.edu/ml/datasets/Diabetes+130-US+hospitals+for+years+1999-
2008}}, Food Stamps\footref{census}, Diabetes\footnote{\url{https://www.cdc.gov/brfss/index.html}}, College Scorecard\footnote{\url{https://collegescorecard.ed.gov/}}, and ASSISTments~\citep{metz2020assistments}\footnote{\url{https://new.assistments.org}}. 
Detailed information of the benchmark TableShift we adopt in \Cref{sec:ood} can be found in~\Cref{tab-ood:dataset_stats}. 

\begin{table}[htbp]
\centering
\caption{Statistics of datasets in TableShift. "\#Features" indicates the number of features. "\#Train" indicates the number of training samples. "\#ID\_test" indicates the number of ID (In-Distribution) test samples. "\#OOD\_test" indicates the number of OOD (Out-of-Distribution) test samples. "Shift Type" indicates how the datasets are split into multiple domains to create distribution shifts. "\#Domains" indicates the numbers of training and test domains.}
\label{tab-ood:dataset_stats}
\footnotesize
\setlength{\tabcolsep}{4pt}
\begin{tabular}{lcccccc}
\toprule
\textbf{Dataset} & \textbf{\#Features} & \textbf{\#Train} & \textbf{\#ID\_test} & \textbf{\#OOD\_test} & \textbf{Shift type} & \textbf{\#Domains} \\
\midrule
Voting & 365 & 34,796 & 4,350 & 21,231 & Geographic Region & 4/1 \\
ASSISTments & 26 & 2,132,526 & 266,566 & 1,906 & School & 386/10 \\
Diabetes & 142 & 969,229 & 121,154 & 209,375 & Race & 1/5 \\
Food Stamps & 239 & 629,018 & 78,628 & 48,878 & Geographic Region & 9/1 \\
Hospital Readmission & 183 & 34,288 & 4,287 & 50,968 & Admission Source & 15/1 \\
Income & 232 & 1,264,123 & 158,016 & 75,911 & Geographic Region & 9/1 \\
Health Insurance & 135 & 4,006,249 & 500,782 & 817,877 & Disability Status & 1/1 \\
Sepsis & 41 & 1,122,299 & 140,288 & 134,402 & Length of Stay & 47/2 \\
Unemployment & 223 & 1,290,914 & 161,365 & 163,611 & Education Level & 9/15 \\
College ScoreCard & 118 & 98,556 & 12,320 & 1,352 & Institution Type & 26/8 \\
\bottomrule
\end{tabular}
\end{table}

\subsection{Hyperparameter Search Space for Baselines}
\label{appendix:searchspace}

The hyperparameter search space for the baseline models is summarized in Table \ref{tab:hyperparameter_space}.

\clearpage

\begin{table}[H]
\centering
\caption{The hyperparameter search space for tree-based models is utilized by Optuna, which provides automatic suggestions via Bayesian optimization. The only difference between classification and regression configurations lies in the optimization target. All other hyperparameter settings remain identical.}
\label{tab:hyperparameter_space}
\begin{tabular}{@{}ccccc@{}}
\toprule
Baseline & Hyperparameter & Data Type & Log & Search Space \\ \midrule
RF & n\_estimators & int & no & {[}100, 500{]} \\
 & max\_depth & int & no & {[}3, 20{]} \\
 & min\_samples\_split & int & no & {[}2, 20{]} \\
 & min\_samples\_leaf & int & no & {[}1, 20{]} \\
 & max\_features & categorical & no & \{Sqrt, Log\textsubscript{2}, None\} \\
 & bootstrap & categorical & no & \{True, False\} \\ \midrule
ET & n\_estimators & int & no & {[}100, 500{]} \\
 & max\_depth & int & no & {[}3, 20{]} \\
 & min\_samples\_split & int & no & {[}2, 20{]} \\
 & min\_samples\_leaf & int & no & {[}1, 20{]} \\ \midrule
XGBoost & n\_estimators & int & no & {[}100, 300{]} \\
 & max\_depth & int & no & {[}3, 9{]} \\
 & learning\_rate & float & yes & {[}0.01, 0.3{]} \\
 & subsample & float & no & {[}0.5, 1.0{]} \\
 & colsample\_bytree & float & no & {[}0.5, 1.0{]} \\ \midrule
LightGBM & n\_estimators & int & no & {[}100, 300{]} \\
 & learning\_rate & float & yes & {[}0.01, 0.3{]} \\
 & max\_depth & int & no & {[}-1, 20{]} \\
 & subsample & float & no & {[}0.5, 1.0{]} \\
 & colsample\_bytree & float & no & {[}0.5, 1.0{]} \\ \midrule
CatBoost & iterations & int & no & {[}100, 300{]} \\
 & depth & int & no & {[}4, 10{]} \\
 & learning\_rate & float & yes & {[}0.01, 0.3{]} \\ \bottomrule
\end{tabular}
\end{table}

\section{Omitted Details in \cref{sect:theory}}
\subsection{Omitted Details in \Cref{sect:joint-distribution}} \label{sect:proof-identifiability}
\begin{proposition}[Formal Version of \Cref{prop:identifiability}] \label{prop:identifiability-full}
    Suppose that the distribution $p(\bfXtest|\bfXcontext)$ has a strictly positive probability density function.
    Then there is a one-to-one correspondence between the distribution $p(\bfXtest|\bfXcontext)$ and the family of conditionals $\{p(\bfXtest_{\pi}|\bfXtest_{-\pi}, \bfXcontext): \forall \pi \in \Pi_k\}$.
\end{proposition}
\begin{proof}
    It is clear that $p(\bfXtest | \bfXcontext)$ directly yields all conditionals $\{p(\bfXtest_{\pi} | \bfXtest_{-\pi}, \bfXcontext): \pi \in \Pi_k\}$. We now show the converse: $p(\bfXtest | \bfXcontext)$ can be recovered from this family of conditionals.
    
    As the first step, we derive $\{p(\Xtest_j | \bfXtest_{-j}, \bfXcontext): j \in [d]\}$. For any $j \in [d]$, select a mask $\pi \in \Pi_k$ with $j \in \pi$. Then
    \[
    p(\Xtest_j | \bfXtest_{-j}, \bfXcontext) 
    = \frac{p(\bfXtest_\pi | \bfXtest_{-\pi}, \bfXcontext)}{p(\bfXtest_{\pi\setminus \{j\}} | \bfXtest_{-\pi}, \bfXcontext)} 
    = \frac{p(\bfXtest_\pi | \bfXtest_{-\pi}, \bfXcontext)}{\int_{\Omega}p\left(\bfXtest_{\pi\setminus \{j\}}, {\Xtest_{j}}' | \bfXtest_{-\pi}, \bfXcontext\right) \mathrm{d} {\Xtest_{j}}'}.
    \]
    Each term on the right-hand side belongs to $\{p(\bfXtest_{\pi} | \bfXtest_{-\pi}, \bfXcontext): \pi \in \Pi_k\}$, hence $p(\Xtest_j | \bfXtest_{-j}, \bfXcontext)$ can indeed be recovered.

    We now use induction to show that the knowledge of $\{p(\Xtest_j | \bfXtest_{-j}, \bfXcontext): j \in [d]\}$ suffices to recover all conditionals in
    \[
    \gP_k = \{p(\bfXtest_{\pi} | \bfXtest_{-\pi}, \bfXcontext): |\pi| = k\}.
    \]
    Clearly, $\gP_1$ coincides with $\{p(\Xtest_j | \bfXtest_{-j}, \bfXcontext): j \in [d]\}$. Suppose $\gP_{k'}$ is obtainable for all $k' \le k$. Consider $p(\bfXtest_{\pi} | \bfXtest_{-\pi}, \bfXcontext) \in \gP_{k+1}$ with $|\pi| = k+1$. Pick $j \in \pi$ and set $\pi' = \pi \setminus \{j\}$. Then
    \[
    \int_{\Omega} \frac{p(\Xtest_j | \bfXtest_{\pi'}, \bfXtest_{-\pi}, \bfXcontext)}{p(\bfXtest_{\pi'} | \Xtest_j, \bfXtest_{-\pi}, \bfXcontext)} \, \mathrm{d}\Xtest_j
    = \int_{\Omega} \frac{p(\Xtest_j | \bfXtest_{-\pi}, \bfXcontext)}{p(\bfXtest_{\pi'} | \bfXtest_{-\pi}, \bfXcontext)} \, \mathrm{d}\Xtest_j
    = \frac{1}{p(\bfXtest_{\pi'} | \bfXtest_{-\pi}, \bfXcontext)}.
    \]
    Each term on the left-hand side belongs to $\gP_1$ or $\gP_k$, so $p(\bfXtest_{\pi'} | \bfXtest_{-\pi}, \bfXcontext)$ is obtainable. Finally,
    \[
    p(\bfXtest_{\pi} | \bfXtest_{-\pi}, \bfXcontext) 
    = p(\bfXtest_{\pi'} | \bfXtest_{-\pi}, \bfXcontext) \cdot p(\Xtest_j | \bfXtest_{\pi'}, \bfXtest_{-\pi}, \bfXcontext),
    \]
    showing that $p(\bfXtest_{\pi} | \bfXtest_{-\pi}, \bfXcontext)$ can be obtained as well. Hence, any element in $\gP_{k+1}$ can be obtained. By induction, the claim follows.
\end{proof}

Now we show that the knowledge of a single conditional distribution $p(\Xtest_j | \bfXtest_{-j}, \bfXcontext)$ for some $j \in [d]$ (i.e., when $\Pi = \{j\}$) is insufficient to recover the full conditional distribution $p(\bfXtest | \bfXcontext)$.

\begin{example} \label{example:identifiability}
    Consider the case $d=2$ and $m=0$, i.e., a setting with no in-context samples and a feature dimension of $2$. In this case, the target distribution $p(\bfXtest | \bfXcontext)$ reduces to $p(\bfXtest)$. Suppose further that $\bfXtest \sim \gN(\bfmu, \Sigma)$, where $\bfmu = (\mu_1, \mu_2) \in \sR^2$ and $\Sigma =\begin{pmatrix}\Sigma_{11} & \Sigma_{12}\\\Sigma_{21} & \Sigma_{22} \end{pmatrix} \succeq 0$.
    Then the conditional distribution of $\Xtest_1$ given $\Xtest_2$ is
    \[
        \Xtest_1 \mid \Xtest_2 \sim \gN\left(\Sigma_{12}\Sigma_{22}^{-1}\Xtest_2 + \mu_1 - \Sigma_{12}\Sigma_{22}^{-1}\mu_2, \Sigma_{11} - \Sigma_{12}\Sigma_{22}^{-1}\Sigma_{21}\right).
    \]
    Consequently, knowledge of $p(\Xtest_1 | \Xtest_2)$ alone provides access only to the quantities $\Sigma_{12}\Sigma_{22}^{-1}$, $\mu_1 - \Sigma_{12}\Sigma_{22}^{-1}\mu_2$, and $\Sigma_{11} - \Sigma_{12}\Sigma_{22}^{-1}\Sigma_{21}$. These are insufficient to uniquely determine the full parameter set $(\bfmu, \Sigma)$, and therefore $p(\bfXtest)$ cannot be fully recovered from $p(\Xtest_1 |\Xtest_2)$ alone. By symmetry, $p(\bfXtest)$ also cannot be recovered solely from $p(\Xtest_2|\Xtest_1)$.
\end{example}

\subsection{Omitted Details in \Cref{sect:mask-number}}
We denote by $q_\theta(\bfXtest|\bfXcontext), \theta \in \Theta$ the distribution induced by the learned family of conditional probabilities $\{q_{\theta}(\bfXtest_{\pi} | \bfXtest_{-\pi}, \bfXcontext) : \pi \in \Pi_k\}$. We overload the notation by writing $q_\theta(\bfXcontext, \bfXtest) \coloneqq q_{\theta}(\bfXtest|\bfXcontext)p(\bfXcontext)$.

\subsubsection{Sample Efficiency}
\begin{theorem}[Formal Version of \Cref{thrm:sample-efficiency}]  \label{thrm:sample-efficiency-full}
    Suppose there exists $\theta^* \in \Theta$ such that $q_{\theta^*}(\bfXtest_{\pi} | \bfXtest_{-\pi}, \bfXcontext) = p(\bfXtest_{\pi} | \bfXtest_{-\pi}, \bfXcontext)$ for all $\bfXcontext \in \Omega^{m \times d}$, $\bfXtest \in \Omega^d$, and $\pi \subseteq [d]$, and that the minimizer of $L_k(\theta)$ is unique for every $k$. Assume that for all $\theta \in \Theta$, $\bfXcontext \in \Omega^{m \times d}$, $\bfXtest \in \Omega^d$, and $\pi \subseteq [d]$, the gradient norm $\lVert\nabla_{\theta} q_{\theta}(\bfXtest_{\pi} | \bfXtest_{-\pi}, \bfXcontext)\rVert_2$ and the Hessian norm $\lVert\nabla^2_{\theta} q_{\theta}(\bfXtest_{\pi} | \bfXtest_{-\pi}, \bfXcontext)\rVert_F$ exist and are finite. Assume $\nabla^2_{\theta} L_k(\theta^*) \succ 0$ for all $k \in [d]$. Then, for every sufficiently small neighborhood $\gB$ of $\theta^*$, there exists a sufficiently large $n$ such that $\htheta_{k,n}$ is the unique minimizer of $\hat{L}_k(\theta)$ in $\gB$. Moreover,
    \[
    \sqrt{n}(\htheta_{k,n} - \theta^*) \xrightarrow{d} \gN(0, \Gamma_k),
    \]
    where $\Gamma_k$ does not depend on $n$ and satisfies $\Gamma_{k+1} \preceq \Gamma_k$.
\end{theorem}

The proof of \Cref{thrm:sample-efficiency-full} is based on the following lemma.

\begin{lemma}[\citet{van2000asymptotic}, Theorem 5.23; statement adapted from \citet{qin2024fit,li2024promises}] \label{lemma:convergence}
    Consider a loss $L: \Theta \rightarrow \mathbb{R}$, such that $L(\theta) = \E_{\bfX \sim p}[\ell_{\theta}(\bfX)]$ for $\ell_{\theta}: \gX \rightarrow \sR$. Let $\Theta^*$ be the the set of global minima of $L$, i.e.,
    \[
    \Theta^* = \{\theta^*: L(\theta^*) = \min_{\theta \in \Theta} L(\theta)\}.
    \]
    Suppose the following conditions are met:
    \begin{itemize}
        \item (Gradient bounds on $\ell_{\theta}$) The map $\theta \mapsto \ell_{\theta}$ is measurable and differentiable at every $\theta^* \in \Theta^*$ for $p$-almost every $\bfX$. Furthermore, there exists a function $B(\bfX)$, s.t. $\E[B(\bfX)^2] < \infty$ and and for every $\theta_1, \theta_2$ near $\theta^*$, we have
        \[
        \left|\ell_{\theta_1}(\bfX) - \ell_{\theta_2}(\bfX)\right|< B(\bfX) \left\lVert\theta_1 - \theta_2\right\rVert_2
        \]
        \item (Twice-differentiability of $L$) $L(\theta)$ is twice-differentiable at every $\theta^* \in \Theta^*$ with Hessian $\nabla^2_{\theta} L(\theta^*)$, and furthermore $\nabla^2_{\theta} L(\theta^*) \succ 0$.
        \item (Uniform law of large numbers) The loss $L$ satisfies a uniform law of large numbers, that is
        \[
        \sup_{\theta \in \Theta} \left|\hat{\E}[\ell_{\theta}(\bfX)] - L(\theta)\right| \xrightarrow{p} 0.
        \]
        \item (Realizability) The data distribution $p$ satisfies: $\exists \theta^* \in \Theta$ such that $p_{\theta^*} = p$.
    \end{itemize}
    Then for every $\theta^* \in \Theta^*$, and every sufficiently small neighborhood $S$ of $\theta^*$, there exists a sufficiently large $n$, such that there is a unique minimizer $\htheta_n$ of $\hat{\E}[\ell_{\theta}(\bfX)]$ in $S$. Furthermore, $\htheta_n$ satisfies:
    \[
    \sqrt{n}\left(\htheta_n-\theta^*\right) \xrightarrow{d} \gN\left(0,\left(\nabla_\theta^2 L\left(\theta^*\right)\right)^{-1} \operatorname{Cov}\left(\nabla_\theta \ell_{\theta^*}\left(\bfX\right)\right)\left(\nabla_\theta^2 L\left(\theta^*\right)\right)^{-1}\right).
    \]
\end{lemma}

Similar to Lemma 2 in \citet{li2024promises}, we have the following lemma.

\begin{lemma} \label{lemma:nabla-cov}
    Under the same assumptions as in \Cref{thrm:sample-efficiency-full}, we have
    \begin{equation} \label{eq:nabla-cov}
    \nabla_{\theta}^{2} L_k(\theta^*)  = \mathrm{Cov}_{(\bfXcontext, \bfXtest) \sim p, \, \pi \sim \uniform(\Pi_k)} \left(-\nabla_{\theta} \log q_{\theta^*}\left(\bfXtest_{\pi} \mid \bfXtest_{-\pi}, \bfXcontext\right)\right).
    \end{equation}
\end{lemma}
\begin{proof}
    It is easy to verify that $\theta^*$ is the minimizer of $L_k(\theta)$ for any $k \in [d]$ and hence $\theta^*_k = \theta^*$ by assumption.

    Rewrite $\nabla_{\theta}^{2} L_k(\theta^*)$ and we get that
    \begin{equation} \label{eq:proof-nabla-cov-1}
        \begin{aligned}
            & \, \nabla_{\theta}^{2} L_k(\theta^*) \\
            = & \, \nabla_{\theta}^{2} \E_{\bfXcontext, \bfXtest, \pi} \left[ -\log q_{\theta^*}(\bfXtest_\pi | \bfXtest_{-\pi}, \bfXcontext) \right] \\
            = & \, \E_{\bfXcontext, \bfXtest, \pi} \left[- \nabla_{\theta}^{2} \log q_{\theta^*}(\bfXtest_\pi | \bfXtest_{-\pi}, \bfXcontext) \right] \\
            = & \, \E_{\bfXcontext, \bfXtest, \pi} \left[\left(\nabla_{\theta}\log q_{\theta^*}(\bfXtest_\pi | \bfXtest_{-\pi}, \bfXcontext)\right)\left(\nabla_{\theta}\log q_{\theta^*}(\bfXtest_\pi | \bfXtest_{-\pi}, \bfXcontext)\right)^{\top} - \frac{\nabla_{\theta}^{2} q_{\theta^*}(\bfXtest_\pi | \bfXtest_{-\pi}, \bfXcontext)}{q_{\theta^*}(\bfXtest_\pi | \bfXtest_{-\pi}, \bfXcontext)}\right].
        \end{aligned}
    \end{equation}
    In addition, we have
    \begin{align*}
        & \, \E_{\bfXcontext, \bfXtest, \pi}\left[\frac{\nabla_{\theta}^{2} q_{\theta^*}(\bfXtest_\pi | \bfXtest_{-\pi}, \bfXcontext)}{q_{\theta^*}(\bfXtest_\pi | \bfXtest_{-\pi}, \bfXcontext)}\right] = \E_{\bfXcontext, \pi}\E_{\bfXtest_{-\pi}|\bfXcontext}\E_{\bfXtest_{\pi} | \bfXtest_{-\pi}, \bfXcontext}\left[\frac{\nabla_{\theta}^{2} q_{\theta^*}(\bfXtest_\pi | \bfXtest_{-\pi}, \bfXcontext)}{q_{\theta^*}(\bfXtest_\pi | \bfXtest_{-\pi}, \bfXcontext)}\right] \\
        = & \, \E_{\bfXcontext, \pi}\E_{\bfXtest_{-\pi}|\bfXcontext} \left[\int_{\Omega^{k}} \frac{\nabla_{\theta}^{2} q_{\theta^*}(\bfXtest_\pi | \bfXtest_{-\pi}, \bfXcontext)}{q_{\theta^*}(\bfXtest_\pi | \bfXtest_{-\pi}, \bfXcontext)} \cdot p(\bfXtest_\pi | \bfXtest_{-\pi}, \bfXcontext) \mathrm{d} \bfXtest_{\pi}\right] \\
        = & \, \E_{\bfXcontext, \pi}\E_{\bfXtest_{-\pi}|\bfXcontext} \left[\int_{\Omega^{k}} \nabla_{\theta}^{2} q_{\theta^*}(\bfXtest_\pi | \bfXtest_{-\pi}, \bfXcontext) \mathrm{d} \bfXtest_{\pi}\right] \tag{Due to the assumption $q_{\theta^*} = p$} \\
        = & \, \E_{\bfXcontext, \pi}\E_{\bfXtest_{-\pi}|\bfXcontext} \left[\nabla_{\theta}^{2}  \int_{\Omega^{k}} q_{\theta^*}(\bfXtest_\pi | \bfXtest_{-\pi}, \bfXcontext) \mathrm{d} \bfXtest_{\pi}\right] = \E_{\bfXcontext, \pi}\E_{\bfXtest_{-\pi}|\bfXcontext} \left[\nabla_{\theta}^{2}  1 \right] = 0.
    \end{align*}
    Combined with \Cref{eq:proof-nabla-cov-1}, we can get that
    \begin{equation} \label{eq:proof-nabla-cov-2}
        \nabla_{\theta}^{2} L_k(\theta^*) = \E_{\bfXcontext, \bfXtest, \pi} \left[\left(\nabla_{\theta}\log q_{\theta^*}(\bfXtest_\pi | \bfXtest_{-\pi}, \bfXcontext)\right)\left(\nabla_{\theta}\log q_{\theta^*}(\bfXtest_\pi | \bfXtest_{-\pi}, \bfXcontext)\right)^{\top}\right].
    \end{equation}
    Now rewrite the right-hand side of \Cref{eq:nabla-cov} and we get
    \begin{equation} \label{eq:proof-nabla-cov-3}
    \begin{aligned}
        & \, \mathrm{Cov} \left(-\nabla_{\theta} \log q_{\theta^*}\left(\bfXtest_{\pi} \mid \bfXtest_{-\pi}, \bfXcontext\right)\right) \\
        = & \, \E_{\bfXcontext, \bfXtest, \pi} \left[\left(\nabla_{\theta}\log q_{\theta^*}(\bfXtest_\pi | \bfXtest_{-\pi}, \bfXcontext)\right)\left(\nabla_{\theta}\log q_{\theta^*}(\bfXtest_\pi | \bfXtest_{-\pi}, \bfXcontext)\right)^{\top}\right] \\
        & \, - \E_{\bfXcontext, \bfXtest, \pi}\left[\nabla_{\theta}\log q_{\theta^*}(\bfXtest_\pi | \bfXtest_{-\pi}, \bfXcontext)\right] \E_{\bfXcontext, \bfXtest, \pi}\left[\nabla_{\theta}\log q_{\theta^*}(\bfXtest_\pi | \bfXtest_{-\pi}, \bfXcontext)\right] ^\top
    \end{aligned}.
    \end{equation}
    Note that
    \begin{align*}
        & \, \E_{\bfXcontext, \bfXtest, \pi}\left[\nabla_{\theta}\log q_{\theta^*}(\bfXtest_\pi | \bfXtest_{-\pi}, \bfXcontext)\right] = \E_{\bfXcontext, \bfXtest, \pi}\left[\frac{\nabla_{\theta}q_{\theta^*}(\bfXtest_\pi | \bfXtest_{-\pi}, \bfXcontext)}{q_{\theta^*}(\bfXtest_\pi | \bfXtest_{-\pi}, \bfXcontext)}\right] \\
        = & \, \E_{\bfXcontext, \pi}\E_{\bfXtest_{-\pi}|\bfXcontext}\E_{\bfXtest_{\pi} | \bfXtest_{-\pi}, \bfXcontext} \left[\frac{\nabla_{\theta}q_{\theta^*}(\bfXtest_\pi | \bfXtest_{-\pi}, \bfXcontext)}{q_{\theta^*}(\bfXtest_\pi | \bfXtest_{-\pi}, \bfXcontext)}\right] \\
        = & \, \E_{\bfXcontext, \pi}\E_{\bfXtest_{-\pi}|\bfXcontext}\left[\int_{\Omega^k}\frac{\nabla_{\theta}q_{\theta^*}(\bfXtest_\pi | \bfXtest_{-\pi}, \bfXcontext)}{q_{\theta^*}(\bfXtest_\pi | \bfXtest_{-\pi}, \bfXcontext)} \cdot p(\bfXtest_\pi | \bfXtest_{-\pi}, \bfXcontext) \mathrm{d} \bfXtest_{\pi}\right] \\
        = & \, \E_{\bfXcontext, \pi}\E_{\bfXtest_{-\pi}|\bfXcontext}\left[\int_{\Omega^k}\nabla_{\theta}q_{\theta^*}(\bfXtest_\pi | \bfXtest_{-\pi}, \bfXcontext)\mathrm{d} \bfXtest_{\pi}\right] \tag{Due to the assumption $q_{\theta^*} = p$} \\
        = & \, \E_{\bfXcontext, \pi}\E_{\bfXtest_{-\pi}|\bfXcontext}\left[\nabla_{\theta} \int_{\Omega^k}q_{\theta^*}(\bfXtest_\pi | \bfXtest_{-\pi}, \bfXcontext)\mathrm{d} \bfXtest_{\pi}\right] \\
        = & \, \E_{\bfXcontext, \pi}\E_{\bfXtest_{-\pi}|\bfXcontext}\left[\nabla_{\theta} 1\right] = 0.
    \end{align*}
    Combined with \Cref{eq:proof-nabla-cov-3}, we can get that
    \begin{equation} \label{eq:proof-nabla-cov-4}
    \mathrm{Cov} \left(-\nabla_{\theta} \log q_{\theta^*}\left(\bfXtest_{\pi} \mid \bfXtest_{-\pi}, \bfXcontext\right)\right) = \E \left[\left(\nabla_{\theta}\log q_{\theta^*}(\bfXtest_\pi | \bfXtest_{-\pi}, \bfXcontext)\right)\left(\nabla_{\theta}\log q_{\theta^*}(\bfXtest_\pi | \bfXtest_{-\pi}, \bfXcontext)\right)^{\top}\right].
    \end{equation}
    Now the claim follows from \Cref{eq:proof-nabla-cov-2,eq:proof-nabla-cov-4}.
\end{proof}
Based on the above lemmas, we can now prove \Cref{thrm:sample-efficiency-full}. The proof follows a similar idea to that of Theorem~1 in \citet{li2024promises}.
\begin{proof}[Proof of \Cref{thrm:sample-efficiency-full}]
    It is easy to verify that $\theta^*$ is the minimizer of $L_k(\theta)$ for any $k \in [d]$ and hence $\theta^*_k = \theta^*$ by assumption. According to \Cref{lemma:convergence}, it holds that for every sufficiently small neighborhood $S$ of $\theta^*$, there exists a sufficiently large $n$, such that there is a unique minimizer $\htheta_{k,n}$ in $S$. Furthermore, $\htheta_{k,n}$ satisfies:
    \[
    \sqrt{n}\left(\htheta_{k,n}-\theta^*\right) \xrightarrow{d} \gN\left(0, \Gamma_k\right).
    \]
    Here due to \Cref{lemma:nabla-cov}, it holds that
    \begin{equation} \label{eq:Gamma-k}
        \begin{aligned}
        \Gamma_k & = \left(\nabla_\theta^2 L_k\left(\theta^*\right)\right)^{-1} \mathrm{Cov}_{(\bfXcontext, \bfXtest) \sim p, \, \pi \sim \uniform(\Pi_k)} \left(-\nabla_{\theta} \log q_{\theta^*}\left(\bfXtest_{\pi} | \bfXtest_{-\pi}, \bfXcontext\right)\right)\left(\nabla_\theta^2 L_k\left(\theta^*\right)\right)^{-1} \\
        & = \left(\nabla_\theta^2 L_k\left(\theta^*\right)\right)^{-1}. \\
        \end{aligned}
    \end{equation}
    Fix a $k \in [d - 1]$. Then for every $\pi \in \Pi_{k+1}$ and $j \in \pi$, let $\gamma = \pi \setminus \{j\}$ and we have
    \begin{equation} \label{eq:proof-thrm:sample-efficiency-full-1}
        \begin{aligned}
            \log q_{\theta}(\bfXtest_{\pi}|\bfXtest_{-\pi}, \bfXcontext) & = \log q_{\theta}\left(\bfXtest_{\gamma}, \Xtest_j | \bfXtest_{-(\gamma \cup \{j\})}, \bfXcontext\right) \\
            & = \log q_{\theta}\left(\Xtest_j | \bfXtest_{-(\gamma \cup \{j\})}, \bfXcontext\right) + \log q_{\theta}\left(\bfXtest_{\gamma} | \bfXtest_{-\gamma}, \bfXcontext\right).
        \end{aligned}
    \end{equation}
    As a result, we have
    \begin{align*}
        & \, \nabla^2_{\theta} L_{k+1}(\theta^*) \\
        = & \, \E_{(\bfXcontext, \bfXtest) \sim p, \pi \sim \uniform(\Pi_{k+1})} \left[\left(\nabla_{\theta}\log q_{\theta^*}(\bfXtest_\pi | \bfXtest_{-\pi}, \bfXcontext)\right)\left(\nabla_{\theta}\log q_{\theta^*}(\bfXtest_\pi | \bfXtest_{-\pi}, \bfXcontext)\right)^{\top}\right] \tag{Due to \Cref{eq:proof-nabla-cov-2}} \\
        = & \, \E_{(\bfXcontext, \bfXtest) \sim p, \gamma \sim \uniform(\Pi_{k}), j \in \uniform([d] \setminus \gamma)} \left[\left(\nabla_{\theta}\log q_{\theta^*}(\bfXtest_\pi | \bfXtest_{-\pi}, \bfXcontext)\right)\left(\nabla_{\theta}\log q_{\theta^*}(\bfXtest_\pi | \bfXtest_{-\pi}, \bfXcontext)\right)^{\top}\right] \tag{Letting $\pi = \gamma \cup \{j\}$} \\
        = & \, \E_{\bfXcontext, \bfXtest, \gamma, j}\Bigg[\left(\nabla_\theta \log q_{\theta^*}\left(\Xtest_j | \bfXtest_{-(\gamma \cup \{j\})}, \bfXcontext\right) + \nabla_\theta \log q_{\theta^*}\left(\bfXtest_{\gamma} | \bfXtest_{-\gamma}, \bfXcontext\right)\right) \times \\
        & \qquad \left(\nabla_\theta \log q_{\theta^*}\left(\Xtest_j | \bfXtest_{-(\gamma \cup \{j\})}, \bfXcontext\right) + \nabla_\theta \log q_{\theta^*}\left(\bfXtest_{\gamma} | \bfXtest_{-\gamma}, \bfXcontext\right)\right)^\top\Bigg]. \tag{By \Cref{eq:proof-thrm:sample-efficiency-full-1}} \\
    \end{align*}
    Define
    \[
    \begin{aligned}
        A & = \E_{\bfXcontext, \bfXtest, \gamma, j}\left[\left(\nabla_\theta \log q_{\theta^*}\left(\Xtest_j | \bfXtest_{-(\gamma \cup \{j\})}, \bfXcontext\right)\right)\left(\nabla_\theta \log q_{\theta^*}\left(\Xtest_j | \bfXtest_{-(\gamma \cup \{j\})}, \bfXcontext\right)\right)^\top\right]. \\
        B & = \E_{\bfXcontext, \bfXtest, \gamma, j}\left[\left(\nabla_\theta \log q_{\theta^*}\left(\Xtest_j | \bfXtest_{-(\gamma \cup \{j\})}, \bfXcontext\right)\right)\left(\nabla_\theta \log q_{\theta^*}\left(\bfXtest_{\gamma} | \bfXtest_{-\gamma}, \bfXcontext\right)\right)^\top\right]. \\
        C & = \E_{\bfXcontext, \bfXtest, \gamma, j}\left[\left(\nabla_\theta \log q_{\theta^*}\left(\bfXtest_{\gamma} | \bfXtest_{-\gamma}, \bfXcontext\right)\right)\left(\nabla_\theta \log q_{\theta^*}\left(\bfXtest_{\gamma} | \bfXtest_{-\gamma}, \bfXcontext\right)\right)^\top\right].
    \end{aligned}
    \]
    Then we have
    \[
    \nabla^2_{\theta} L_{k+1}(\theta^*) = A + B + B^\top + C.
    \]
    It is easy to verify that $C = \nabla^2_{\theta} L_{k}(\theta^*)$. Now consider $B$. It holds that
    \begin{align*}
        B & = \E_{\bfXcontext, \bfXtest, \gamma, j}\left[\left(\nabla \log q_{\theta^*}\left(\Xtest_j | \bfXtest_{-(\gamma \cup \{j\})}, \bfXcontext\right)\right)\left(\nabla \log q_{\theta^*}\left(\bfXtest_{\gamma} | \bfXtest_{-\gamma}, \bfXcontext\right)\right)^\top\right] \\
        & = \E_{\bfXcontext, \gamma, j, \bfXtest_{-\gamma}}\left[\left(\nabla \log q_{\theta^*}\left(\Xtest_j | \bfXtest_{-(\gamma \cup \{j\})}, \bfXcontext\right)\right) \cdot \E_{\bfXtest_{\gamma}|\bfXtest_{-\gamma},\bfXcontext}\left[\left(\nabla \log q_{\theta^*}\left(\bfXtest_{\gamma} | \bfXtest_{-\gamma}, \bfXcontext\right)\right)^\top\right]\right] \\
        & = \E_{\bfXcontext, \gamma, j, \bfXtest_{-\gamma}}\left[\left(\nabla \log q_{\theta^*}\left(\Xtest_j | \bfXtest_{-(\gamma \cup \{j\})}, \bfXcontext\right)\right) \cdot 0\right] \tag{See \Cref{eq:proof-thrm:sample-efficiency-full-2} below} \\
        & = 0.
    \end{align*}
    Here the third equation is due to:
    \begin{equation} \label{eq:proof-thrm:sample-efficiency-full-2}
    \begin{aligned}
        \E_{\bfXtest_{\gamma}|\bfXtest_{-\gamma},\bfXcontext}\left[\nabla_\theta \log q_{\theta^*}\left(\bfXtest_{\gamma} | \bfXtest_{-\gamma}, \bfXcontext\right)\right] & = \int_{\Omega^k} \frac{\nabla_\theta q_{\theta^*}\left(\bfXtest_{\gamma} | \bfXtest_{-\gamma}, \bfXcontext\right)}{q_{\theta^*}\left(\bfXtest_{\gamma} | \bfXtest_{-\gamma}, \bfXcontext\right)} \cdot p\left(\bfXtest_{\gamma} | \bfXtest_{-\gamma}, \bfXcontext\right) \mathrm{d} \bfXtest_{\gamma} \\
        & = \int_{\Omega^k} \nabla_\theta q_{\theta^*}\left(\bfXtest_{\gamma} | \bfXtest_{-\gamma}, \bfXcontext\right) \mathrm{d} \bfXtest_{\gamma} \\
        & = \nabla_{\theta}\int_{\Omega^k} q_{\theta^*}\left(\bfXtest_{\gamma} | \bfXtest_{-\gamma}, \bfXcontext\right) \mathrm{d} \bfXtest_{\gamma} = \nabla_{\theta}1 = 0.
    \end{aligned}
    \end{equation}
    In addition, since $A \succeq 0$, we have
    \[
    \nabla^2_{\theta} L_{k+1}(\theta^*) = A + B + B^\top + C = A + \nabla^2_{\theta} L_{k}(\theta^*) \succeq \nabla^2_{\theta} L_{k}(\theta^*).
    \]
    Noting that $\Gamma_k = \big(\nabla^2_{\theta} L_k(\theta^*)\big)^{-1}$ by \Cref{eq:Gamma-k}, it follows that $\Gamma_{k+1} \preceq \Gamma_k$. Now the claim follows.
\end{proof}

\subsubsection{Generalization for Joint Distribution Learning}
The generalization error for joint distribution learning is characterized by a key concept termed \emph{approximate tensorization of entropy}. It measures the ``complexity'' of the distribution over $(\bfXtest,\bfXcontext)$ by evaluating how easily an algorithm can generate samples from the joint distribution $q(\bfXtest,\bfXcontext)$ with access to local conditional distributions $q(\cdot|\bfXcontext, \bfXtest_{-\pi})$. Technically, approximate tensorization of entropy is associated with the mixing time of Gibbs sampling dynamics, which is the sample generation algorithm to be considered.

\begin{definition}[Approximate Tensorization of Entropy~\citep{caputo2021block}]
\label{def:entropy}
For a distribution $q(\bfXtest,\bfXcontext)$ and the set of $k$-cell masks $\Pi_k$, if there exists a constant $C_k(q)$ depending on $q$ and $k$, such that for any distribution $r$ over $(\bfXtest,\bfXcontext)$,
\[
    D_{\operatorname{KL}} (r\parallel q) \leq C_k(q)\cdot \E_{(\bfXcontext, \bfXtest) \sim p, \pi \sim \uniform(\Pi_k)} \left [ D_{\operatorname{KL}} (r(\cdot|\bfXcontext, \bfXtest_{-\pi}) \parallel  q(\cdot|\bfXcontext, \bfXtest_{-\pi})) \right],
\]
then $q$ satisfies approximate tensorization of entropy with respect to the constant $C_k$ and the mask set $\Pi_k$. Let $\underline{C_k} (q)$ be the minimum of all possible constants $C_k(q)$ such that $q$ satisfies approximate tensorization of entropy.
\end{definition}

Before presenting the main result, we introduce a few regularity conditions in the parametric class $q_\theta(\bfXtest|\bfXcontext)$ that defines the model.
\begin{assumption}[Regularity Conditions in the Parametric Class~\citep{li2024promises}]
\label{assum:regularity}

\hspace{1pt}
\begin{enumerate}[leftmargin=*]
\item There exists $\beta \in (0,1)$ such that  $\forall~1\leq k \leq d,\forall~\pi \in \Pi_k$ and $\forall~\theta\in\Theta$, $p(\bfXtest_\pi|\bfXcontext,\bfXtest_{-\pi}) > 0$ implies $q_\theta(\bfXtest_\pi|\bfXcontext,\bfXtest_{-\pi}) > \beta$.
\item For any $\epsilon>0$, there exists a partition $\text{Par}(\Theta) = \{ \Theta_1,\dots,\Theta_{|\text{Par}(\Theta)|}\}$ of $\Theta$, such that $\forall~1\leq k \leq d,\forall~\pi \in \Pi_k$, $\forall~\Theta_i\in \text{Par}(\Theta)$, $\forall~\theta_1,\theta_2 \in \Theta_i$, and any $(\bfXtest,\bfXcontext)$,
\[
\left| \log q_{\theta_1}(\bfXtest_\pi|\bfXcontext, \bfXtest_{-\pi}) - \log q_{\theta_2}(\bfXtest_\pi|\bfXcontext, \bfXtest_{-\pi}) \right| 
\leq 
\frac{\epsilon}{2}.
\]
Let $N(\Theta,\epsilon)$ be the cardinality of the smallest partition $\text{Par}(\Theta)$ that satisfies the condition above.
\end{enumerate}
\end{assumption}
The first assumption implies that the true distribution $p(\cdot | \bfXcontext, \bfXtest_{-\pi})$ supports the parametric distribution $q_\theta(\cdot | \bfXcontext, \bfXtest_{-\pi})$ . The second assumption specifies the covering number of the parameter space $\Theta$ and the lipschitz continuity of the log-likelihood loss function.

\begin{theorem}[Formal Version of \Cref{theo:joint_dist}] 
\label{theo:joint_dist-full}
For $\theta \in \Theta$, assume $q_\theta(\bfXtest,\bfXcontext)$ satisfies approximate tensorization of entropy with respect to some constant $C_1(q_{\theta})$ and the mask set $\Pi_1$. Then for any $1\leq k \leq d$, $q_\theta(\bfXtest,\bfXcontext)$ satisfies approximate tensorization of entropy with respect to some constant $C_k(q_{\theta})$ and the mask set $\Pi_k$. Furthermore, $\underline{C_{k+1}}(q_{\theta}) \leq \underline{C_{k}}(q_{\theta})$.

Under \Cref{assum:regularity}, for any $\epsilon>0$ and any $\delta\in(0,\frac{1}{d})$, with probability at least $1-d\delta$, for any $1\leq k \leq d$ and any $\theta \in \Theta$,
\[
\E_{\bfXcontext\sim p} \left[ D_{\operatorname{TV}}(q_{\theta}(\cdot|\bfXcontext) \parallel p(\cdot|\bfXcontext)) \right]
<
\sqrt{
\frac{1}{2}\underline{C_k}(q_{\theta}) \left( \hat{L}_k(\theta) + B\log\frac{1}{\beta} + \epsilon \right) + C
},
\]
where $B=\sqrt{\frac{1}{\delta} \cdot (8|\Omega|)^{d(m+1)} N(\Theta, \epsilon) }
+
\sqrt{\frac{1}{2n} \cdot \log \frac{8 N(\Theta, \epsilon)}{\delta}}$, and 
$C = \sqrt{\frac{|\Omega|^{3d(m+1)}}{8\delta n}}$.
\end{theorem}

With a sufficiently large sample size $n$ and a sufficiently small loss value $\hat L_k(\theta)$, the upper bound is dominated by the term $\sqrt{\frac{1}{2}\underline{C_k}(q_{\theta})(B\log\frac{1}{\beta}+\epsilon)}$, which is scaled by $\underline{C_k}(q_{\theta})$, the constant for the approximate tensorization of entropy. The theorem implies a reduced upper bound for the estimation error of the joint distribution with an increasing number of masked cells.

We prove the monotonicity of $\underline{C_{k}}(q_{\theta})$ with respect to $k$ in \Cref{prop:app_monotonous}, and prove the upper bound for generalization error in \Cref{prop:app_joint_bound}.

\begin{proposition}
\label{prop:app_monotonous}
For $\theta \in \Theta$, assume $q_\theta(\bfXtest,\bfXcontext)$ satisfies approximate tensorization of entropy with respect to some constant $C_1(q_{\theta})$ and the mask set $\Pi_1$. Then for any $1\leq k \leq d$, $q_\theta(\bfXtest,\bfXcontext)$ satisfies approximate tensorization of entropy with respect to some constant $C_k(q_{\theta})$ and the mask set $\Pi_k$. Furthermore, $\underline{C_{k+1}}(q_{\theta}) \leq \underline{C_{k}}(q_{\theta})$.
\end{proposition}
\begin{proof}[Proof of \Cref{prop:app_monotonous}]
Since $q_\theta(\bfXtest,\bfXcontext)$ satisfies approximate tensorization of entropy with respect to some constant $C_1(q_{\theta})$ and the mask set $\Pi_1$, for any distribution $r$ over $(\bfXtest,\bfXcontext)$,
\[
D_{\operatorname{KL}} (r\parallel q_\theta) \leq C_1(q_\theta)\cdot \E_{(\bfXcontext, \bfXtest) \sim p, \pi \sim \uniform(\Pi_1)} \left [ D_{\operatorname{KL}} (r(\cdot|\bfXcontext, \bfXtest_{-\pi}) \parallel  q_\theta(\cdot|\bfXcontext, \bfXtest_{-\pi})) \right].
\]
By \Cref{def:entropy}, 
\[
D_{\operatorname{KL}} (r\parallel q_\theta) \leq \underline{C_1}(q_\theta)\cdot \E_{(\bfXcontext, \bfXtest) \sim p, \pi \sim \uniform(\Pi_1)} \left [ D_{\operatorname{KL}} (r(\cdot|\bfXcontext, \bfXtest_{-\pi}) \parallel  q_\theta(\cdot|\bfXcontext, \bfXtest_{-\pi})) \right].
\]
We prove the proposition by deduction. Assume for some $1\leq k < d$, $q_\theta(\bfXtest,\bfXcontext)$ satisfies approximate tensorization of entropy with respect to some constant $C_k(q_{\theta})$ and the mask set $\Pi_k$. It follows that
\begin{align}
\label{eq:app_monotonous_1}
D_{\operatorname{KL}} (r\parallel q_\theta) \leq \underline{C_k}(q_\theta)\cdot \E_{(\bfXcontext, \bfXtest) \sim p, \pi \sim \uniform(\Pi_k)} \left [ D_{\operatorname{KL}} (r(\cdot|\bfXcontext, \bfXtest_{-\pi}) \parallel  q_\theta(\cdot|\bfXcontext, \bfXtest_{-\pi})) \right].
\end{align}
We have
\begin{align}
\begin{split}
\label{eq:app_monotonous_2}
&\quad \E_{(\bfXcontext, \bfXtest) \sim p, \pi \sim \uniform(\Pi_{k+1})} \left [ D_{\operatorname{KL}} (r(\cdot|\bfXcontext, \bfXtest_{-\pi}) \parallel  q_\theta(\cdot|\bfXcontext, \bfXtest_{-\pi})) \right]  \\
&= \E_{(\bfXcontext, \bfXtest) \sim p, \pi \sim \uniform(\Pi_{k}), a \sim \uniform([d] \backslash \pi)} \left [ D_{\operatorname{KL}} (r(\cdot|\bfXcontext, \bfXtest_{-\pi \cup \{a\}}) \parallel  q_\theta(\cdot|\bfXcontext, \bfXtest_{-\pi \cup \{a\}})) \right] \\
&\geq \E_{(\bfXcontext, \bfXtest) \sim p, \pi \sim \uniform(\Pi_{k}), a \sim \uniform([d] \backslash \pi)} \left [ D_{\operatorname{KL}} (r(\cdot|\bfXcontext, \bfXtest_{-\pi}) \parallel  q_\theta(\cdot|\bfXcontext, \bfXtest_{-\pi})) \right] \\
&= \E_{(\bfXcontext, \bfXtest) \sim p, \pi \sim \uniform(\Pi_{k})} \left [ D_{\operatorname{KL}} (r(\cdot|\bfXcontext, \bfXtest_{-\pi}) \parallel  q_\theta(\cdot|\bfXcontext, \bfXtest_{-\pi})) \right].
\end{split}
\end{align}
The inequality follows from the data processing inequality: 
\[
D_{\operatorname{KL}}(p(x,y) \parallel q(x,y)) = \E_x\left[ D_{\operatorname{KL}}(p(y|x) \parallel q(y|x)) \right] + D_{\operatorname{KL}}(p(x) \parallel q(x)) \geq  \E_x\left[ D_{\operatorname{KL}}(p(y|x) \parallel q(y|x)) \right].
\]
Combining \Cref{eq:app_monotonous_1,eq:app_monotonous_2},
\begin{align*}
D_{\operatorname{KL}} (r\parallel q_\theta) \leq \underline{C_k}(q_\theta)\cdot \E_{(\bfXcontext, \bfXtest) \sim p, \pi \sim \uniform(\Pi_{k+1})} \left [ D_{\operatorname{KL}} (r(\cdot|\bfXcontext, \bfXtest_{-\pi}) \parallel  q_\theta(\cdot|\bfXcontext, \bfXtest_{-\pi})) \right].
\end{align*}
Therefore, $q_\theta(\bfXtest,\bfXcontext)$ satisfies approximate tensorization of entropy with respect to $\underline{C_k}(q_{\theta})$ and the mask set $\Pi_{k+1}$. It follows that $\underline{C_{k+1}}(q_{\theta}) \leq \underline{C_k}(q_{\theta})$. By deduction, for any $1\leq k \leq d$, $q_\theta(\bfXtest,\bfXcontext)$ satisfies approximate tensorization of entropy with respect to some constant $C_k(q_{\theta})$ and the mask set $\Pi_{k}$.
\end{proof}

\begin{proposition}
\label{prop:app_joint_bound}
Under \Cref{assum:regularity} and the condition in \Cref{prop:app_monotonous}, for any $\epsilon>0$ and any $\delta\in(0,\frac{1}{d})$, with probability at least $1-d\delta$, for any $1\leq k \leq d$ and any $\theta \in \Theta$,
\[
\E_{\bfXcontext\sim p} \left[ D_{\operatorname{TV}}(q_{\theta}(\cdot|\bfXcontext) \parallel p(\cdot|\bfXcontext)) \right]
<
\sqrt{
\frac{1}{2}\underline{C_k}(q_{\theta}) \left( \hat{L}_k(\theta) + B\log\frac{1}{\beta} + \epsilon \right) + C
},
\]
where $B=\sqrt{\frac{1}{\delta} \cdot (8|\Omega|)^{d(m+1)} N(\Theta, \epsilon) }
+
\sqrt{\frac{1}{2n} \cdot \log \frac{8 N(\Theta, \epsilon)}{\delta}}$, and 
$C = \sqrt{\frac{|\Omega|^{3d(m+1)}}{8\delta n}}$.
\end{proposition}
The proposition is a corollary from Theorem~4 in \citet{li2024promises}, which provides an upper bound for $D_{\operatorname{TV}}(q_{\theta} \parallel p)$ in our setting. The remaining gap is an extension of the result to total variation between conditional distributions. We present \citet{li2024promises}'s result as a lemma.
\begin{lemma}[\citet{li2024promises}, Theorem~4]
\label{lem:li_joint}
Consider random variables $\bfX \in \Omega^{d}$ and $\pi \subset [d]$. We are given $n$ i.i.d. samples $\{\rvx^{(i)}\}_{i=1}^n$ drawn from $p(\bfX)$. For each sample $\rvx^{(i)}$, we observe $l$ i.i.d. masks $\{\pi_j^{(i)}\}_{j=1}^l$ drawn from $p(\pi|\rvx^{(i)})$. Consider the empirical loss function:
\[
\hat{L}(\theta) = \frac{1}{nl} \sum_{i=1}^n \sum_{j=1}^l -\log q_{\theta}\left(\rvx_{\pi_j^{(i)}}^{(i)}|  \rvx_{-\pi_j^{(i)}}^{(i)}, \pi_j^{(i)} \right).
\]
Suppose the following conditions are met:
\begin{enumerate}[leftmargin=*]
\item There exists a constant $C(q_\theta)$ depending on $q_\theta$, such that for any distribution $r$ over $\bfX$,
\[
    D_{\operatorname{KL}} (r\parallel q) \leq C(q_\theta)\cdot \E_{\bfX \sim p(\bfX), \pi \sim p(\pi|\bfX)} \left [ D_{\operatorname{KL}} (r(\cdot|\bfX_{-\pi},\pi) \parallel  q(\cdot|\bfX_{-\pi},\pi) \right].
\]
\item There exists $\beta \in (0,1)$ such that  $\forall~\pi \subset [d]$ and $\forall~\theta\in\Theta$, $p(\rvx_\pi|\rvx_{-\pi},\pi) > 0$ implies $q_\theta(\rvx_\pi|\rvx_{-\pi},\pi) > \beta$.
\item For any $\epsilon>0$, there exists a partition $\text{Par}(\Theta) = \{ \Theta_1,\dots,\Theta_{|\text{Par}(\Theta)|}\}$ of $\Theta$, such that $\forall~\pi \subset [d]$, $\forall~\Theta_i\in \text{Par}(\Theta)$, $\forall~\theta_1,\theta_2 \in \Theta_i$, and any $\rvx$,
\[
\left| \log q_{\theta_1}(\rvx_\pi|\rvx_{-\pi},\pi) - \log q_{\theta_2}(\rvx_\pi|\rvx_{-\pi},\pi) \right| 
\leq 
\frac{\epsilon}{2}.
\]
Let $N(\Theta,\epsilon)$ be the cardinality of the smallest partition $\text{Par}(\Theta)$ that satisfies the condition above.
\end{enumerate}
Then for any $\epsilon>0$ and any $\delta\in(0,1)$, with probability at least $1-\delta$, for any $\theta \in \Theta$,
\[
 D_{\operatorname{TV}}(q_{\theta}(\bfX) \parallel p(\bfX)) 
<
\sqrt{
\frac{1}{2}C(q_{\theta}) \left( \hat{L}(\theta) + B\log\frac{1}{\beta} + \epsilon \right) + C
},
\]
where $B=\sqrt{\frac{1}{l\delta} \cdot (8|\Omega|)^{d} N(\Theta, \epsilon) }
+
\sqrt{\frac{1}{2n} \cdot \log \frac{8 N(\Theta, \epsilon)}{\delta}}$, and 
$C = \sqrt{\frac{|\Omega|^{3d}}{8\delta n}}$.
\end{lemma}
\begin{remark}
Theorem~4 in \citet{li2024promises} is specified for $\hat \theta$ as the minimizer of the empirical loss function. In fact, the proof applies uniformly to arbitrarily $\theta \in \Theta$.
\end{remark}

\begin{proof}[Proof of \Cref{prop:app_joint_bound}]
For each pair of $(\bfxcontexti, \bfxtesti)$, exactly one mask $\pi_i$ is drawn independently from $\uniform(\Pi_k)$. 
Therefore, $\pi \perp \mkern-9mu \perp (\bfXcontext,\bfXtest)$. 
It follows from \Cref{lem:li_joint} that for each $1\leq k \leq d$, 
for any $\epsilon>0$ and any $\delta\in(0,1)$, 
with probability at least $1-\delta$, for any $\theta\in\Theta$,
\begin{align}
\label{eq:app_prop_joint_bound_1}
 D_{\operatorname{TV}}(q_{\theta}(\bfXtest,\bfXcontext) \parallel p(\bfXtest,\bfXcontext)) 
<
\sqrt{
\frac{1}{2}\underline{C_k}(q_{\theta}) \left( \hat{L}_k(\theta) + B\log\frac{1}{\beta} + \epsilon \right) + C
},
\end{align}
where $B=\sqrt{\frac{1}{\delta} \cdot (8|\Omega|)^{d(m+1)} N(\Theta, \epsilon) }
+
\sqrt{\frac{1}{2n} \cdot \log \frac{8 N(\Theta, \epsilon)}{\delta}}$, and 
$C = \sqrt{\frac{|\Omega|^{3d(m+1)}}{8\delta n}}$.

By union bound, for any $\epsilon>0$ and any $\delta \in (0, \frac{1}{d})$, with probability at least $1-d\delta$, \Cref{eq:app_prop_joint_bound_1} is satisfied for any $1\leq k\leq d$ and any $\theta \in \Theta$. 

Furthermore, we have
\begin{align}
\begin{split}
\label{eq:app_prop_joint_bound_2}
\E_{\bfXcontext\sim p} \left[ D_{\operatorname{TV}}(q_{\theta}(\cdot|\bfXcontext) \parallel p(\cdot|\bfXcontext)) \right] 
&= \sum_{\bfxcontext \in \Omega^{m\times d}} p(\bfxcontext) D_{\operatorname{TV}}(q_{\theta}(\cdot|\bfXcontext=\bfxcontext) \parallel p(\cdot|\bfXcontext=\bfxcontext)) \\
&= \sum_{\bfxcontext \in \Omega^{m\times d}} p(\bfxcontext) \cdot \frac{1}{2} \sum_{\bfxtest \in \Omega^d} \left| q_\theta(\bfxtest|\bfxcontext) - p(\bfxtest|\bfxcontext)\right| \\
&= \frac{1}{2}\sum_{\bfxcontext \in \Omega^{m\times d}, \bfxtest \in \Omega^d}   \left| p(\bfxcontext)q_\theta(\bfxtest|\bfxcontext) - p(\bfxcontext)p(\bfxtest|\bfxcontext) \right| \\
&=  \frac{1}{2} \sum_{\bfxcontext \in \Omega^{m\times d}, \bfxtest \in \Omega^d}\left| q_\theta(\bfxtest, \bfxcontext) - p(\bfxtest, \bfxcontext) \right| \\
&=  D_{\operatorname{TV}}(q_{\theta}(\bfXtest,\bfXcontext) \parallel p(\bfXtest,\bfXcontext)).
\end{split}
\end{align}
The proof is complete by combining \Cref{eq:app_prop_joint_bound_1,eq:app_prop_joint_bound_2}.
\end{proof}

\end{document}